\documentclass[journal]{IEEEtran}

\usepackage{cite}
\usepackage{color}
\definecolor{tabcolor}{rgb}{.105,.410,.113}%
\usepackage{booktabs}
\usepackage{colortbl}
\usepackage[colorlinks=true, linkcolor=blue, filecolor=blue, citecolor=blue, urlcolor=cyan]{hyperref}
\usepackage{algpseudocode}

\usepackage{graphicx}       
\usepackage{epstopdf}
\usepackage{amsmath}
\usepackage{amssymb}
\usepackage{threeparttable}

\usepackage{multirow}
\ifCLASSINFOpdf

\else

\fi

\usepackage{subfigure}
\usepackage{cite}
\usepackage{colortbl}

\usepackage{bm}
\usepackage{algorithm}
\usepackage{algpseudocode}

\usepackage{graphicx}     
\usepackage{epstopdf}

\floatname{algorithm}{\small Algorithm}

\usepackage{algorithm}
\usepackage{algpseudocode}
\usepackage{amsmath}
\usepackage{makecell}

\newtheorem{remark}{Remark}

\newtheorem{definition}{Definition}
\newtheorem{lemma}{Lemma}
\newtheorem{theorem}{Theorem}

\newtheorem{corollary}{Corollary}

\begin{document}

\title{Diffusion Policy with Behavioral Advantage Correction for Offline Reinforcement Learning}

\author{Botao~Dong, Longyang~Huang, Ning~Pang, and Hongtian~Chen
\thanks{Botao Dong is with the School of Chemistry, Chemical Engineering and Biotechnology, Nanyang Technological University, Singapore 637459 (e-mail: botao.dong@ntu.edu.sg).}
}

\markboth{
}%
{Shell \MakeLowercase{\textit{et al.}}: Bare Demo of IEEEtran.cls for IEEE Journals}

\maketitle

\begin{abstract}
	In offline reinforcement learning (RL), the distribution shift between behavioral data and the learned policy can lead to erroneous \emph{Q}-value estimation, thereby misguiding the direction of policy optimization. To address this issue, we develop a behavioral advantage corrected policy evaluation (BAC-PE) approach, which utilizes the \emph{Q}-function of the behavior policy to correct the learned policy's \emph{Q}-function, thus mitigating pessimistic conservatism and overestimation bias. Furthermore, the convergence of BAC-PE is analyzed theoretically, and an upper bound on the difference between the learned \emph{Q}-function and the true \emph{Q}-function is derived. To alleviate distribution shift, this work employs diffusion models to represent both the behavior policy and the learned policy, performing distribution matching for accurate policy regularization. Additionally, \emph{Q}-value guidance is incorporated into the training process to achieve effective policy improvement. By combining BAC-PE with diffusion policy modeling, we propose the diffusion policy with behavioral advantage correction (DPBAC) algorithm. Compared to existing offline methods, DPBAC demonstrates stronger policy representation capabilities and effectively mitigates the bias in \emph{Q}-value estimation. Experimental results on multiple domains of D4RL tasks show that DPBAC achieves superior performance, with notable advantages over state-of-the-art (SOTA) algorithms.
\end{abstract}


\IEEEpeerreviewmaketitle

\section{Introduction}
\IEEEPARstart{A}{s} a pivotal subfield of machine learning, RL enables agents to learn optimal policies through iterative interactions with the environment and feedback-driven adaptation \cite{Sutton1}. This methodology has achieved substantial advancements in various domains such as robotics \cite{Feng2025TCDS, WangH2024, 9817657}, autonomous driving \cite{Wang_X2023, Chen2023Milestones, JiangY2024}, and industrial processes \cite{Dogru2024}, underscoring its versatility and potential in solving complex decision-making tasks \cite{CuiY_TNNLS2023}. However, online RL training requires extensive agent-environment interactions to collect samples \cite{CuiY_TNNLS2024}, leading to safety risks \cite{WHChen2024} and high data acquisition costs \cite{Yang2022}. These limitations are particularly pronounced in scenarios requiring real-time decision-making and stringent safety guarantees.

Compared with online RL, offline RL relies on pre-collected datasets for training \cite{CaoSSMC2024}, which eliminates the necessity of real-time interaction with the environment \cite{XiaL2024}. This not only enhances agent safety but also ensures the efficient utilization of available data. However, directly applying online RL methods in offline settings often results in training failures due to inherent challenges in offline RL \cite{Prudencio2023}, such as undesirable distribution shift \cite{WangJ2023} and the overestimation of out-of-distribution (OOD) actions \cite{LiuJAAAI}. To address these issues, existing studies primarily focus on two main strategies: value regularization, which aims to mitigate overestimation biases in \emph{Q}-values \cite{HuS2024}, and policy regularization, which enhances the consistency between the learned policy and the behavior policy \cite{YuanZ2024}.

Value regularization approaches aim to correct value-estimation errors caused by Bellman backups over actions outside the support of the offline dataset. Since the critic is directly trained only on state-action pairs from the given dataset, OOD actions queried by the actor or target policy may be assigned unrealistically large value estimates due to the lack of empirical transition support. Such positive extrapolation errors can be recursively amplified through temporal-difference backups, driving the learned policy toward unsupported regions. To break this error-amplification loop, value regularization modifies the critic objective by lowering the values of unsupported actions, suppressing uncertain estimates, or restricting backups to reliable in-dataset actions \cite{MiaoC2024, ZhangY2024}. Representative methods include conservative \emph{Q}-learning (CQL) \cite{KumarA2020}, which penalizes OOD actions to learn a lower-bound critic; implicit \emph{Q}-learning (IQL) \cite{KostrikovI2021}, which avoids explicit OOD action sampling by estimating expectile values and extracting high-advantage in-dataset actions; falSe COrrelation REduction (SCORE) \cite{DengZ2024}, which suppresses unreliable estimates using bootstrapped ensembles; and supported value regularization (SVR) \cite{MaoY2024}, which preserves Bellman updates on supported actions while penalizing OOD regions. Although these methods differ in identifying unsupported or uncertain actions, they all inject pessimism into policy evaluation to prevent OOD value exploitation. However, excessive pessimism may also limit policy improvement, since conservative penalties cannot always distinguish unreliable OOD actions from potentially high-return but under-sampled actions in partially covered datasets. Moreover, broad penalties may spill over through function approximation to nearby or potentially valuable in-distribution regions, and the induced negative bias can further propagate through Bellman backups, especially in sparse-reward or long-horizon tasks \cite{ShaoJ20223}. Consequently, value regularization can reduce overestimation, but it may also mis-rank promising actions and provide weak signals for policy improvement. This limitation motivates the proposed behavioral advantage corrected policy evaluation mechanism. Instead of imposing a uniformly conservative penalty, the proposed mechanism uses the \emph{Q}-value of the behavior policy as an adaptive reference to correct both overestimation and excessive pessimism.

Unlike value regularization, which focuses on constraining the learned \emph{Q}-values, policy regularization addresses distribution shift from the actor side by constraining the learned policy to remain close to the behavior policy induced by the dataset \cite{MaC2025}. Its objective can be interpreted as maximizing the learned value while imposing a constraint or penalty on the divergence between the learned policy and the behavior policy \cite{ZhangR2024}. Existing methods mainly differ in how the behavior support is represented and how strictly the learned policy is allowed to deviate from it. At the most conservative end, behavior cloning (BC) \cite{Fujimoto2021330} performs pure supervised imitation, which achieves strong behavioral consistency but ignores reward-driven improvement \cite{10004017}. Action-support constraint methods such as batch-constrained \emph{Q}-learning (BCQ) \cite{FujimotoS2019} relax this restriction by learning a generative model of dataset actions and applying a small perturbation, so that the actor can improve values while remaining close to sampled behavior actions. Divergence-based or trust-region-based methods further formulate this idea as an explicit regularized optimization problem, where a penalty term controls the distance between the learned policy and the behavior policy \cite{RanY2023ef, ZhangZ2024}. Dataset-geometry-aware approaches such as DOGE \cite{LiJXzhan2023} refine the notion of support by using geometric information of the dataset, allowing policy improvement in regions that are close to or generalizable from the data while remaining conservative in uncertain areas. More recently, diffusion-based behavior regularization uses expressive diffusion behavior models to estimate heterogeneous behavior distributions and regularize policy optimization more accurately \cite{ChenH2024}. Thus, policy regularization methods form a spectrum from strict imitation to value-aware constrained improvement, differing mainly in whether they constrain pointwise actions, distributional divergence, trust regions, dataset geometry, or generative behavior scores. Their common limitation is that an overly tight constraint inherits the suboptimality of the behavior policy, whereas an overly loose constraint reintroduces OOD actions and extrapolation error \cite{CaoS2024, ZhouZ2024}. Moreover, practical regularization relies on an accurate representation of the behavior policy and an appropriate discrepancy measure. When the behavior distribution is multimodal, long-tailed, or generated by multiple skills, simple deterministic or unimodal Gaussian multi-layer perceptron policies may collapse different action modes and yield inaccurate divergence estimates. In this case, feasible behavior modes can be mistakenly treated as OOD, and the learned policy may either be over-regularized or drift toward unsupported actions, resulting in degraded offline RL performance.

Inspired by the aforementioned analysis, we identify three critical challenges faced by existing offline RL methods: (1) improving the \emph{Q}-value estimation mechanism to avoid both optimistic overestimation and excessive pessimism, (2) implementing appropriate policy regularization to control distribution shift without simply inheriting the suboptimality of the behavior policy, and (3) constructing an expressive policy representation capable of accurately modeling complex behavior distributions while supporting value-guided policy improvement. To address these challenges, this work develops a behavioral advantage corrected policy evaluation (BAC-PE) mechanism, which leverages the \emph{Q}-function of the behavior policy as an adaptive reference for correcting the learned policy's \emph{Q}-function. Moreover, diffusion models are employed to represent both the behavior policy and the learned policy, enabling high-fidelity behavior-policy modeling and \emph{Q}-guided diffusion policy improvement. The key contributions of this work are outlined as follows:
\begin{enumerate}
	\item A behavioral advantage corrected policy evaluation mechanism, termed BAC-PE, is proposed for offline RL. Rather than broadly penalizing unsupported actions, BAC-PE uses the behavior-policy \emph{Q}-value as an adaptive reference to correct the learned critic, thereby mitigating both OOD overestimation and excessive pessimism.
	\item The convergence of BAC-PE is rigorously established via the corrected Bellman operator. An explicit upper bound on the discrepancy between the corrected and true \emph{Q}-functions is further derived, showing its dependence on transition estimation bias and the distributional discrepancy between the learned and behavior policies.
	\item BAC-PE is integrated with diffusion-based behavior-policy modeling and \emph{Q}-guided policy improvement to construct DPBAC. This design combines adaptive critic correction with an expressive diffusion policy class, enabling accurate policy regularization and reward-driven improvement beyond behavior imitation.
	\item Comprehensive experiments are conducted on multiple domains of the D4RL benchmark to compare DPBAC with SOTA offline RL algorithms. Evaluations of task performance, policy expressiveness, sensitivity, sparse-reward learning, and \emph{Q}-value estimation bias demonstrate the effectiveness and superiority of DPBAC.
\end{enumerate}

Compared with representative offline RL methods, the distinctions of DPBAC are elaborated as follows. Unlike value-regularization methods such as CQL, IQL, SCORE, and SVR, which suppress OOD value exploitation through conservative, in-dataset, uncertainty-aware, or support-aware value estimation \cite{KumarA2020, KostrikovI2021, DengZ2024, MaoY2024}, DPBAC introduces BAC-PE to correct the learned critic using the behavior policy \emph{Q}-value as an adaptive reference, thereby mitigating both overestimation and excessive pessimism. Compared with policy-regularization methods such as BC, BCQ, DOGE, and behavior-regularized approaches, which constrain policy learning through imitation losses, generative action constraints, dataset geometry, or divergence/trust-region penalties \cite{Fujimoto2021330, FujimotoS2019, LiJXzhan2023, RanY2023ef, ZhangZ2024}, DPBAC combines diffusion-based behavior modeling with value-guided improvement to maintain dataset consistency while seeking higher-return actions. Unlike diffusion-based offline RL methods such as DQL \cite{Wang172023} and SRPO \cite{ChenH2024}, which mainly employ diffusion models as expressive policy classes or behavior score regularizers, DPBAC further integrates diffusion policy modeling with behavioral advantage corrected critic learning and provides theoretical guarantees for policy evaluation. Thus, DPBAC jointly addresses critic bias correction, expressive behavior modeling, and \emph{Q}-guided policy improvement within a unified offline RL framework.

\section{Preliminaries}\label{d3239neee}
\subsection{Online Reinforcement Learning}
As the mathematical foundation of RL, the Markov decision process is formally expressed by a tuple $(\mathcal{S}, \mathcal{A}, \mathbb{P}, r, \gamma, \varrho_0)$, where $\mathcal{S}$ denotes the state space, $\mathcal{A}$ represents the action space, $\mathbb{P}$ corresponds to the state transition probability, $r$ defines the reward function satisfying $ | r(\bm{s}, \bm{a}) | < \overline{r}$, $\gamma$ specifies the discount factor, and $\varrho_0$ characterizes the initial state distribution.

To derive the optimal policy $\pi^{\ast}(\cdot |\bm{s})$, RL approaches typically involve learning a \emph{Q}-function $Q^{\pi}(\bm{s}, \bm{a})$, which quantifies the expected cumulative reward of a given policy $\pi(\cdot |\bm{s})$ for a given state-action pair $(\bm{s}, \bm{a})$, i.e., $Q^\pi(\bm{s}, \bm{a}) \triangleq \mathbb{E}_{\rho_{\pi}}[\sum^\infty_{i=t}\gamma^{i-t} r( \bm{s}_i, \bm{a}_i) | \bm{s}_t = \bm{s}, \bm{a}_t = \bm{a}]$, where $\rho_{\pi}$ represents the trajectory distribution arising from $\pi$. Notice that $Q^\pi(\bm{s}, \bm{a})$ is uniquely determined as the solution to the classic Bellman equation $Q^{\pi}(\bm{s}, \bm{a}) = (\mathcal{B}^{\pi} Q^{\pi}) (\bm{s}, \bm{a})$, where the Bellman operator $\mathcal{B}^{\pi}$ is formulated as
\begin{align}\label{dano382j0r537g67f4}
	(\mathcal{B}^{\pi} Q) (\bm{s}, \bm{a}) = r(\bm{s}, \bm{a}) + \gamma \mathbf{P}^{\pi} Q(\bm{s}, \bm{a}),
\end{align}
with $\mathbf{P}^{\pi} Q(\bm{s}, \bm{a}) \triangleq \mathbb{E}_{\bm{s}' \sim \mathbb{P}(\cdot |\bm{s}, \bm{a}), \bm{a}' \sim \pi(\cdot |\bm{s}')} [Q(\bm{s}', \bm{a}')]$.

\subsection{Offline Reinforcement Learning}
Unlike online RL, offline RL operates solely on a static dataset $\mathcal{D}$ collected by a behavior policy $\mu(\cdot |\bm{s})$, with no interaction with environment throughout the learning process. As a result, the state transition probability $\mathbb{P}(\cdot |\bm{s}, \bm{a})$ and the Bellman operator $\mathcal{B}^{\pi}$ are not directly accessible in offline RL. To this end, an estimated state transition probability $\widehat{\mathbb{P}}(\cdot |\bm{s}, \bm{a})$ can be approximated based on the given dataset, resulting in an estimated Bellman operator $\widehat{\mathcal{B}}^{\pi}$ defined as
\begin{align}\label{dn38hf7439b3uef}
	\Big(\widehat{\mathcal{B}}^{\pi} Q \Big) (\bm{s}, \bm{a}) = r(\bm{s}, \bm{a}) + \gamma \widehat{\mathbf{P}}^{\pi} Q(\bm{s}, \bm{a}),
\end{align}
where $\widehat{\mathbf{P}}^{\pi} Q(\bm{s}, \bm{a}) \triangleq \mathbb{E}_{\bm{s}' \sim \widehat{\mathbb{P}}(\cdot |\bm{s}, \bm{a}), \bm{a}' \sim \pi(\cdot |\bm{s}')} \big[Q(\bm{s}', \bm{a}')\big]$. Let $\widehat{Q}^{\pi}(\bm{s}, \bm{a})$ denote the unique fixed point of $\widehat{\mathcal{B}}^{\pi}$, i.e., $\widehat{Q}^{\pi}(\bm{s}, \bm{a}) = \big(\widehat{\mathcal{B}}^{\pi} \widehat{Q}^{\pi}\big) (\bm{s}, \bm{a})$. 

\subsection{Diffusion Model}
Diffusion models utilize a forward process of progressive noise injection combined with a reverse denoising mechanism to generate data samples \cite{Ho2020}. By capturing the probabilistic structure of the given data distribution, diffusion models can generate samples that closely match the statistical characteristics of the training data. The core of training diffusion models lies in optimizing a neural network to accurately predict the noise introduced at each step of the forward diffusion process. This ability enables the network to progressively remove the corresponding noise during the reverse denoising steps, ultimately generating samples that closely resemble the original data \cite{ZhangJ2024}.

\section{BAC-PE Analysis and Diffusion Policy Design}\label{d34954keg}
Due to the distribution shift, offline RL may involve the selection of OOD actions, which lack sample support, leading to potential overestimation of the \emph{Q}-function. This can cause the learned policy to diverge from the behavior policy, resulting in degraded performance or even policy failure. To address this issue, offline RL attempts to maximize performance within the data-supported distribution of the behavior policy $\mu(\cdot |\bm{s})$, with the optimization objective being:
\begin{align}\label{dn2o3lr2390rh205}
	\max_{\pi} \mathbb{E}_{\bm{s} \sim \mathcal{D}, \bm{a} \sim \pi(\cdot |\bm{s})} \bigg[\widehat{Q}^{\pi}(\bm{s}, \bm{a}) - \frac{1}{\lambda} \varUpsilon_{KL}\big[\pi(\cdot |\bm{s}) \Vert \mu(\cdot |\bm{s}) \big]\bigg],
\end{align}
where $\lambda$ represents the weighting factor and $\varUpsilon_{KL}(\cdot \|\cdot)$ denotes the KL divergence. As demonstrated in \cite{ChenH2024}, the optimization problem (\ref{dn2o3lr2390rh205}) has an analytical solution given by:
\begin{align}\label{dan9124n58909}
	\pi^{\ast}(\bm{a} |\bm{s}) = \frac{1}{W(\bm{s})} \mu(\bm{a} |\bm{s}) \exp\big(\lambda \widehat{Q}^{\pi}(\bm{s}, \bm{a})\big),
\end{align}
where $W(\bm{s})$ denotes the partition function. As inferred from (\ref{dan9124n58909}), achieving the optimal policy $\pi^{\ast}(\cdot |\bm{s})$ requires accurate modeling of the behavior policy $\mu(\cdot |\bm{s})$ and reliable estimation of the \emph{Q}-function. To this end, we develop a policy evaluation approach aimed at providing accurate and trustworthy \emph{Q}-function estimation, as detailed in Section \ref{da89234h29353590u8}, while facilitating high-fidelity representation of the behavior policy $\mu(\cdot |\bm{s})$ and empowering the learned policy $\pi(\cdot |\bm{s})$ to model complex and multimodal behavior distributions, as elaborated in Section \ref{d2138954766b233}.

\subsection{Behavioral Advantage Corrected Policy Evaluation}\label{da89234h29353590u8}
To concurrently address the issues of pessimistic bias and overestimation in \emph{Q}-function, a policy evaluation method incorporating behavioral advantage correction is designed. Prior to introducing the update law for the \emph{Q}-function, the formal definition of the behavioral advantage is provided as follows.
\begin{definition}
	The behavioral advantage $A^{\mu}_{Q}$, defined as the difference between the \emph{Q}-function of the behavior policy $Q^{\mu}$ and a given \emph{Q}-function $Q(\bm{s}, \bm{a})$, is expressed as:
	\begin{align}\label{dj983th3867893}
		A^{\mu}_{Q}(\bm{s}, \bm{a}) = Q^{\mu}(\bm{s}, \bm{a}) - Q(\bm{s}, \bm{a}), \forall (\bm{s}, \bm{a}).
	\end{align}
\end{definition}
\begin{remark}
	Based on the above definition, a positive behavioral advantage indicates that the \emph{Q}-function of the behavior policy, $Q^{\mu}(\bm{s}, \bm{a})$, suggests a higher expected return for executing the behavior policy at the current state-action pair $(\bm{s}, \bm{a})$. Conversely, a negative behavioral advantage implies a potential overestimation in the given \emph{Q}-function, $Q(\bm{s}, \bm{a})$, for the corresponding state-action pair. Such overestimation may propagate through the successive policy iteration process, potentially compounding errors and leading to a significant deterioration in overall policy performance.
\end{remark}

To address the challenges of pessimistic bias and erroneous overestimation in the learned \emph{Q}-function, the classic Bellman operator $\mathcal{B}^{\pi}$ in (\ref{dano382j0r537g67f4}) is refined by incorporating the behavioral advantage $A^{\mu}_{Q}$ in (\ref{dj983th3867893}). The formulation of the corrected Bellman operator $\mathcal{T}^{\pi}$ is presented as follows:
\begin{align}
	\big(\mathcal{T}^{\pi} Q \big)(\bm{s}, \bm{a}) = \big(\mathcal{B}^{\pi} Q \big) (\bm{s}, \bm{a}) + \eta A^{\mu}_{Q}(\bm{s}, \bm{a}),
\end{align}
where $\eta$ denotes the correction coefficient. In offline settings, where the state transition probability $\mathbb{P}(\cdot |\bm{s}, \bm{a})$ is unavailable, the estimated form of the corrected Bellman operator, denoted by $\widehat{\mathcal{T}}^{\pi}$, is formulated as:
\begin{align}\label{dan73693hr39976g}
	\Big(\widehat{\mathcal{T}}^{\pi} Q \Big)(\bm{s}, \bm{a}) = \Big(\widehat{\mathcal{B}}^{\pi} Q \Big) (\bm{s}, \bm{a}) + \eta \widehat{A}^{\mu}_{Q}(\bm{s}, \bm{a}),
\end{align}
where the estimated behavioral advantage is described as $\widehat{A}^{\mu}_{Q}(\bm{s}, \bm{a}) = \widehat{Q}^{\mu}(\bm{s}, \bm{a}) - Q(\bm{s}, \bm{a})$, and $\widehat{Q}^{\mu}(\bm{s}, \bm{a})$ denotes the estimated \emph{Q}-function of the behavior policy. Motivated by the corrected Bellman operator $\widehat{\mathcal{T}}^{\pi}$ in (\ref{dan73693hr39976g}), a novel iterative rule for policy evaluation is developed:
\begin{align}\label{dn3974t34th39}
	\widetilde{Q}_{k+1}(\bm{s}, \bm{a}) = \Big(\widehat{\mathcal{T}}^{\pi} \widetilde{Q}_{k}\Big)(\bm{s}, \bm{a}), k = 0, 1, 2 \cdots,
\end{align}
where $\widetilde{Q}_{k+1}(\bm{s}, \bm{a})$ denotes the value at each iteration step, starting from the initial iteration value $\widetilde{Q}_{0}(\bm{s}, \bm{a})$.

The vectorized form of the \emph{Q}-function is denoted by the bold symbol $\bm{Q}$, whose dimensions correspond to the product of the state space and action space dimensions. Each component of $\bm{Q}$ represents the \emph{Q}-value associated with a specific $(\bm{s}, \bm{a})$ pair. Accordingly, the vectorized representation of (\ref{dn3974t34th39}) is given by $\widetilde{\bm{Q}}_{k+1} = \widehat{\mathcal{T}}^{\pi} \widetilde{\bm{Q}}_{k}$. Then, the fixed point of $\widehat{\mathcal{T}}^{\pi}$, along with the convergence analysis of the iterative rule (\ref{dn3974t34th39}), is subsequently established in Theorem \ref{194hb5476f4909j33}.
\begin{theorem}\label{194hb5476f4909j33}
	Given a policy $\pi$ and an initial iterative value $\widetilde{\bm{Q}}_{0}$, the sequence of iterative values will converge to $\widetilde{\bm{Q}}^{\pi}$, which satisfies the fixed-point equation $\widetilde{\bm{Q}}^{\pi} = \widehat{\mathcal{T}}^{\pi} \widetilde{\bm{Q}}^{\pi}$, provided that the iterative rule (\ref{dn3974t34th39}) is employed.
\end{theorem}
\begin{IEEEproof}
	See Appendix \ref{proof_of_theorem_1}.
\end{IEEEproof}

Following this, when the estimation bias in the state transition probability $\widehat{\mathbb{P}}(\cdot |\bm{s}, \bm{a})$ is eliminated, $\widehat{\mathcal{T}}^{\pi}$ degenerates into $\mathcal{T}^{\pi}$. The contraction mapping property of $\mathcal{T}^{\pi}$, along with its fixed point $\overline{\bm{Q}}^{\pi}$, is presented in the subsequent corollary.
\begin{corollary}\label{dn7539hr339h8}
	The corrected Bellman operator $\mathcal{T}^{\pi}$ is a contraction mapping, with its fixed-point equation formulated as $\overline{\bm{Q}}^{\pi} = \mathcal{T}^{\pi} \overline{\bm{Q}}^{\pi}$.
\end{corollary}
\begin{IEEEproof}
	The proof of Corollary \ref{dn7539hr339h8} can be readily derived from that of Theorem \ref{194hb5476f4909j33} by substituting $\widehat{\mathcal{T}}^{\pi}$ with $\mathcal{T}^{\pi}$.
\end{IEEEproof}

Theorem \ref{194hb5476f4909j33} presents the convergence and fixed-point properties of the iterative rule (\ref{dn3974t34th39}). However, the relationship between the fixed points, $\widetilde{\bm{Q}}^{\pi}$ and $\bm{Q}^{\pi}$, which are induced by the operators $\widehat{\mathcal{T}}^{\pi}$ and $\mathcal{B}^{\pi}$ respectively, is further investigated in Theorem \ref{dn239r8324r73840jfd9df}.
\begin{theorem}\label{dn239r8324r73840jfd9df}
	The relationship between the fixed point $\widetilde{\bm{Q}}^{\pi}$ of the operator $\widehat{\mathcal{T}}^{\pi}$ and the fixed point $\bm{Q}^{\pi}$ of the operator $\mathcal{B}^{\pi}$ is characterized as follows:
	\begin{align}\label{dan392893hr83}
		\widetilde{\bm{Q}}^{\pi} = \bm{Q}^{\pi} & + \gamma (\bm{I} - \gamma \mathbf{P}^{\pi})^{-1} \big(\widehat{\mathbf{P}}^{\pi} - \mathbf{P}^{\pi}\big) \widetilde{\bm{Q}}^{\pi} \nonumber \\ & \ \ \ \ \ \ + \eta (\bm{I} - \gamma \mathbf{P}^{\pi})^{-1} \big(\widehat{\bm{Q}}^{\mu} - \widetilde{\bm{Q}}^{\pi}\big).
	\end{align}
\end{theorem}
\begin{IEEEproof}
	See Appendix \ref{proof_of_theorem_2}.
\end{IEEEproof}

Assuming that the state transition probability of the environment can be accurately estimated, i.e., $\widehat{\mathbb{P}}(\cdot |\bm{s}, \bm{a}) = \mathbb{P}(\cdot |\bm{s}, \bm{a})$, the equivalence between $\widehat{\mathcal{T}}^{\pi}$ and $\mathcal{T}^{\pi}$ holds, thereby yielding the following corollary concerning the fixed point $\overline{\bm{Q}}^{\pi}$ of $\mathcal{T}^{\pi}$.
\begin{corollary}
	The relationship between $\overline{\bm{Q}}^{\pi}$ and $\bm{Q}^{\pi}$ is characterized by $\overline{\bm{Q}}^{\pi} = \bm{Q}^{\pi} + \eta (\bm{I} - \gamma \mathbf{P}^{\pi})^{-1} \big(\bm{Q}^{\mu} - \overline{\bm{Q}}^{\pi}\big)$.
\end{corollary}
\begin{IEEEproof}
	The proof can be established by setting $\widehat{\mathbf{P}}^{\pi} = \mathbf{P}^{\pi}$ and replacing $\widetilde{\bm{Q}}^{\pi}$ with $\overline{\bm{Q}}^{\pi}$ in (\ref{dan392893hr83}).
\end{IEEEproof}

To formally quantify the discrepancy between $\overline{\bm{Q}}^{\pi}$ and $\bm{Q}^{\pi}$, Theorem \ref{sadn239er823rh923} presents a rigorous analysis of the associated bound. Before establishing this result, Lemma \ref{en32976dg37239922} is introduced to provide the necessary preliminary results for the subsequent derivation.
\begin{lemma}\label{en32976dg37239922}
	\cite{HuangLTPAMI2024} Given a policy $\pi$, the difference between its \emph{Q}-function, $\bm{Q}^{\pi}$, and that of the behavior policy, $\bm{Q}^{\mu}$, is bounded as follows:
	\begin{align}\label{d29387reh2937444}
		\big\| \bm{Q}^{\pi} - \bm{Q}^{\mu} \big\|_{\infty} \leq \frac{2\gamma \overline{r}}{(1-\gamma)^2} \max_{\bm{s}} \varUpsilon_{TV}\big[\pi(\cdot |\bm{s}) \Vert \mu(\cdot |\bm{s}) \big],
	\end{align}
	where $\varUpsilon_{TV}(\cdot \|\cdot)$ refers to the total variation divergence.
\end{lemma}

Building upon Lemma \ref{en32976dg37239922}, Theorem \ref{sadn239er823rh923} formally characterizes the difference between $\overline{\bm{Q}}^{\pi}$ and $\bm{Q}^{\pi}$ as follows.
\begin{theorem}\label{sadn239er823rh923}
	For an arbitrary policy $\pi$, the difference between $\overline{\bm{Q}}^{\pi}$ and $\bm{Q}^{\pi}$ is governed by the following bound:
	\begin{align}\label{en238437gbh48784}
		\Big\| \overline{\bm{Q}}^{\pi} - \bm{Q}^{\pi} \Big\|_{\infty} \leq 2 \eta \gamma \overline{r} \cdot \frac{\max_{\bm{s}} \varUpsilon_{TV}\big[\pi(\cdot |\bm{s}) \Vert \mu(\cdot |\bm{s}) \big]}{(1 - \widetilde{\gamma})(1 - \gamma)^2}. 
	\end{align}
\end{theorem}
\begin{IEEEproof}
	See Appendix \ref{proof_of_theorem_3}.
\end{IEEEproof}

Furthermore, taking into account the estimation bias between $\widehat{\mathbb{P}}(\cdot |\bm{s}, \bm{a})$ and $\mathbb{P}(\cdot |\bm{s}, \bm{a})$, the difference between $\widetilde{\bm{Q}}^{\pi}$ and $\bm{Q}^{\pi}$ is analyzed in Theorem \ref{dmff2398472}. In preparation for this analysis, Lemma \ref{d2n32983898y74bbnf} provides the requisite preliminary results as follows.
\begin{lemma}\label{d2n32983898y74bbnf}
	Considering the estimation bias in $\widehat{\mathbb{P}}(\cdot |\bm{s}, \bm{a})$, the difference between $\bm{Q}^{\pi}$ and the estimated \emph{Q}-function of behavior policy, $\widehat{\bm{Q}}^{\mu}$, is characterized by:
	\begin{align}\label{dj2903r83b8h939hn34}
		& \Big\| \bm{Q}^{\pi} - \widehat{\bm{Q}}^{\mu} \Big\|_{\infty} \leq \frac{2\gamma \overline{r}}{(1-\gamma)^2} \max_{\bm{s}} \varUpsilon_{TV}\big[\pi(\cdot |\bm{s}) \Vert \mu(\cdot |\bm{s}) \big] \nonumber \\
		& \ \ + \frac{2\gamma \overline{r}}{1-\gamma} \Big\| \big(\bm{I} - \gamma \widehat{\mathbf{P}}^{\mu}\big)^{-1} \Big\|_{\infty} \max_{\bm{s},\bm{a}} \varUpsilon_{TV}\Big[\widehat{\mathbb{P}}(\cdot |\bm{s}, \bm{a}) \big\Vert \mathbb{P}(\cdot |\bm{s}, \bm{a}) \Big]. 
	\end{align}
\end{lemma}
\begin{IEEEproof}
	See Appendix \ref{proof_of_lemma_2}.
\end{IEEEproof}

Drawing from Lemma \ref{d2n32983898y74bbnf}, Theorem \ref{dmff2398472} establishes the result concerning the difference between $\widetilde{\bm{Q}}^{\pi}$ and $\bm{Q}^{\pi}$.
\begin{theorem}\label{dmff2398472}
	Under an arbitrary policy $\pi$, the difference between $\widetilde{\bm{Q}}^{\pi}$ and $\bm{Q}^{\pi}$ is bounded as follows:
	\begin{align}\label{dnn38794br874r7484r49}
		& \Big\| \widetilde{\bm{Q}}^{\pi} - \bm{Q}^{\pi} \Big\|_{\infty} \leq 2 \eta \gamma \overline{r} \cdot \frac{\max_{\bm{s}} \varUpsilon_{TV}\big[\pi(\cdot |\bm{s}) \Vert \mu(\cdot |\bm{s}) \big]}{(1 - \widetilde{\gamma})(1 - \gamma)^2} \nonumber \\ 
		& + 2\gamma \overline{r} \Big(1 + \eta \Big\| \big(\bm{I} - \gamma \widehat{\mathbf{P}}^{\mu}\big)^{-1} \Big\|_{\infty}\Big) \frac{\max_{\bm{s},\bm{a}} \varUpsilon_{TV}\big[\widehat{\mathbb{P}} \big \Vert \mathbb{P} \big]}{(1-\gamma)(1 - \widetilde{\gamma})},
	\end{align}
	where $\varUpsilon_{TV}\big[\widehat{\mathbb{P}} \big \Vert \mathbb{P} \big]$ represents the abbreviated notation for $\varUpsilon_{TV}\big[\widehat{\mathbb{P}}(\cdot |\bm{s}, \bm{a}) \big \Vert \mathbb{P}(\cdot |\bm{s}, \bm{a}) \big]$.
\end{theorem}
\begin{IEEEproof}
	See Appendix \ref{proof_of_theorem_4}.
\end{IEEEproof}

\begin{remark}
	As analyzed in Theorem \ref{dmff2398472}, the difference between $\widetilde{\bm{Q}}^{\pi}$ and $\bm{Q}^{\pi}$ is primarily governed by the sum of two terms. One term represents the discrepancy between the actual state transition probability $\mathbb{P}(\cdot |\bm{s}, \bm{a})$ and its estimation $\widehat{\mathbb{P}}(\cdot |\bm{s}, \bm{a})$, quantified as $\varUpsilon_{TV}\big[\widehat{\mathbb{P}}(\cdot |\bm{s}, \bm{a}) \big\Vert \mathbb{P}(\cdot |\bm{s}, \bm{a}) \big]$, which is inherently influenced by the statistical properties and distribution of the given dataset. The other term captures the discrepancy between the learned policy $\pi(\cdot |\bm{s})$ and the behavior policy $\mu(\cdot |\bm{s})$, expressed as $\varUpsilon_{TV}\big[\pi(\cdot |\bm{s}) \Vert \mu(\cdot |\bm{s})\big]$. To reduce this policy distributional discrepancy, it is essential for the learned policy $\pi(\cdot |\bm{s})$ to be capable of modeling complex, multimodal behavior distributions, thereby enabling better alignment with the behavior policy $\mu(\cdot |\bm{s})$. To this end, the following diffusion-based policy mechanism is developed and elaborated in Section \ref{d2138954766b233}.
\end{remark}

\subsection{Diffusion-Based Policy Scheme}\label{d2138954766b233}
To capture the complex and multimodal behavior distribution inherent in the given static dataset $\mathcal{D} = \{(\bm{s},\bm{a},r,\bm{s}')\}$, a diffusion model is adopted. By learning the intricate distributional properties, the diffusion model facilitates the generation of actions that align with the dataset's statistical characteristics. The modeling process of the diffusion-based policy, which comprises the subsequent two stages, is detailed below.

\subsubsection{Forward Action Diffusion}
For every state-action pair $(\bm{s}, \bm{a}) \in \mathcal{D}$, the diffusion model progressively injects noise into the initial action $\bm{a}^0 = \bm{a} \sim \mu(\cdot |\bm{s})$, transforming it step-by-step into Gaussian noise $\bm{a}^{I_d} \sim \mathcal{N}(\bm{0}, \bm{I})$ over $I_d$ diffusion steps. At each step, noise is added to satisfy the transition:
\begin{align}\label{dn23883429r2gg}
	\hslash \big(\bm{a}^i \big|\bm{a}^{i-1}\big) = \mathcal{N}\big(\bm{a}^i ; \sqrt{\alpha_i} \bm{a}^{i-1}, (1 - \alpha_i) \bm{I}\big),
\end{align}
where $\alpha_i$ represents the noise scheduling parameter. From (\ref{dn23883429r2gg}), the marginal distribution associated with the joint distribution $\hslash \big(\bm{a}^{1:I_d}|\bm{a}^0\big) = \prod_{i=1}^{I_d} \hslash \big(\bm{a}^i|\bm{a}^{i-1}\big)$ can be formulated analytically as:
\begin{align}\label{dm3884n98fg4nowmg}
	\hslash \big(\bm{a}^i \big|\bm{a}^0 \big) = \mathcal{N}\big(\bm{a}^i ; \sqrt{\bar{\alpha}_i} \bm{a}^0, (1 - \bar{\alpha}_i) \bm{I}\big),
\end{align}
where $\bar{\alpha}_i = \prod_{m=1}^{i} \alpha_m$. By utilizing the reparameterization trick on (\ref{dm3884n98fg4nowmg}), the following expression is obtained:
\begin{align}\label{dm23rh9233bfdee}
	\bm{a}^i = \bm{\xi} \sqrt{1 - \bar{\alpha}_i} + \bm{a}^0 \sqrt{\bar{\alpha}}_i, i \in \{ 1, 2, \cdots, I_d \},
\end{align}
where $\bm{\xi} \sim \mathcal{N}(\bm{0}, \bm{I})$ represents a noise sample.

\subsubsection{Reverse Denoising for Action Generation}
The reverse denoising process facilitates the recovery of actions by progressively reducing noise and reconstructing the target distribution from an initial Gaussian sample $\bm{a}^{I_d} \sim \mathcal{N}(\bm{0}, \bm{I})$. To this end, the policy distribution is defined as:
\begin{align}
	\pi(\bm{a} | \bm{s}) & \triangleq \hbar \big(\bm{a}^{0:I_d}\big| \bm{s}\big) \nonumber \\ & = \mathcal{N} \big(\bm{a}^{I_d}; \bm{0}, \bm{I}\big) \prod_{i=1}^{I_d} \hbar \big(\bm{a}^{i-1}\big|\bm{a}^i, \bm{s} \big),
\end{align}
where $\hbar \big(\bm{a}^{i-1}\big|\bm{a}^i, \bm{s} \big) = \mathcal{N}\big(\bm{a}^{i-1}; \bm{\chi}(\bm{a}^i, \bm{s}, i), \bm{\varphi}^2(\bm{a}^i, \bm{s}, i)\big)$ characterizes the distribution of a single denoising step. Here, $\bm{\chi}(\bm{a}^i, \bm{s}, i)$ and $\bm{\varphi}^2(\bm{a}^i, \bm{s}, i)$ denote the corresponding denoising mean and covariance, respectively.

\section{Algorithm Implementation}\label{d83nekfnekei}
Building upon the BAC-PE developed in Section \ref{da89234h29353590u8} and the diffusion policy scheme introduced in Section \ref{d2138954766b233}, this section formally presents the DPBAC algorithm, with its structural diagram shown in Fig. \ref{54nono2895r}. The detailed algorithm workflow is illustrated in Algorithm \ref{DPBAC_algorithm}, and the optimization objectives for all employed networks are elaborated as follows.

\begin{figure}[!ht]
	\centering
	\includegraphics[width=8.87cm]{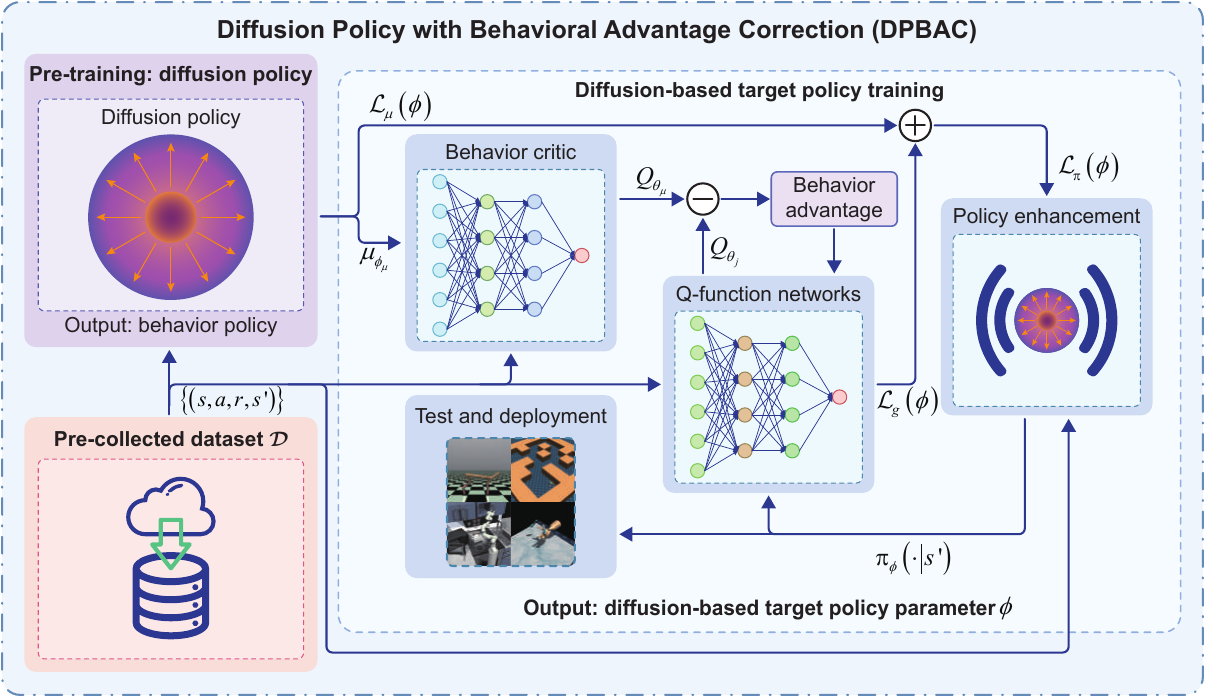}
	\caption{Schematic representation of the developed DPBAC algorithm.}
	\label{54nono2895r}
\end{figure}

\subsection{Diffusion-Based Behavior Policy Generation}\label{dn29384bbnkett}
In the DPBAC algorithm, a diffusion model parameterized by $\phi$ is employed to model the behavior policy $\mu(\cdot |\bm{s})$, capturing its complex and multimodal data distribution characteristics. Following \cite{Ho2020}, $\hbar_{\phi}\big(\bm{a}^{i-1}\big|\bm{a}^i, \bm{s}\big)$ can be parameterized as a noise prediction module, with the covariance $\bm{\varphi}^2 \big(\bm{a}^i, \bm{s}, i \big) = (1-\alpha_i) \bm{I}$ and the mean formulated as:
\begin{align}\label{adm32rt209458t340jf}
	\bm{\chi}_{\phi}(\bm{a}^i, \bm{s}, i) = \frac{1}{\sqrt{\alpha_i}} \bigg(\bm{a}^i - \frac{1-\alpha_i}{\sqrt{1-\bar{\alpha}_i}} \bm{\xi}_{\phi}(\bm{a}^i, \bm{s}, i)\bigg).
\end{align}
The practical implementation involves predicting the Gaussian noise $\bm{\xi}$ in (\ref{dm23rh9233bfdee}) using a neural network $\bm{\xi}_{\phi}$, which is trained with the following loss function:
\begin{align}\label{adno3nfib3fi2}
	\mathcal{L}_{\mu}(\phi) = \mathbb{E}_{\bm{\xi} \sim \mathcal{N}(\bm{0}, \bm{I}), (\bm{s}, \bm{a}) \sim \mathcal{D}} \Big[\big\Vert \bm{\xi} - \bm{\xi}_{\phi}(\bm{a}^i, \bm{s}, i) \big\Vert^2 \Big], 
\end{align}
where $\bm{a}^i = \bm{\xi} \sqrt{1 - \bar{\alpha}_i} + \bm{a}^0 \sqrt{\bar{\alpha}}_i$ and $i \sim \mathcal{U}\{1, I_d\}$. As outlined in Algorithm \ref{DPBAC_algorithm}, once the behavior policy has been trained using the diffusion policy loss in (\ref{adno3nfib3fi2}), its parameters are set as $\phi_{\mu} = \phi$, yielding the parameterized behavior policy $\mu_{\phi_{\mu}}(\cdot |\bm{s}')$. When generating an action $\bm{a}$ for the current state $\bm{s}$, the process begins by sampling Gaussian noise $\bm{a}^{I_d} \sim \mathcal{N}(\bm{0}, \bm{I})$, followed by a stepwise reverse denoising procedure based on the following formulation:
\begin{align}\label{dab8394beiie}
	\bm{a}^{i-1} | \bm{a}^{i} = \bm{\chi}_{\phi}\big(\bm{a}^i, \bm{s}, i \big) + \bm{\varphi}\big(\bm{a}^i, \bm{s}, i \big) \bm{\xi}, i \in \{1, \cdots, I_d \},
\end{align}
ultimately yielding the generated action $\bm{a} = \bm{a}^0$. To expedite the action sampling process in (\ref{dab8394beiie}), the DPM-Solver \cite{CLu2022} is utilized to accelerate the reverse denoising process.

\begin{algorithm}[htpb]\footnotesize
	\caption{\small DPBAC}
	\renewcommand{\algorithmicrequire}{\textbf{Initialize:}}
	\renewcommand{\algorithmicensure}{\textbf{Output:}}
	\label{DPBAC_algorithm}
	\begin{algorithmic}[1]
		\Require \emph{Q}-function networks parameters $\theta_1$, $\theta_2$, diffusion policy network parameter $\phi$, behavior critic network parameter $\theta_{\mu}$, pre-collected dataset $\mathcal{D}$, target networks parameters for \emph{Q}-function $\bar{\theta}_1 \leftarrow \theta_1$, $\bar{\theta}_2 \leftarrow \theta_2$, target network parameters for diffusion policy $\bar{\phi} \leftarrow \phi$, target policy update periodicity $\delta$, target network smoothing factor $\sigma$, behavior policy training steps $n_{\max}$, policy training steps $N_{\max}$.
		\State \textbf{\# Pre-training: diffusion-based behavior policy training}
		\For{behavior policy training step $n=1, 2,\cdots,n_{\max}$}
			\State Randomly extract a mini-batch from the pre-collected dataset $\mathcal{D}$.
			\State Optimize the diffusion policy parameter $\phi$ via minimizing (\ref{adno3nfib3fi2}).
		\EndFor
		\State Record the current diffusion policy parameter $\phi$ as $\phi_{\mu}$, and set $\bar{\phi} = \phi$.
		\State \textbf{\# Diffusion-based learned policy training}
		\For{policy training step $n = 1, 2, \cdots,N_{\max}$}
			\State Randomly extract a mini-batch from the pre-collected dataset $\mathcal{D}$.
			\State Optimize the behavior critic parameter $\theta_{\mu}$ via minimizing (\ref{adnon7342934ff}).
			\State Optimize \emph{Q}-function networks parameters $\theta_j$ via minimizing (\ref{dab712ne920345}).
			\If {$\delta | n$}
				\State Optimize the diffusion policy parameter $\phi$ via minimizing (\ref{adnb329045426}).
				\State Update the corresponding target network parameters with
				\begin{align}
					& \bar{\theta}_j \leftarrow (1-\sigma) \bar{\theta}_j + \sigma \theta_j, \nonumber \\ 
					& \ \bar{\phi} \leftarrow (1-\sigma) \bar{\phi}+\sigma \phi. \nonumber
				\end{align}
			\EndIf
		\EndFor
		\Ensure The diffusion-based learned policy network parameters $\phi$.
	\end{algorithmic}
\end{algorithm}

\subsection{Behavior Critic Training}
A behavior critic network parameterized by $\theta_{\mu}$ is employed to approximate the \emph{Q}-function $Q^{\mu}(\bm{s}, \bm{a})$ of the behavior policy. The training objective is to minimize the difference between its estimated value and the TD target $z_{\mu}$, with the loss function formulated as follows:
\begin{align}\label{adnon7342934ff}
	\mathcal{L}^{\mu}_{Q}(\theta_{\mu}) = \mathbb{E}_{(\bm{s}, \bm{a}, r, \bm{s}') \sim \mathcal{D}} \Big[\big(z_{\mu} - Q_{\theta_{\mu}}(\bm{s}, \bm{a}) \big)^2 \Big],
\end{align}
where $z_{\mu} = r + \gamma Q_{\theta_{\mu}}(\bm{s}', \tilde{\bm{a}}')$, and $\tilde{\bm{a}}' \sim \mu_{\phi_{\mu}}(\cdot |\bm{s}')$.

\subsection{Behavioral Advantage Corrected Critic Learning}
To implement the BAC-PE rule (\ref{dn3974t34th39}), the following loss function is derived from (\ref{dan73693hr39976g}) for training the \emph{Q}-function networks corresponding to the learned policy, each parameterized by $\theta_j$:
\begin{align}\label{dab712ne920345}
	\mathcal{L}^{\pi}_{Q}(\theta_j) = \mathbb{E}_{(\bm{s}, \bm{a}, r, \bm{s}') \sim \mathcal{D}} \Big[\big(z_{\pi} - Q_{\theta_j}(\bm{s}, \bm{a}) \big)^2 \Big],
\end{align}
where the TD target $z_{\pi}$ incorporates the reward signal, the discounted estimate of future returns, and the behavioral advantage correction term: $z_{\pi} = r + \gamma \min_{j = 1, 2} \{ Q_{\bar{\theta}_j}(\bm{s}', \hat{\bm{a}}') \} + \eta (Q_{\theta_{\mu}}(\bm{s}, \bm{a}) - Q_{\bar{\theta}_j}(\bm{s}, \bm{a}))$, with $\hat{\bm{a}}' \sim \pi_{\bar{\phi}}(\cdot |\bm{s}')$.

\begin{figure}[!ht]
	\centering
	\subfigure[]{
	\includegraphics[width=1.881cm]{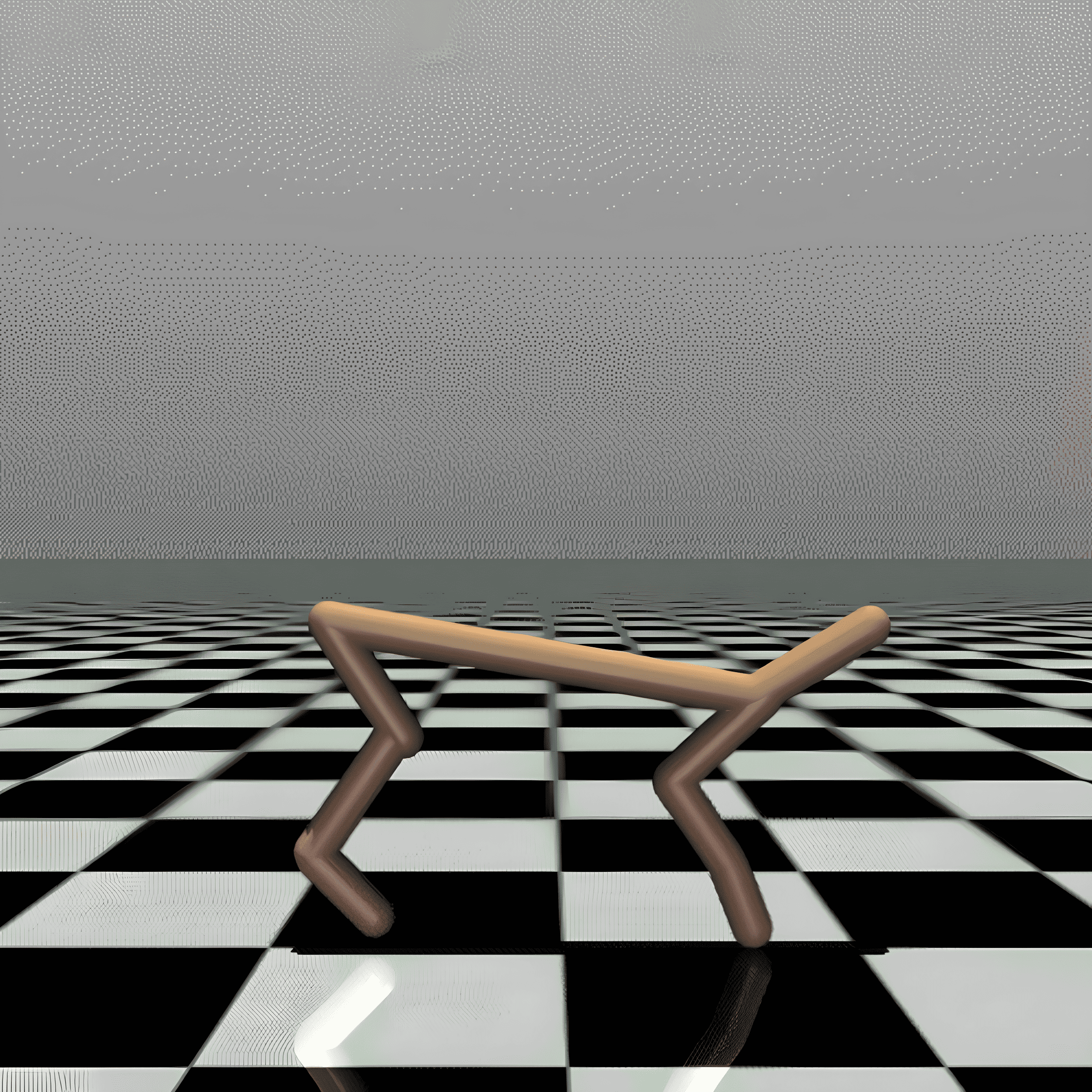}
	}
	\subfigure[]{
	\includegraphics[width=1.881cm]{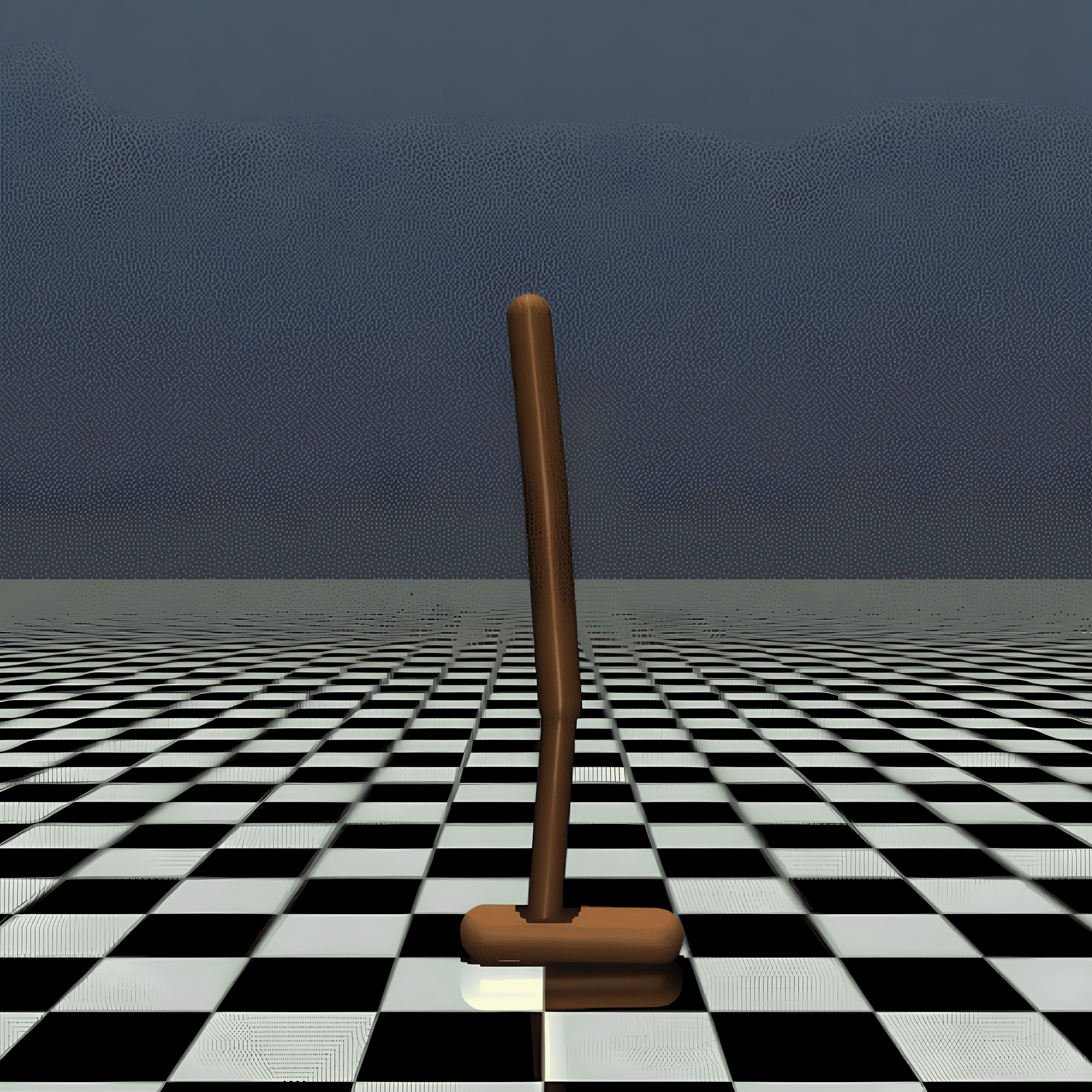}
	}
	\subfigure[]{
	\includegraphics[width=1.881cm]{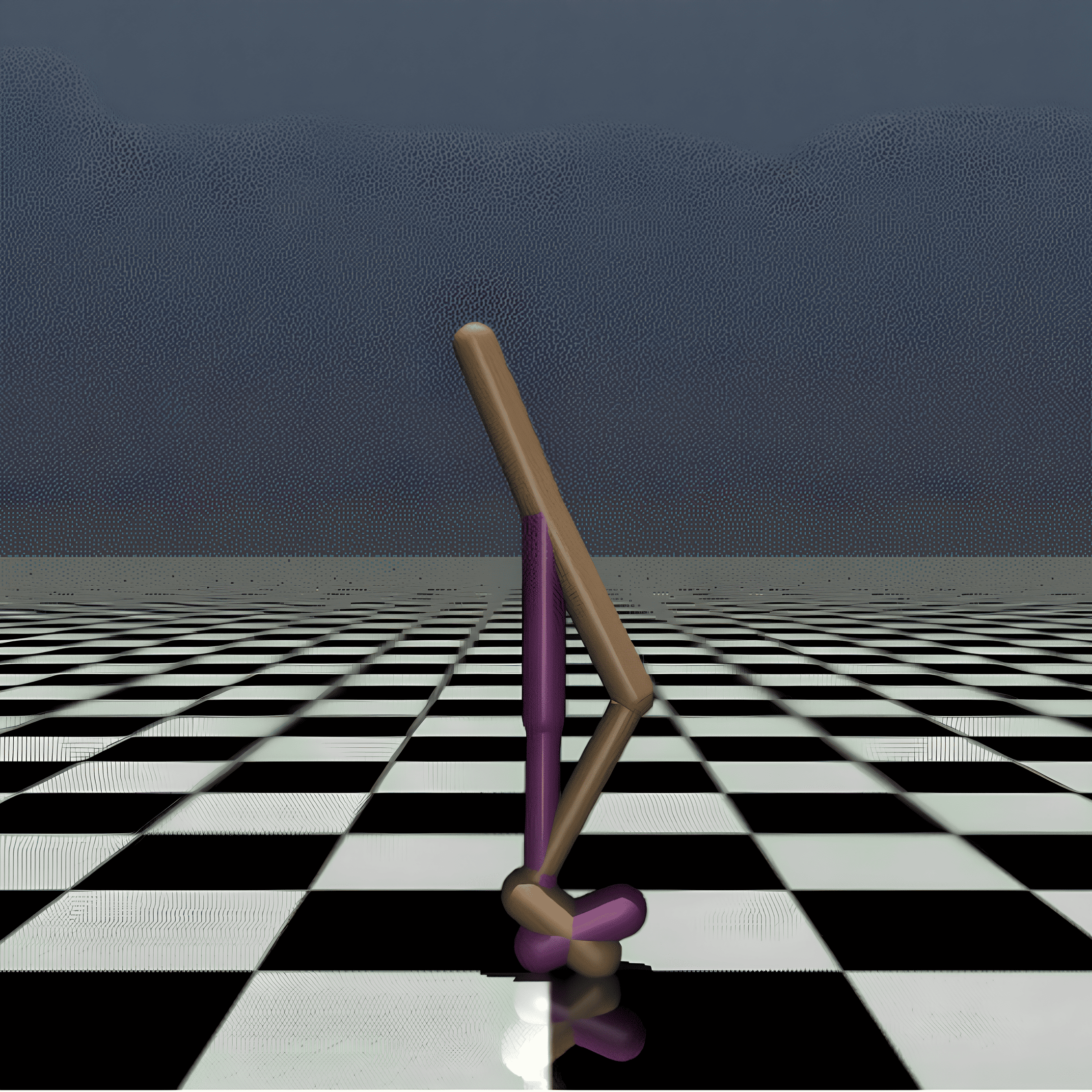}
	}
	\subfigure[]{
	\includegraphics[width=1.881cm]{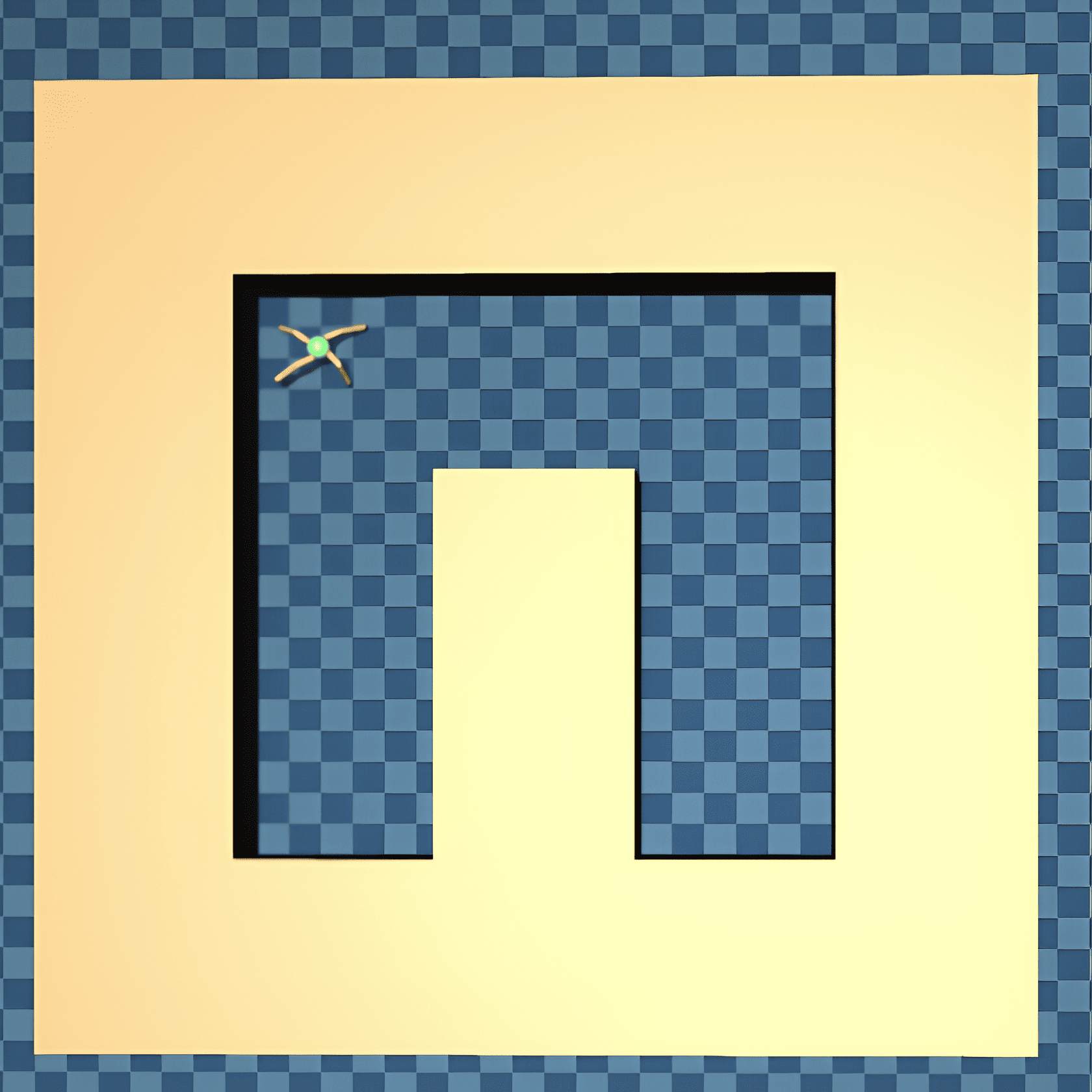}
	}
	\subfigure[]{
	\includegraphics[width=1.881cm]{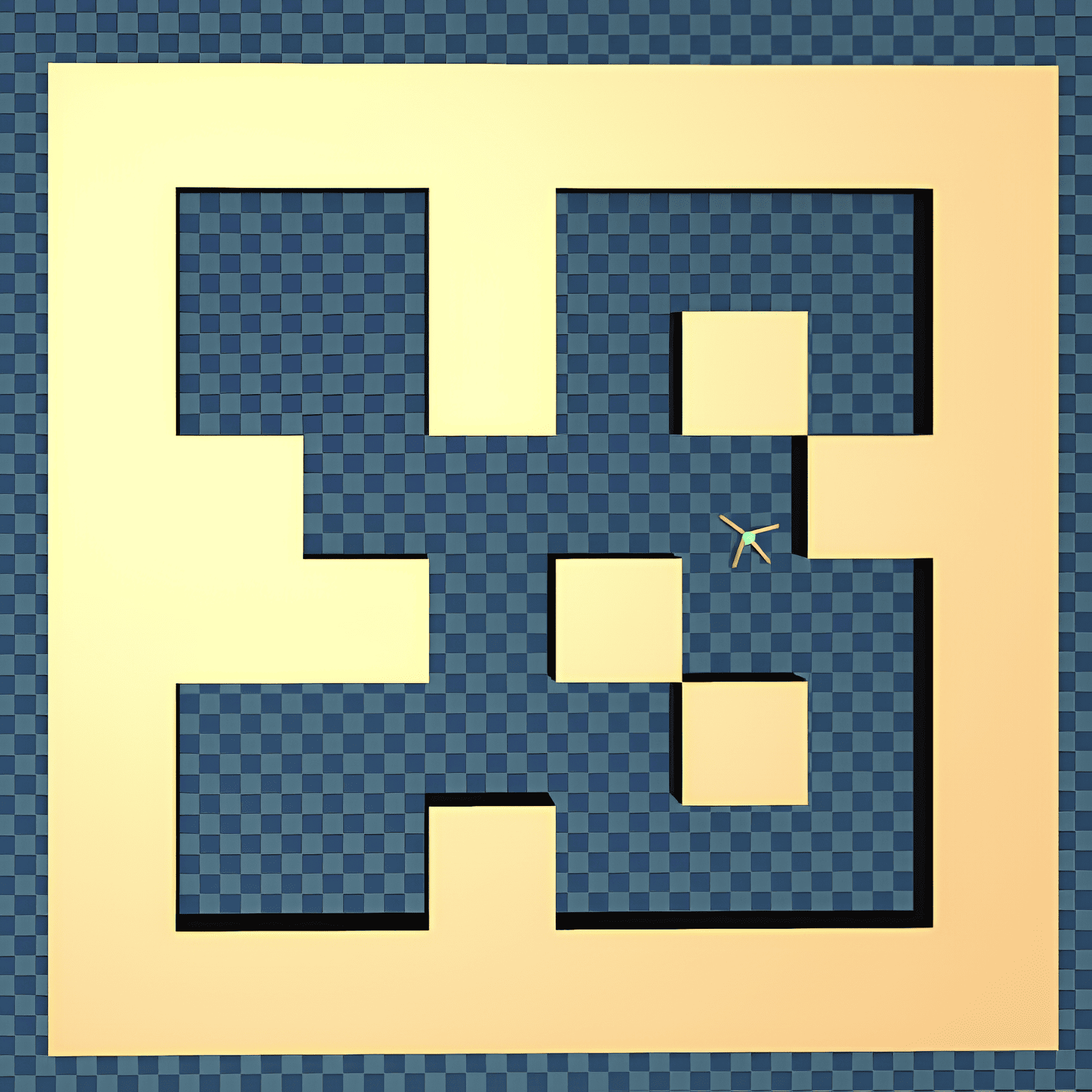}
	}
	\subfigure[]{
	\includegraphics[width=1.881cm]{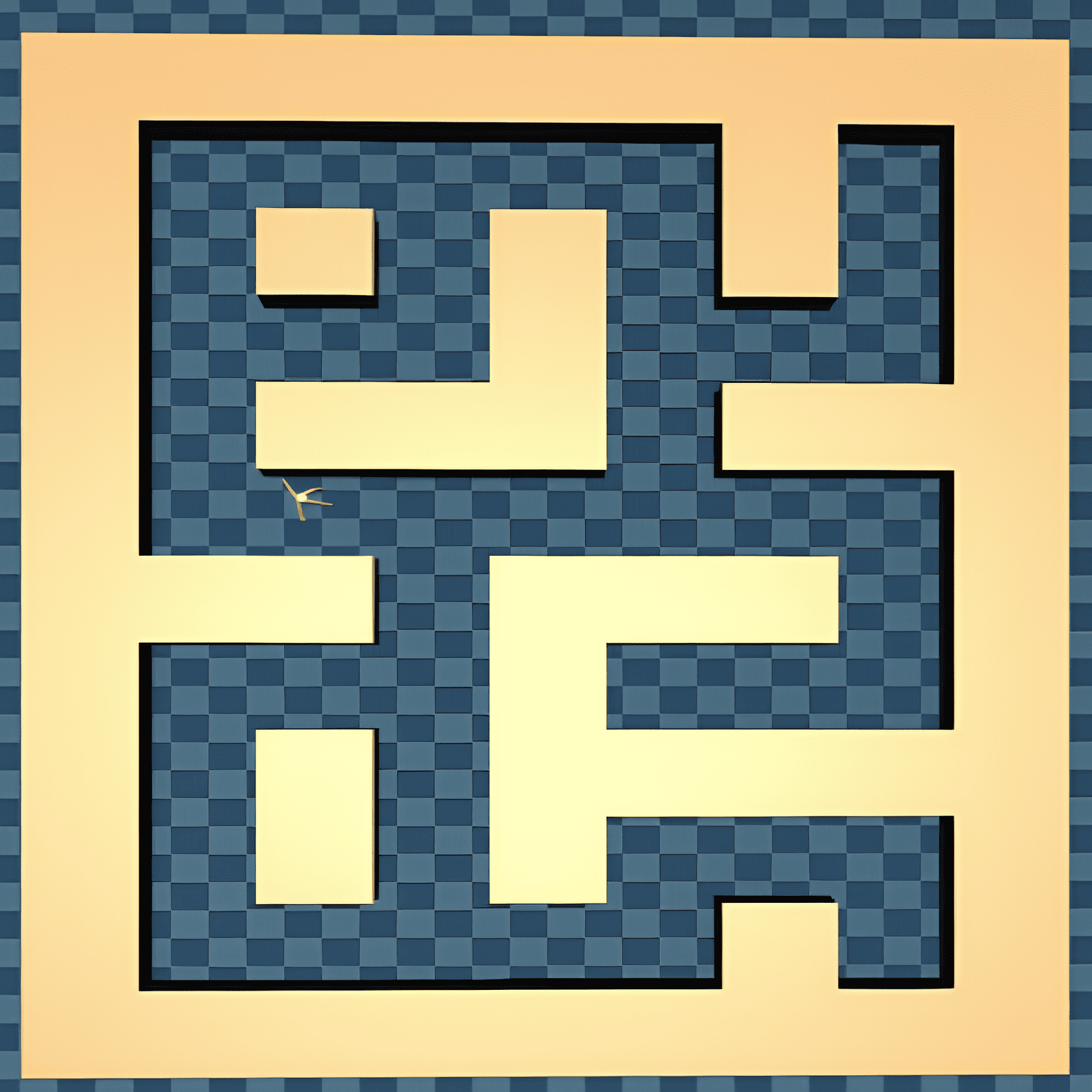}
	}
	\subfigure[]{
	\includegraphics[width=1.881cm]{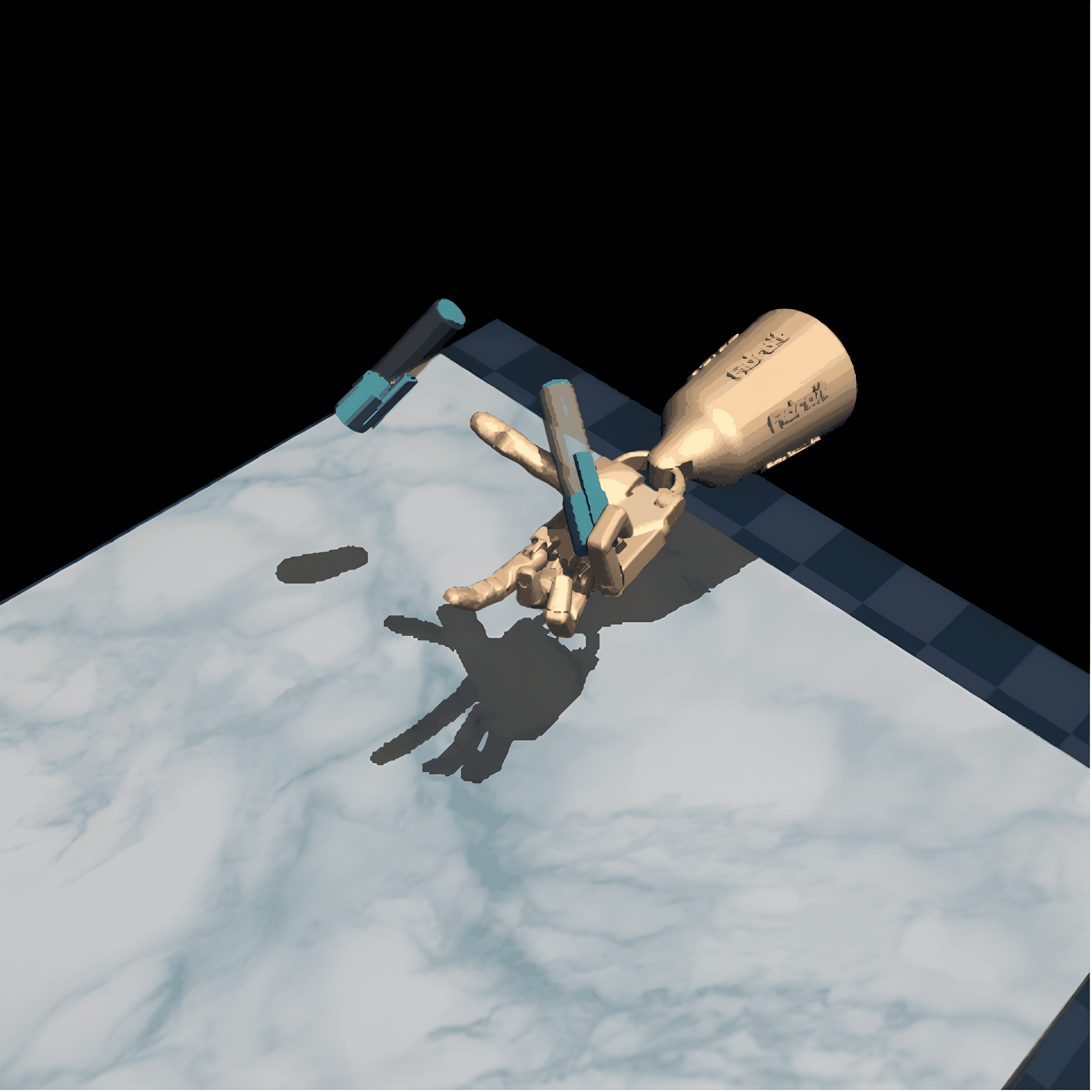}
	}
	\subfigure[]{
	\includegraphics[width=1.881cm]{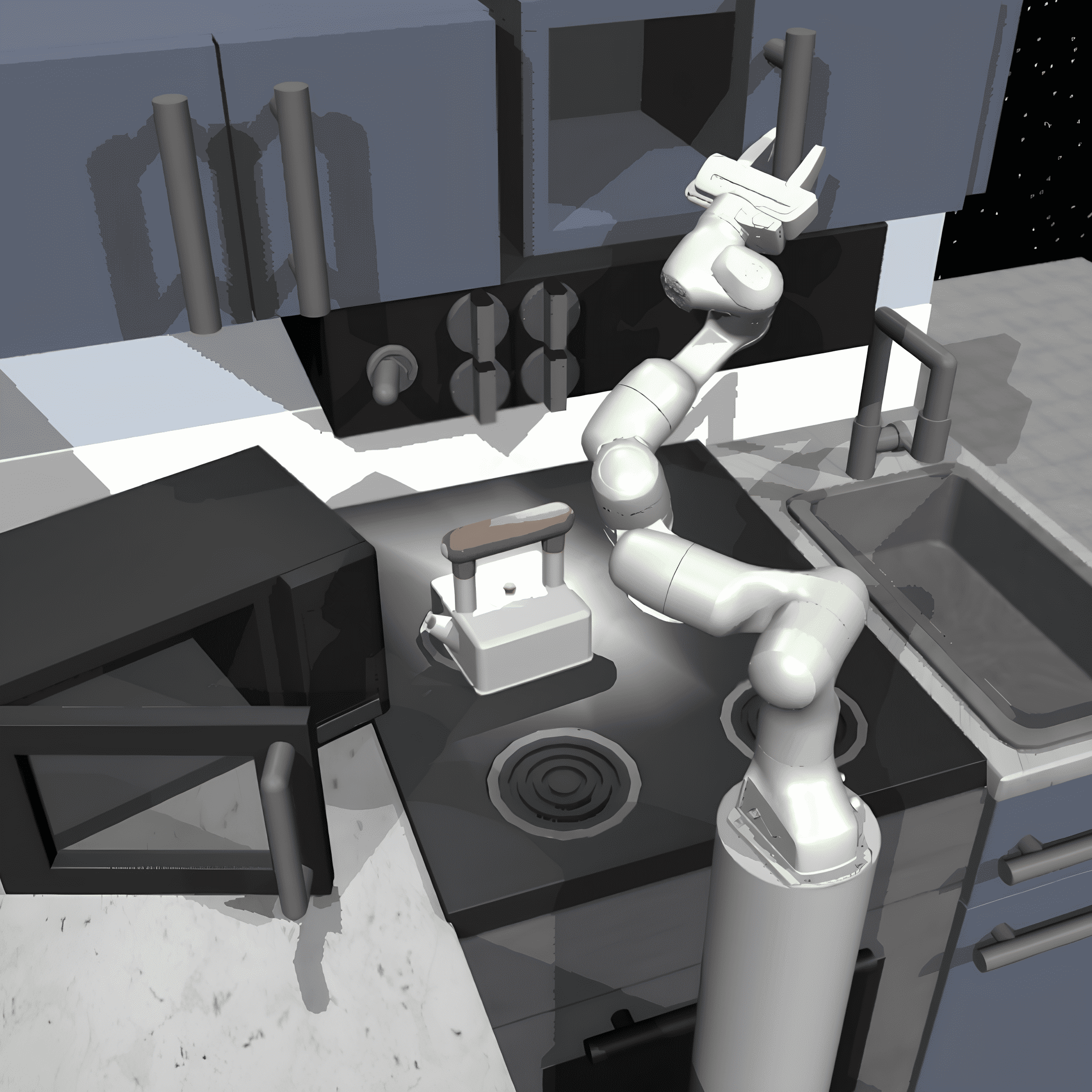}
	}
	\caption{Representative task environments from the D4RL benchmark used for evaluation. (a) Gym-MuJoCo: halfcheetah. (b) Gym-MuJoCo: hopper. (c) Gym-MuJoCo: walker2d. (d) AntMaze: umaze. (e) AntMaze: medium. (f) AntMaze: large. (g) Adroit: pen. (h) FrankaKitchen: kitchen.}
	\label{performance_testing}
\end{figure}

\subsection{Diffusion Actor Improvement}
To achieve the optimal policy described in (\ref{dn2o3lr2390rh205}), the diffusion-based policy network presented in Section \ref{dn29384bbnkett} is further enhanced using \emph{Q}-function guidance. This guidance optimizes the pre-trained diffusion-based policy network parameters by directing the learned policy to prioritize actions with higher \emph{Q}-values, thereby enabling the generation of a learned policy capable of surpassing the performance of the behavior policy. To this end, the loss function for the learned policy is formulated as a weighted combination of the diffusion policy loss, $\mathcal{L}_{\mu}(\phi)$, and the performance improvement loss, $\mathcal{L}_{g}(\phi)$, defined as follows:
\begin{align}\label{adnb329045426}
	\mathcal{L}_{\pi}(\phi) = \mathcal{L}_{\mu}(\phi) + \bar{\varsigma} \mathcal{L}_{g}(\phi),
\end{align}
where $\bar{\varsigma} = \varsigma / \mathbb{E}_{(\bm{s}, \bm{a}) \sim \mathcal{D}} \big| Q_{\theta_1}(\bm{s}, \bm{a}) \big|$ serves as a normalized scaling coefficient, and $\mathcal{L}_{g}(\phi) = - \mathbb{E}_{\bm{s} \sim \mathcal{D}, \bm{a} \sim \pi_{\phi}(\cdot |\bm{s})} \big[ Q_{\theta_1}(\bm{s},\bm{a}) \big]$. 
\begin{remark}
	In order to improve performance on the benchmark task and fully leverage the knowledge embedded in the pretrained behavior policy, the optimization objective (\ref{adnb329045426}) is designed to enable further optimization of the diffusion policy network parameter $\phi$. As presented in Section \ref{dn29384bbnkett}, the pre-trained model provides a strong initialization, facilitating faster convergence and lowering both training time and computational requirements for the target task.
\end{remark}

\section{Experiments}\label{dn32823742rere}
In this section, we present a comprehensive evaluation of the proposed DPBAC algorithm on the D4RL benchmark \cite{FuD4RL}. To demonstrate its effectiveness across a diverse range of tasks, we compare the performance of DPBAC with the SOTA offline RL algorithms. In addition, we conduct sensitivity analysis to investigate the individual contributions of each component within the DPBAC framework. Beyond standard benchmark evaluations, we further examine the ability of DPBAC in solving sparse-reward tasks. Finally, we analyze the expressiveness and intrinsic regularization properties of the diffusion-based policy.

\subsection{Experimental Configurations}
\subsubsection{Dataset for Benchmarking and Evaluation}
We evaluate the proposed DPBAC algorithm on four domains from the D4RL benchmark: Gym-MuJoCo, AntMaze, FrankaKitchen, and Adroit. These environments are depicted in Fig. \ref{performance_testing}, with detailed descriptions provided in Table \ref{Description_of_testing_task_environments}. Gym-MuJoCo consists of continuous control tasks with smooth and dense reward functions, making it a widely adopted benchmark for performance evaluation. AntMaze features sparse rewards and requires the agent to navigate an ant robot through a maze by stitching together suboptimal trajectories. FrankaKitchen involves four sequential sub-tasks, testing the algorithm's ability for long-horizon value optimization and generalization. Adroit, collected from human demonstrations, exhibits narrow state-action distributions and demands strong regularization to prevent policy drift caused by extrapolation errors.

\begin{figure*}[!ht]
	\centering
	\subfigure[halfcheetah-m-v2]{
	\includegraphics[width=3.288cm]{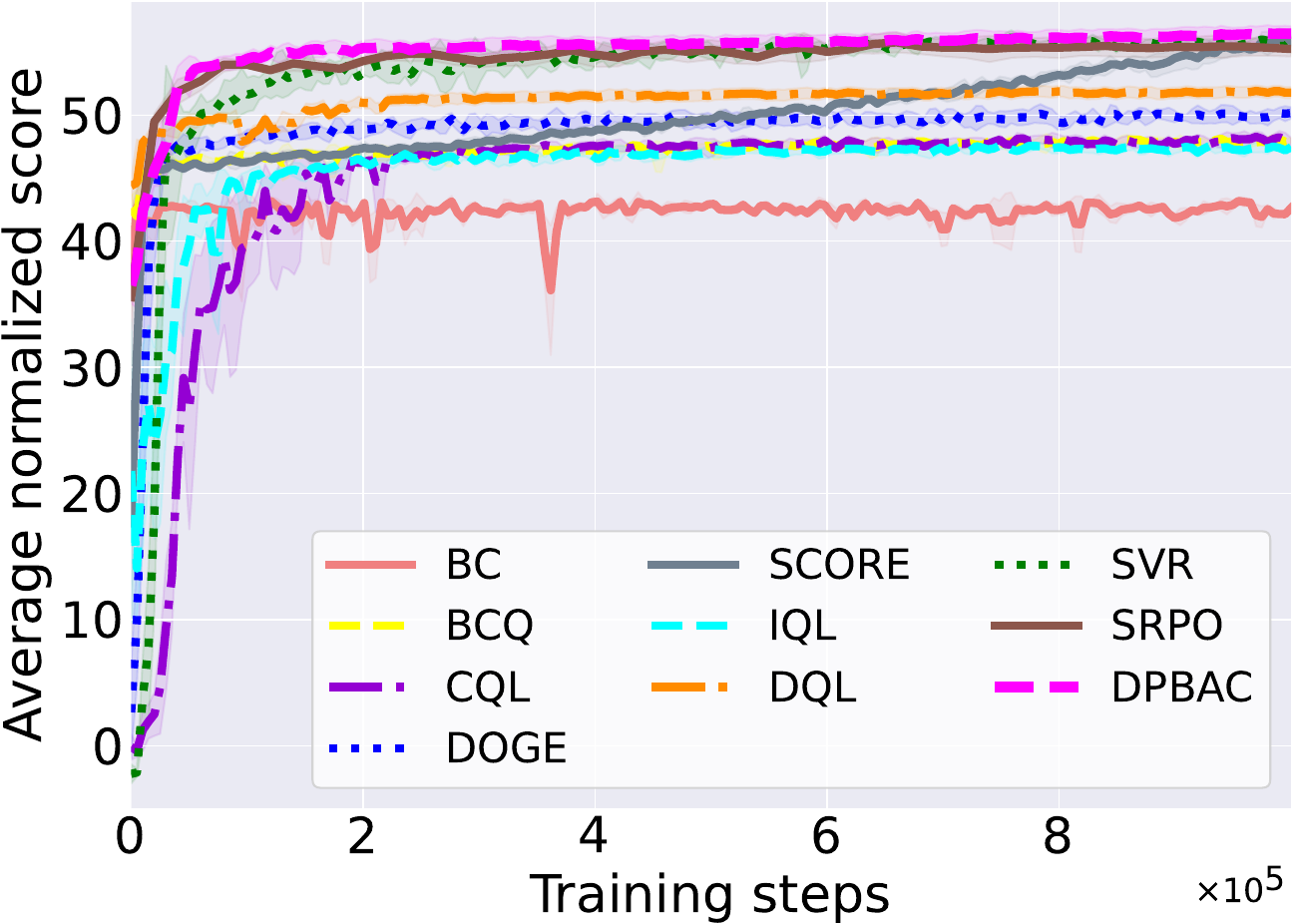}
	}
	\subfigure[hopper-m-v2]{
	\includegraphics[width=3.288cm]{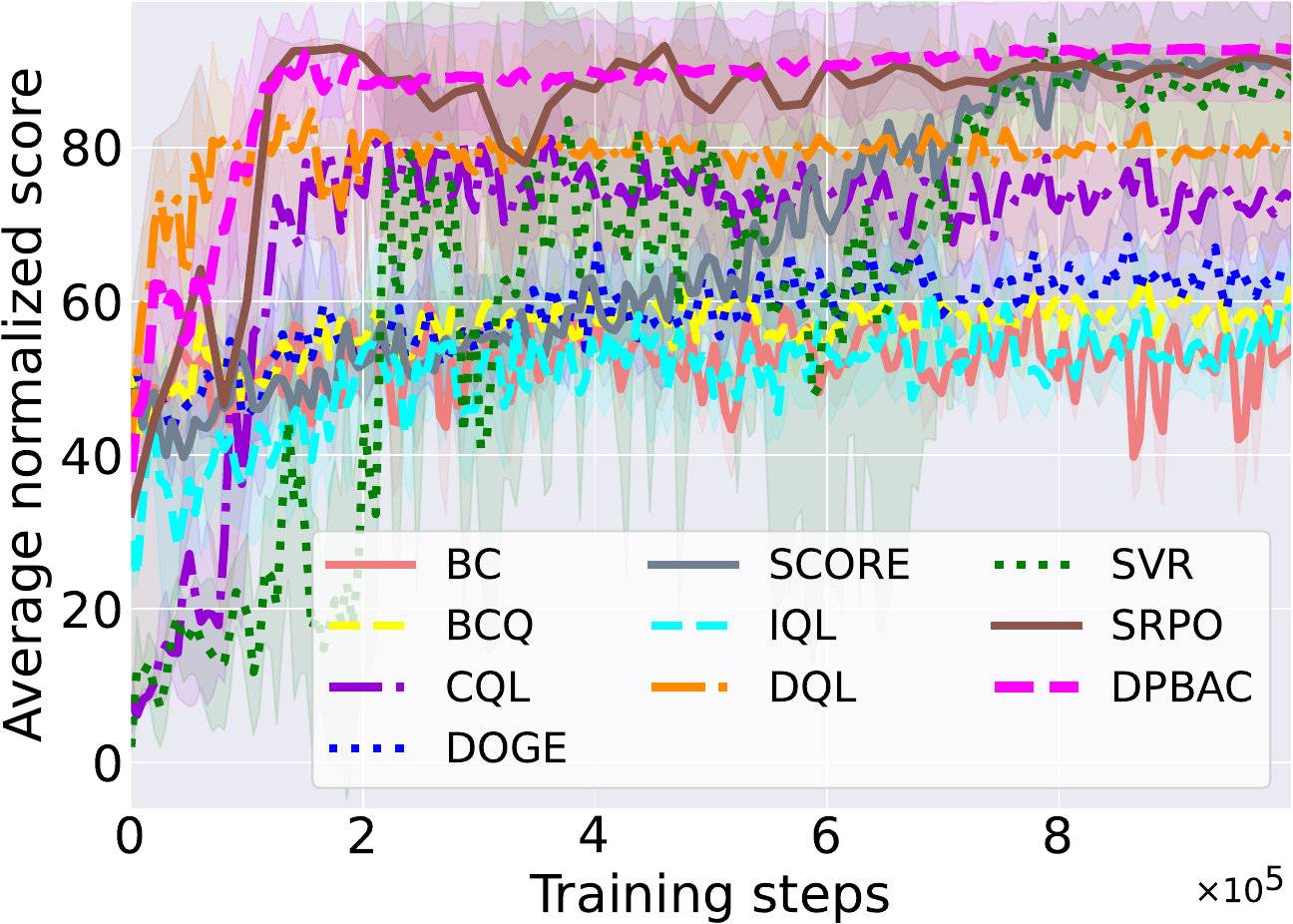}
	}
	\subfigure[walker2d-m-v2]{
	\includegraphics[width=3.288cm]{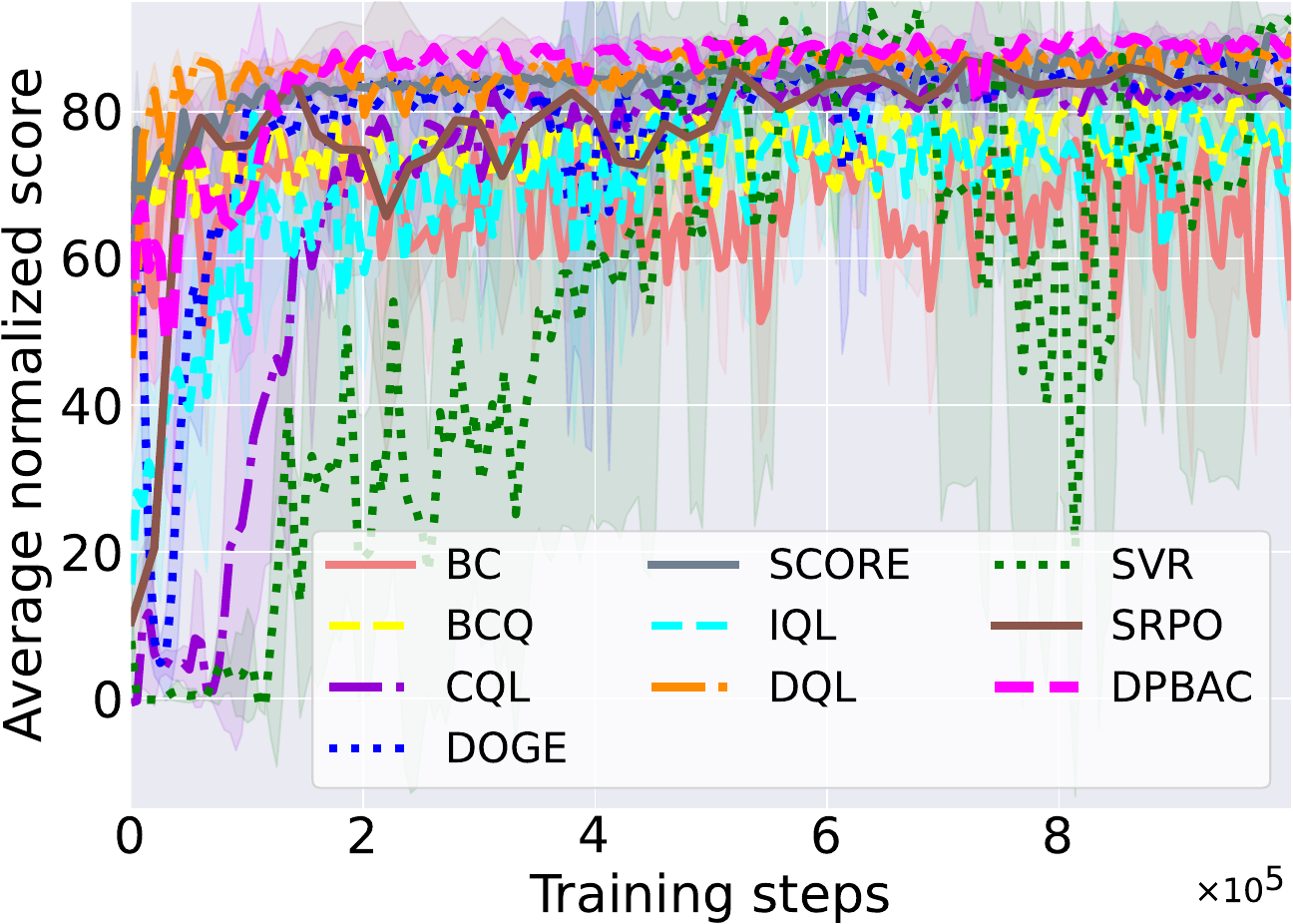}
	}
	\subfigure[halfcheetah-m-r-v2]{
	\includegraphics[width=3.288cm]{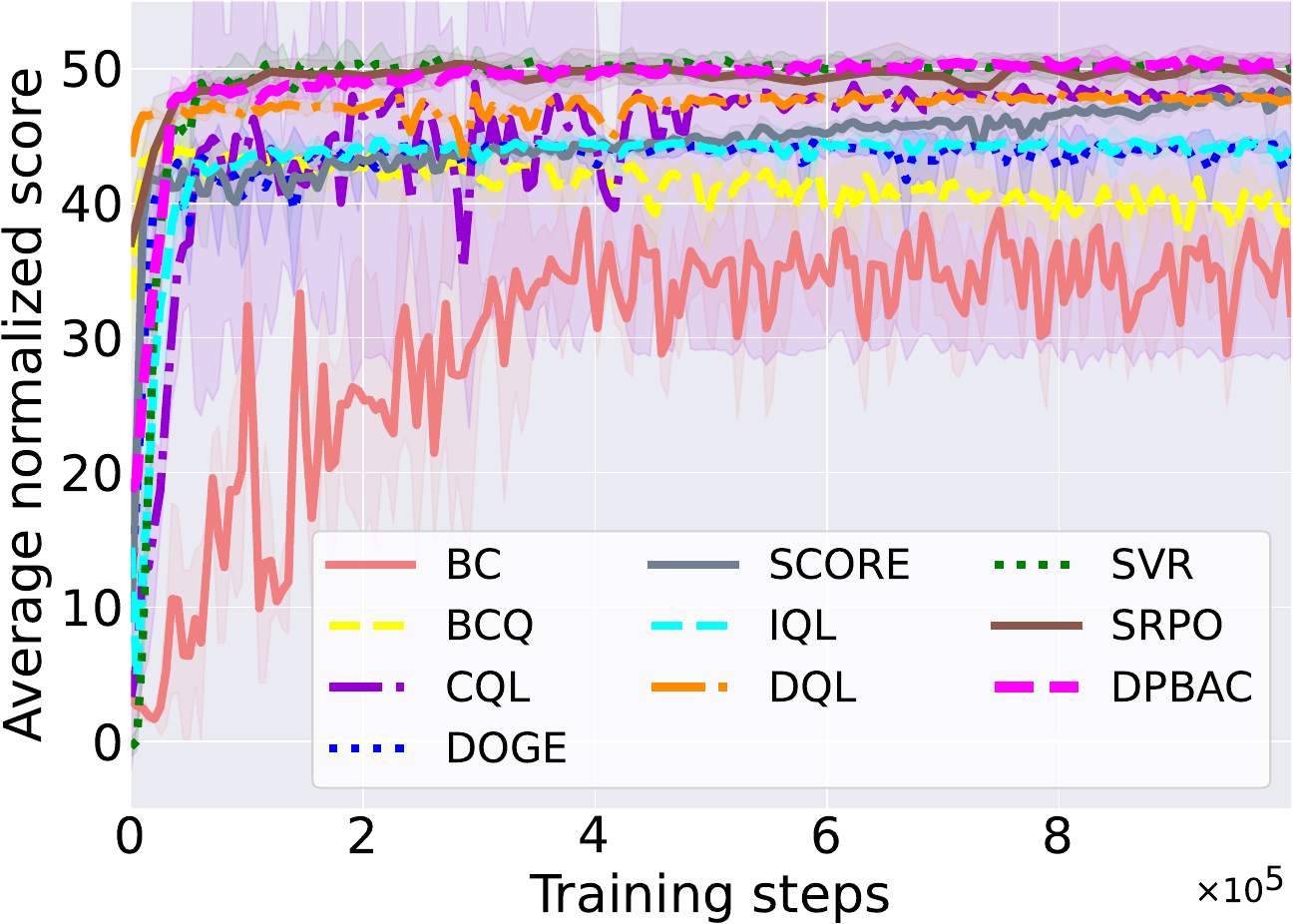}
	}
	\subfigure[hopper-m-r-v2]{
	\includegraphics[width=3.288cm]{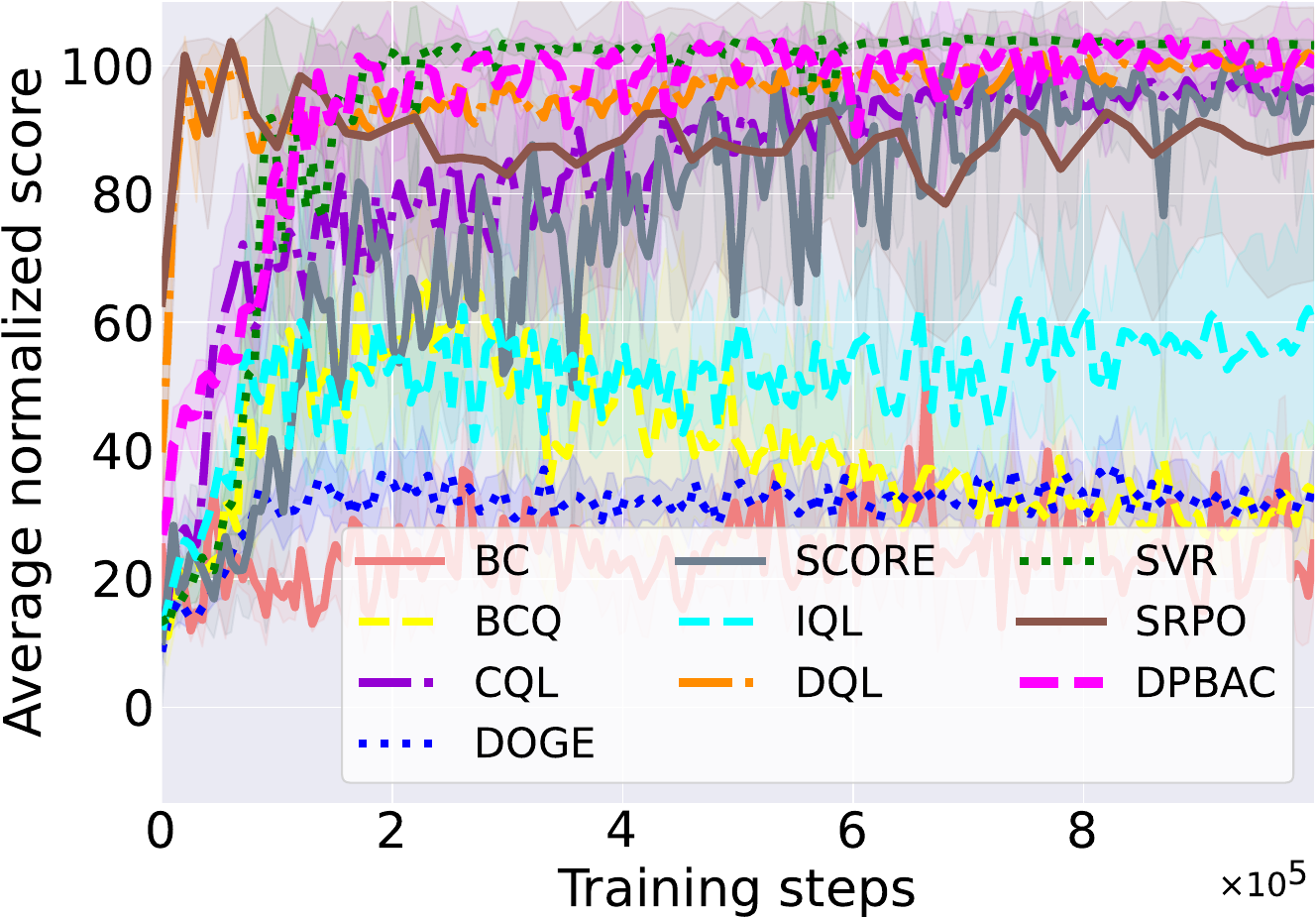}
	}
	\subfigure[walker2d-m-r-v2]{
	\includegraphics[width=3.288cm]{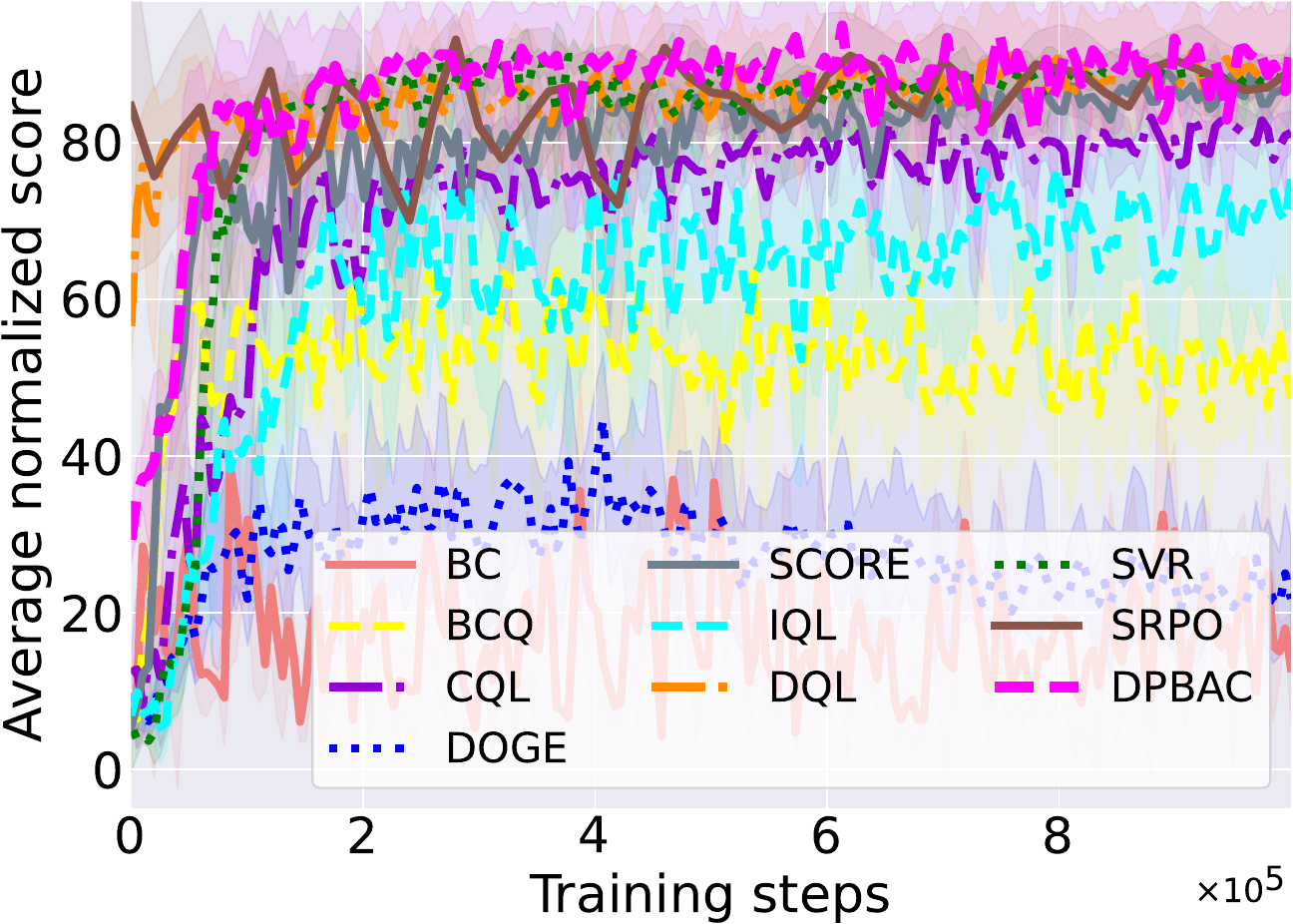}
	}
	\subfigure[halfcheetah-m-e-v2]{
	\includegraphics[width=3.288cm]{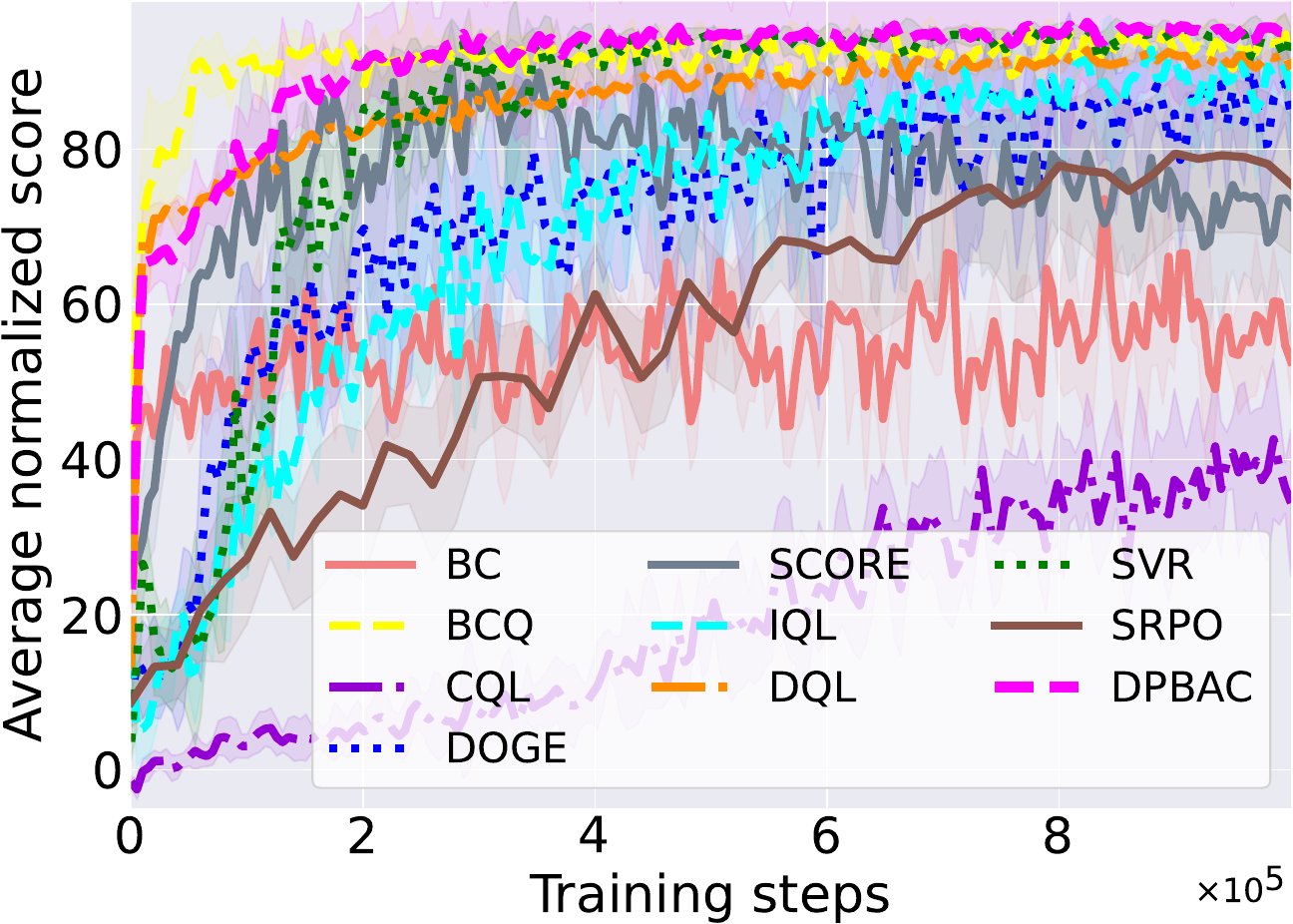}
	}
	\subfigure[hopper-m-e-v2]{
	\includegraphics[width=3.288cm]{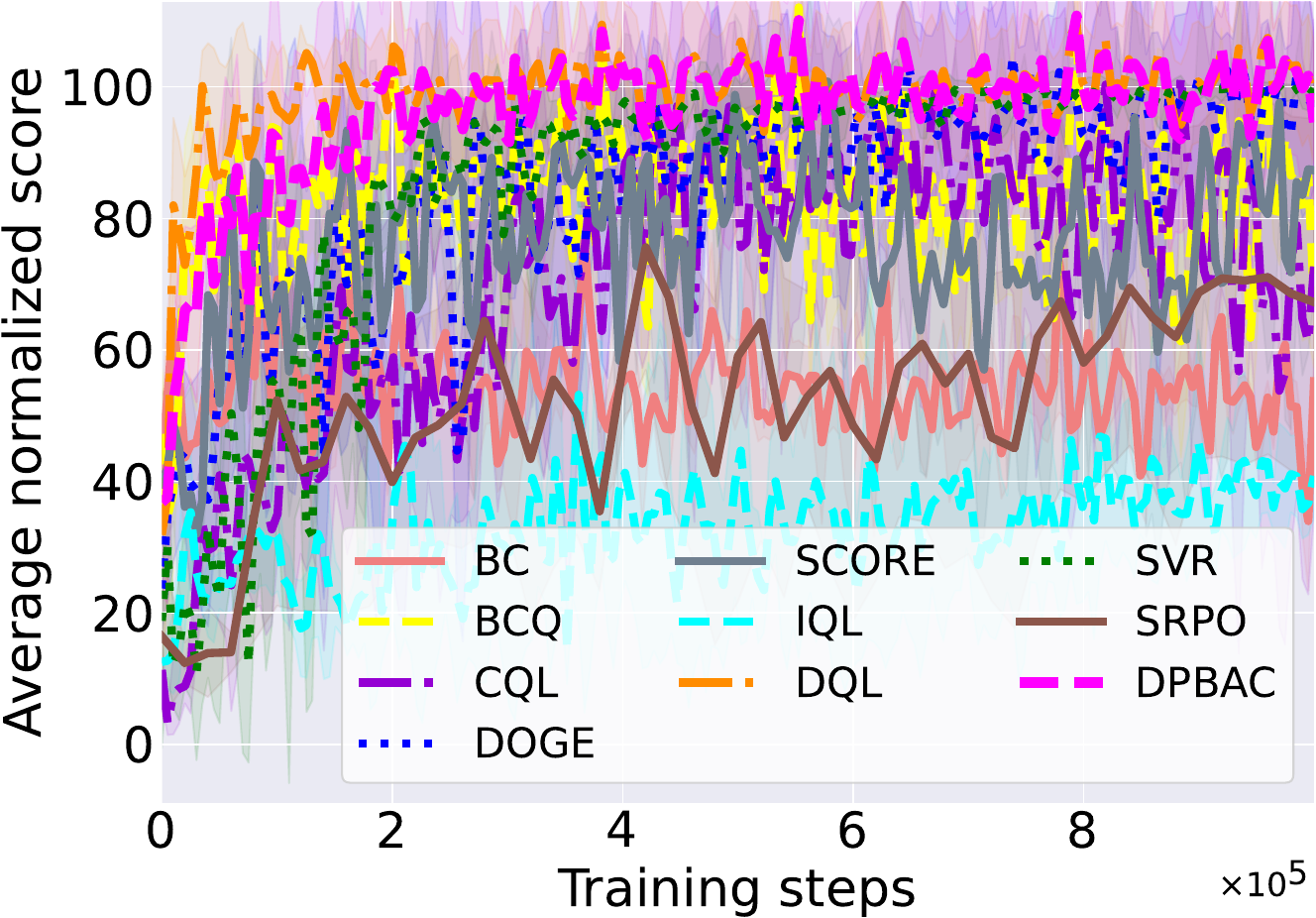}
	}
	\subfigure[walker2d-m-e-v2]{
	\includegraphics[width=3.288cm]{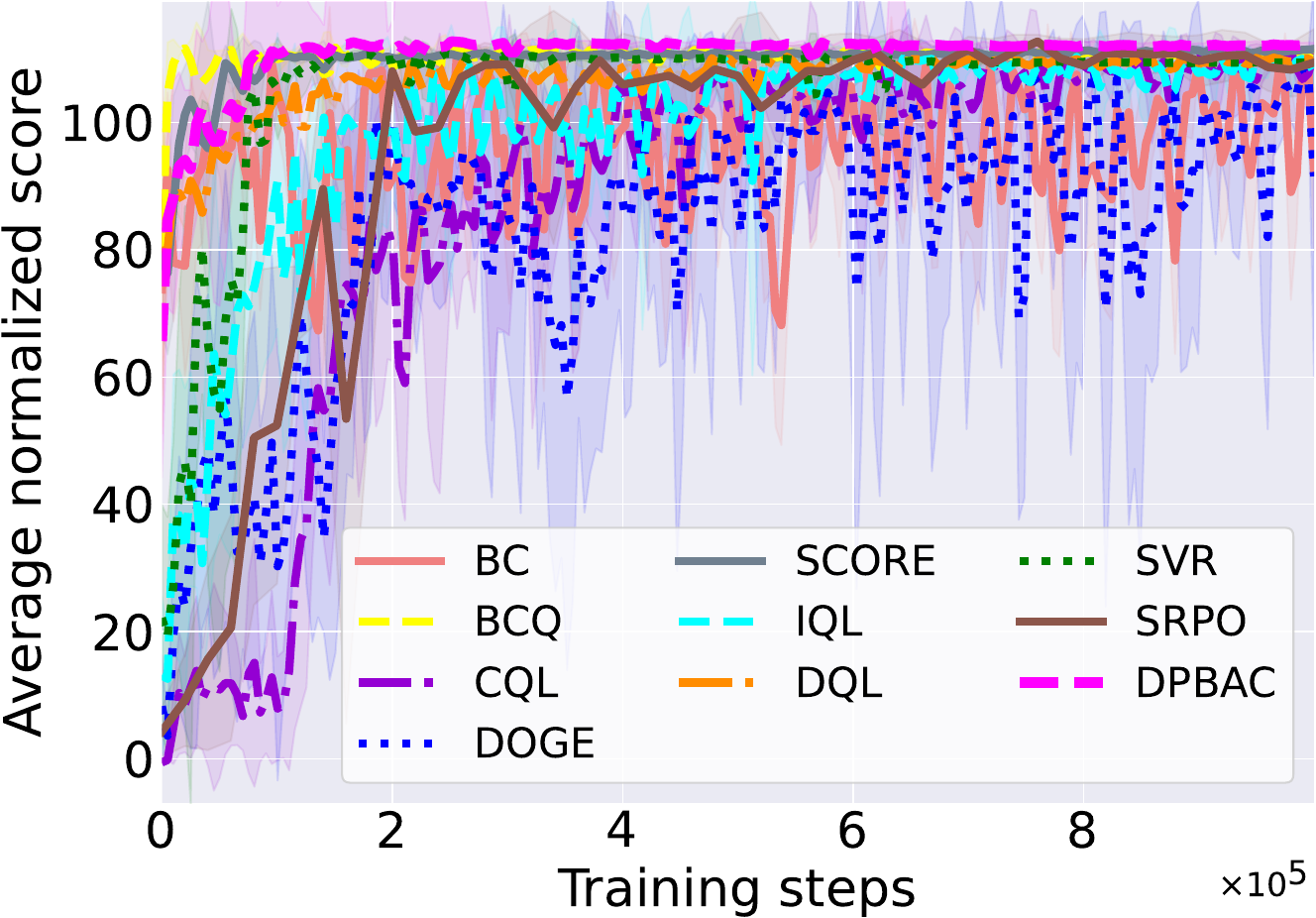}
	}
	\subfigure[kitchen-complete-v0]{
	\includegraphics[width=3.288cm]{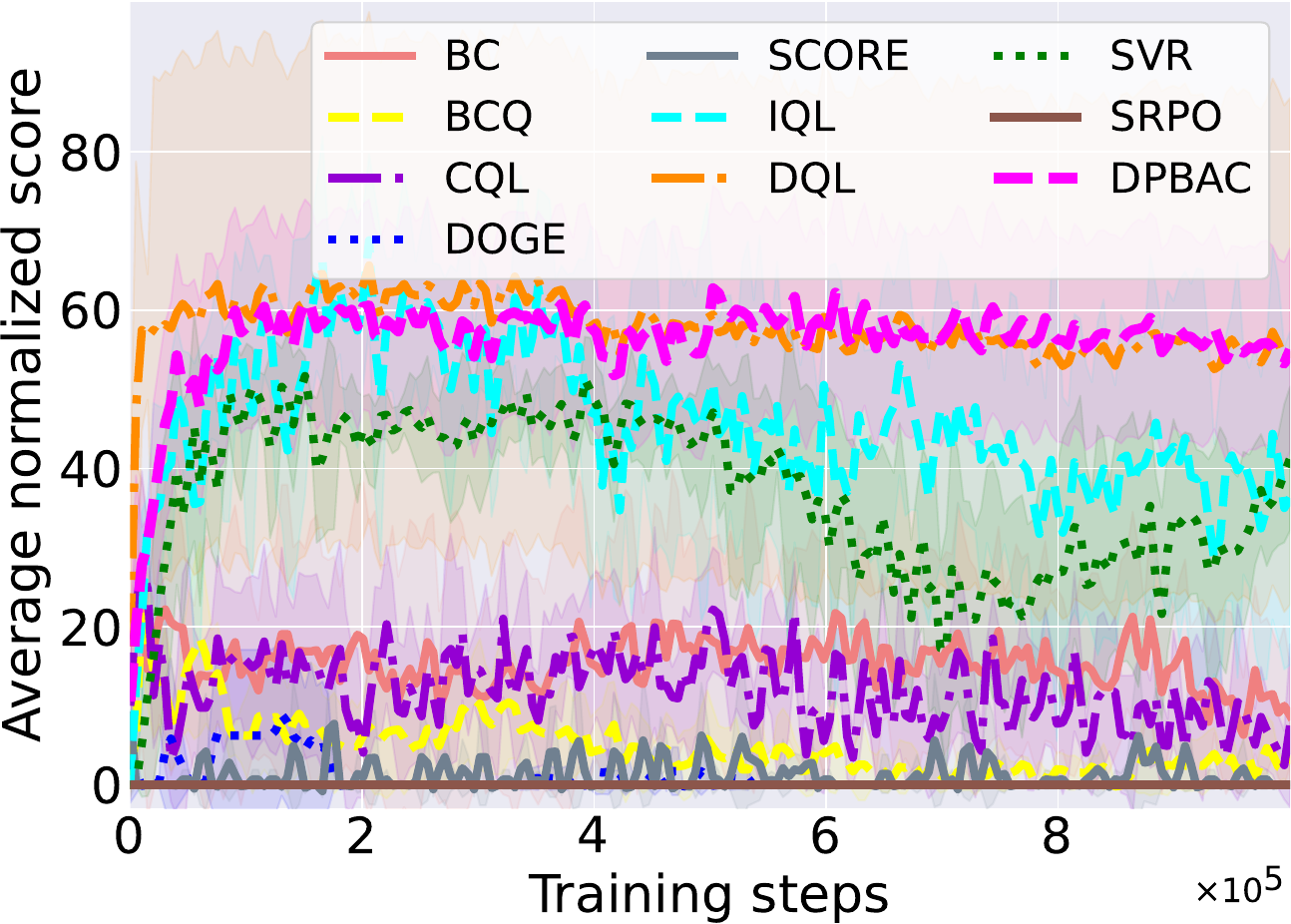}
	}
	\subfigure[kitchen-partial-v0]{
	\includegraphics[width=3.288cm]{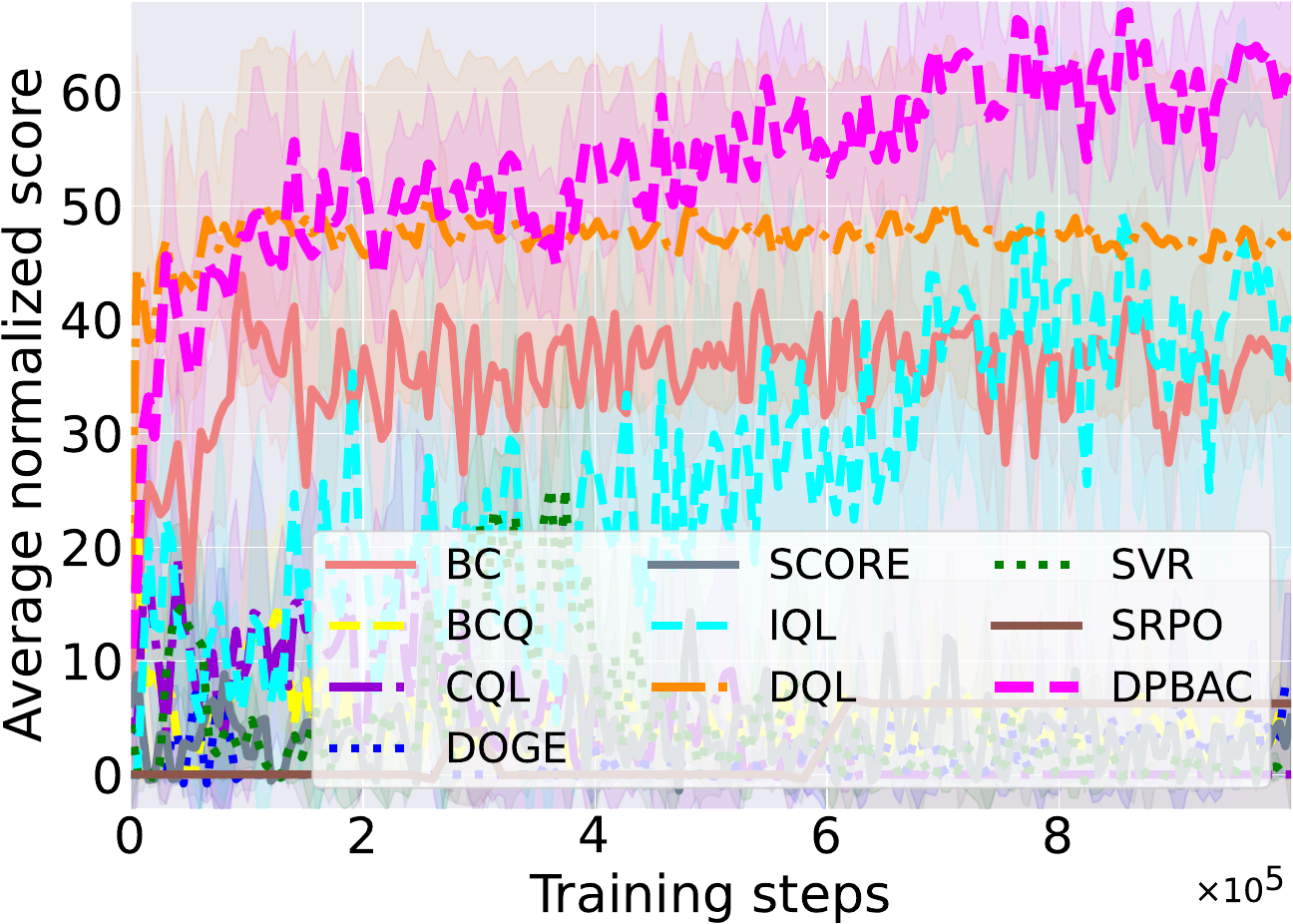}
	}
	\subfigure[kitchen-mixed-v0]{
	\includegraphics[width=3.288cm]{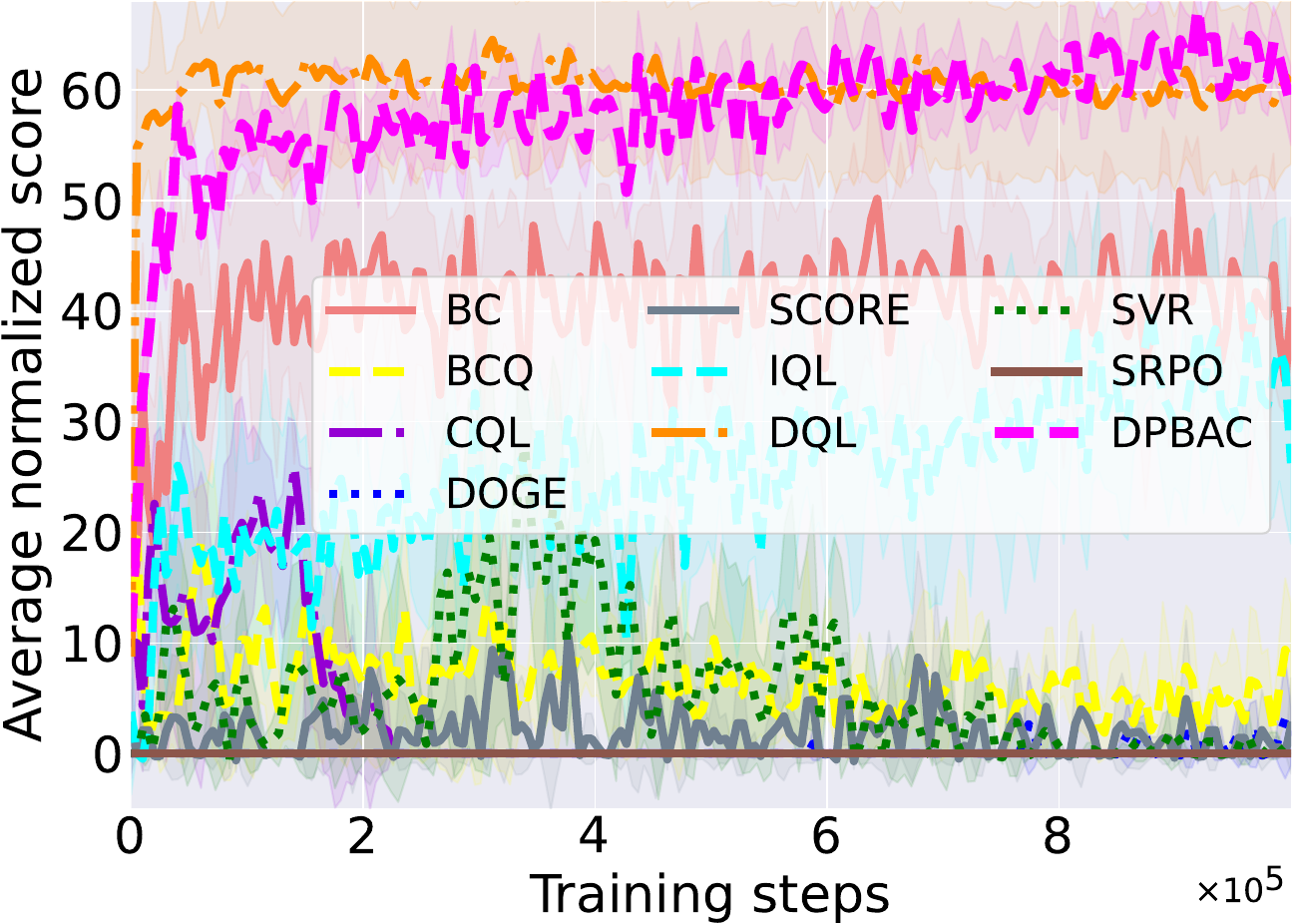}
	}
	\subfigure[antmaze-u-v2]{
	\includegraphics[width=3.288cm]{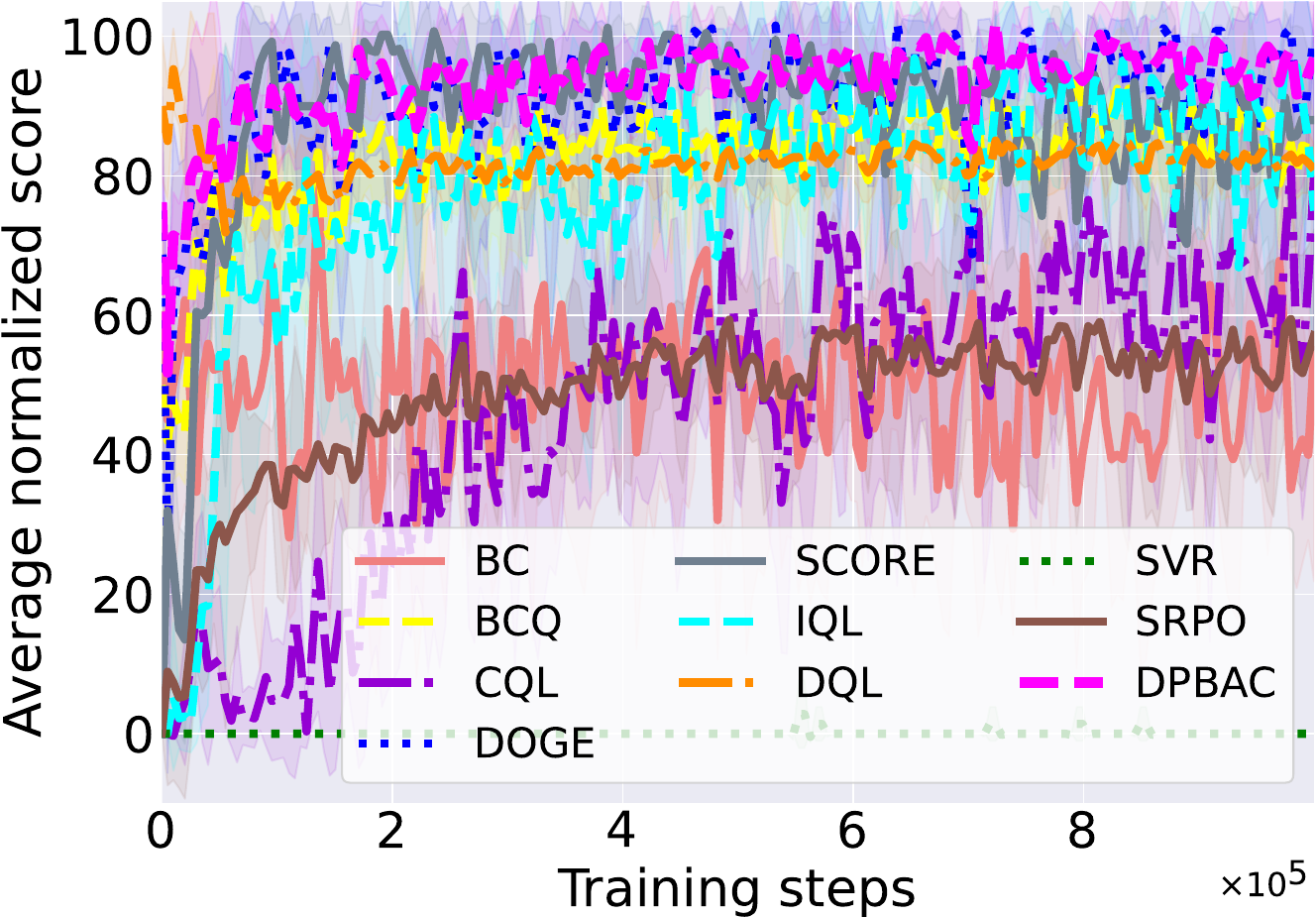}
	}
	\subfigure[antmaze-u-d-v2]{
	\includegraphics[width=3.288cm]{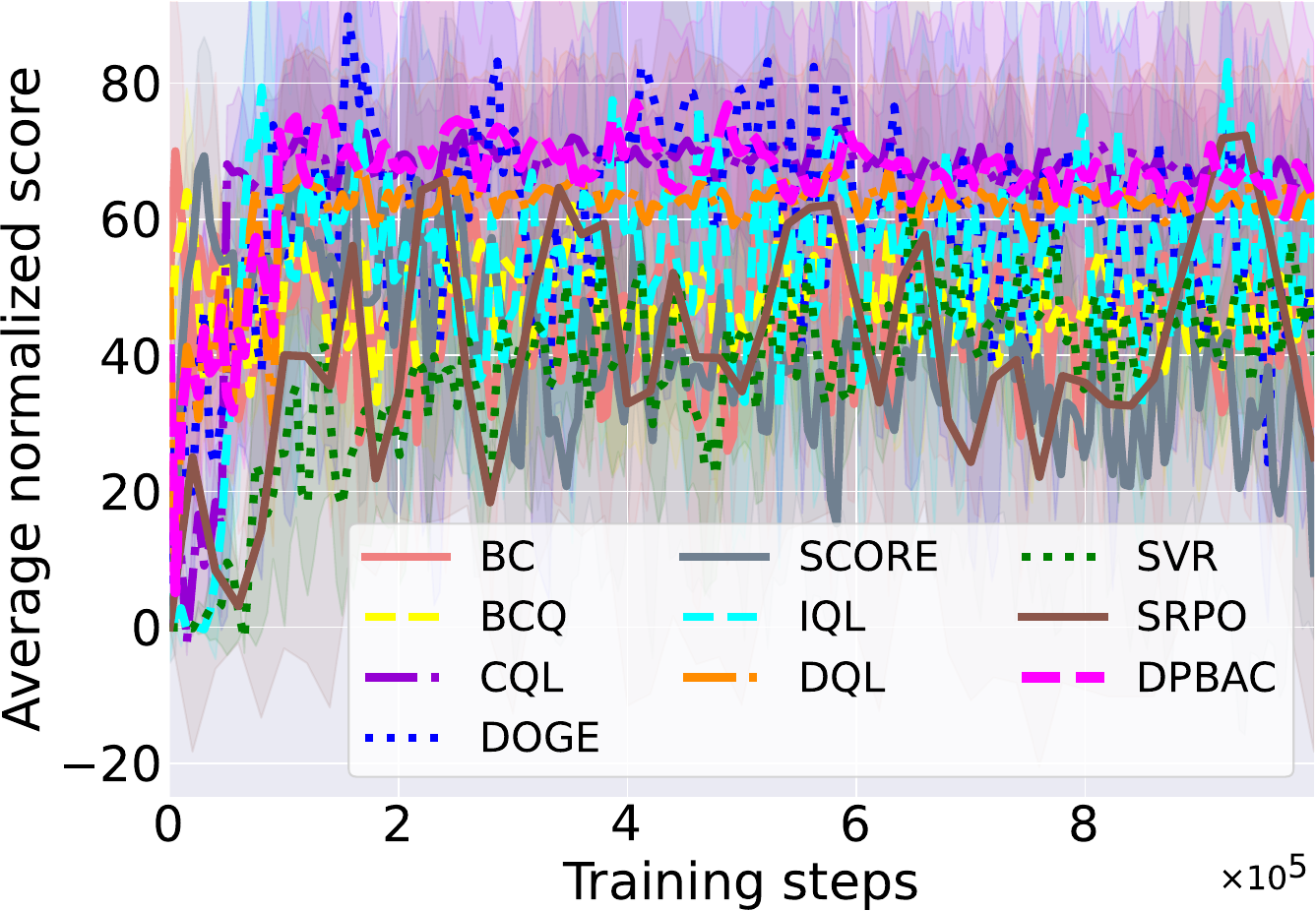}
	}
	\subfigure[antmaze-m-p-v2]{
	\includegraphics[width=3.288cm]{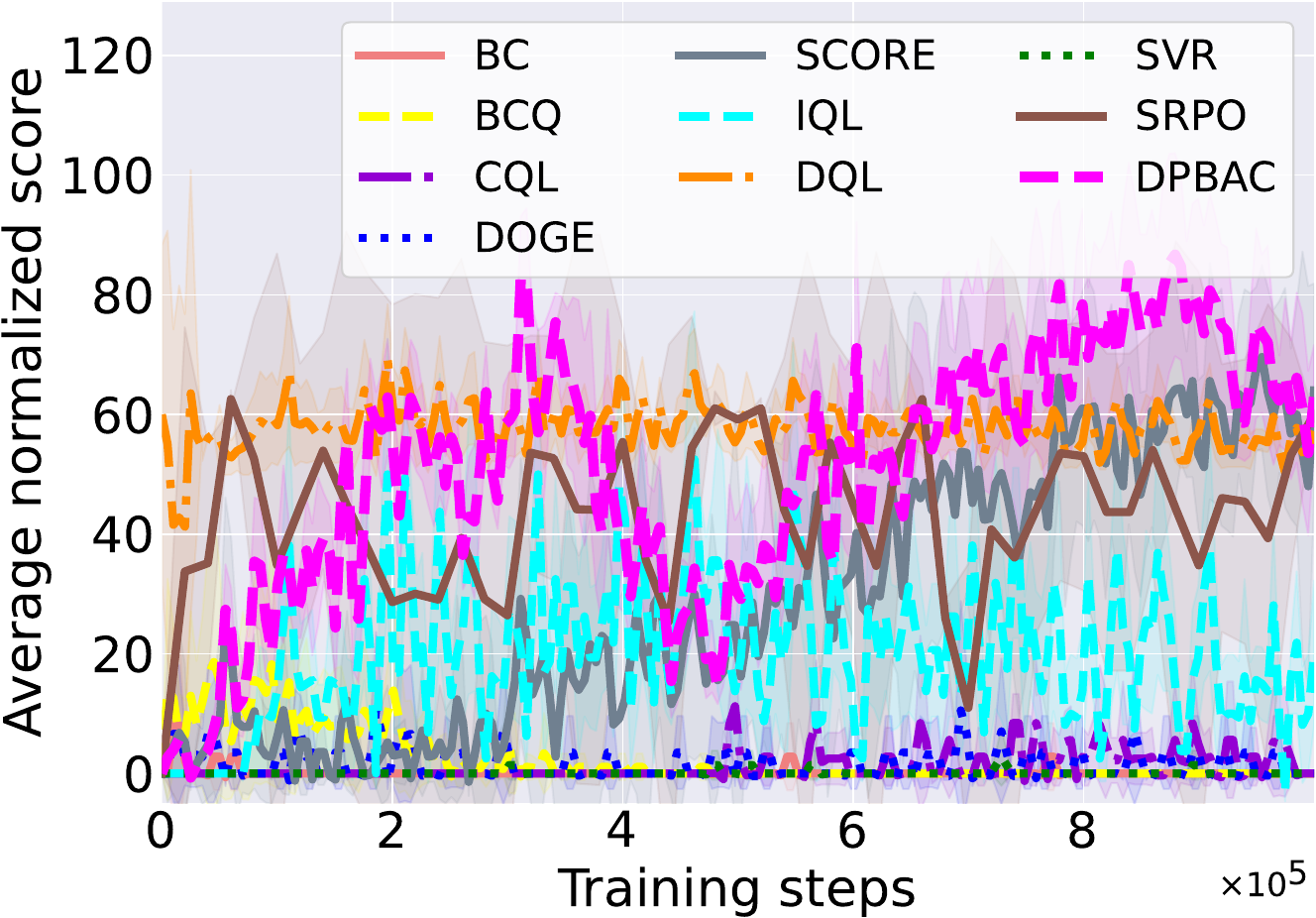}
	}
	\subfigure[antmaze-m-d-v2]{
	\includegraphics[width=3.288cm]{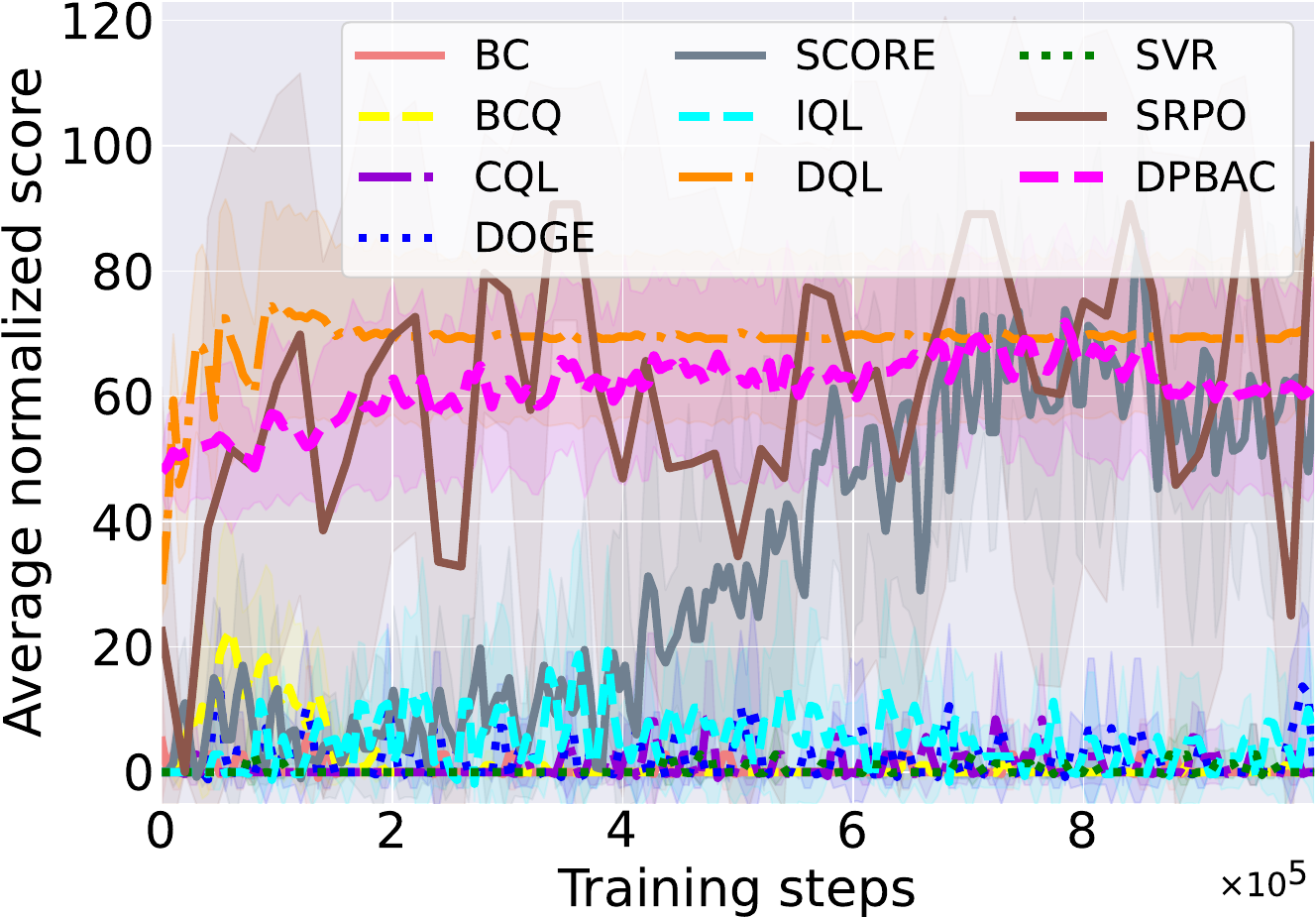}
	}
	\subfigure[antmaze-l-p-v2]{
	\includegraphics[width=3.288cm]{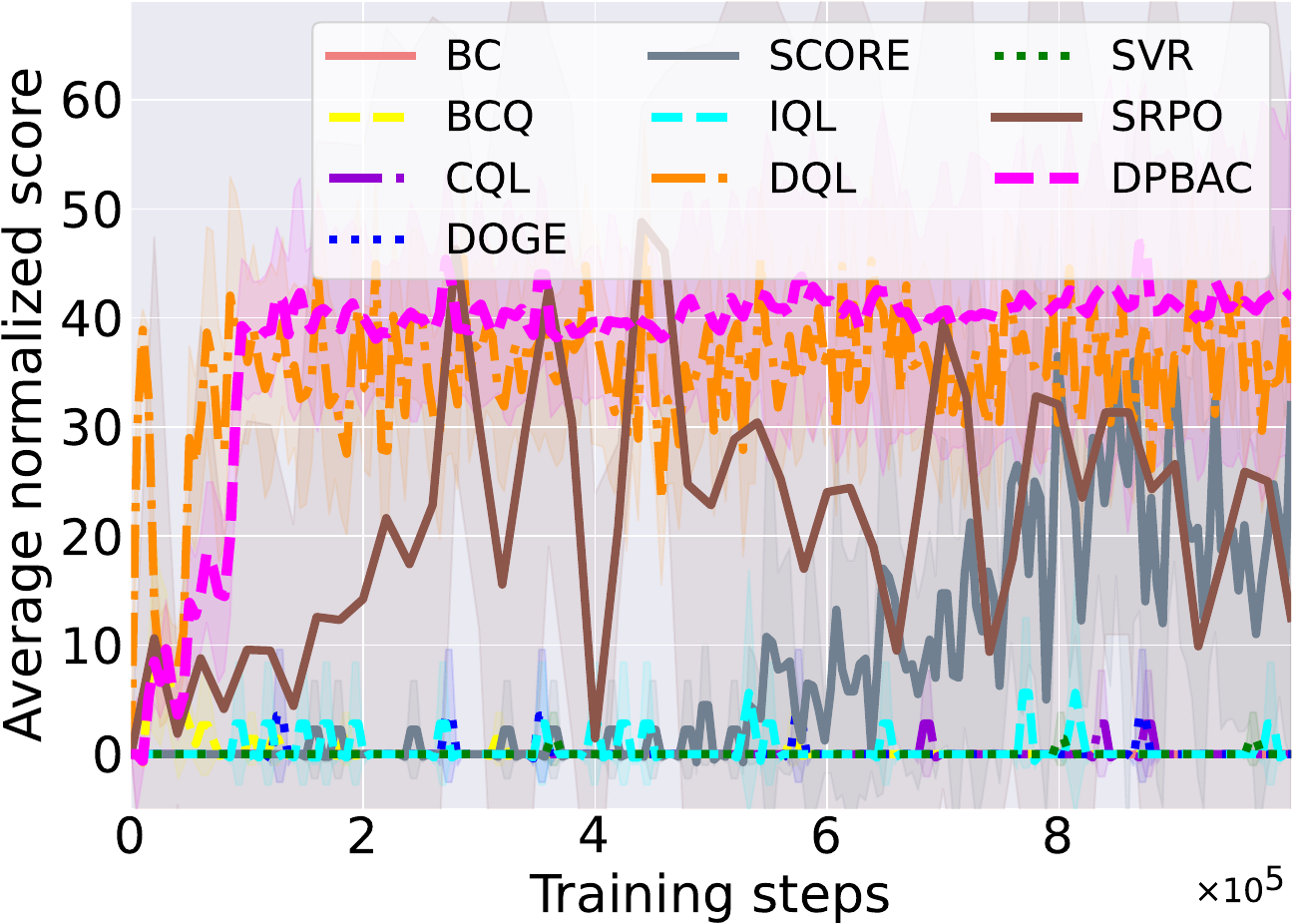}
	}
	\subfigure[antmaze-l-d-v2]{
	\includegraphics[width=3.288cm]{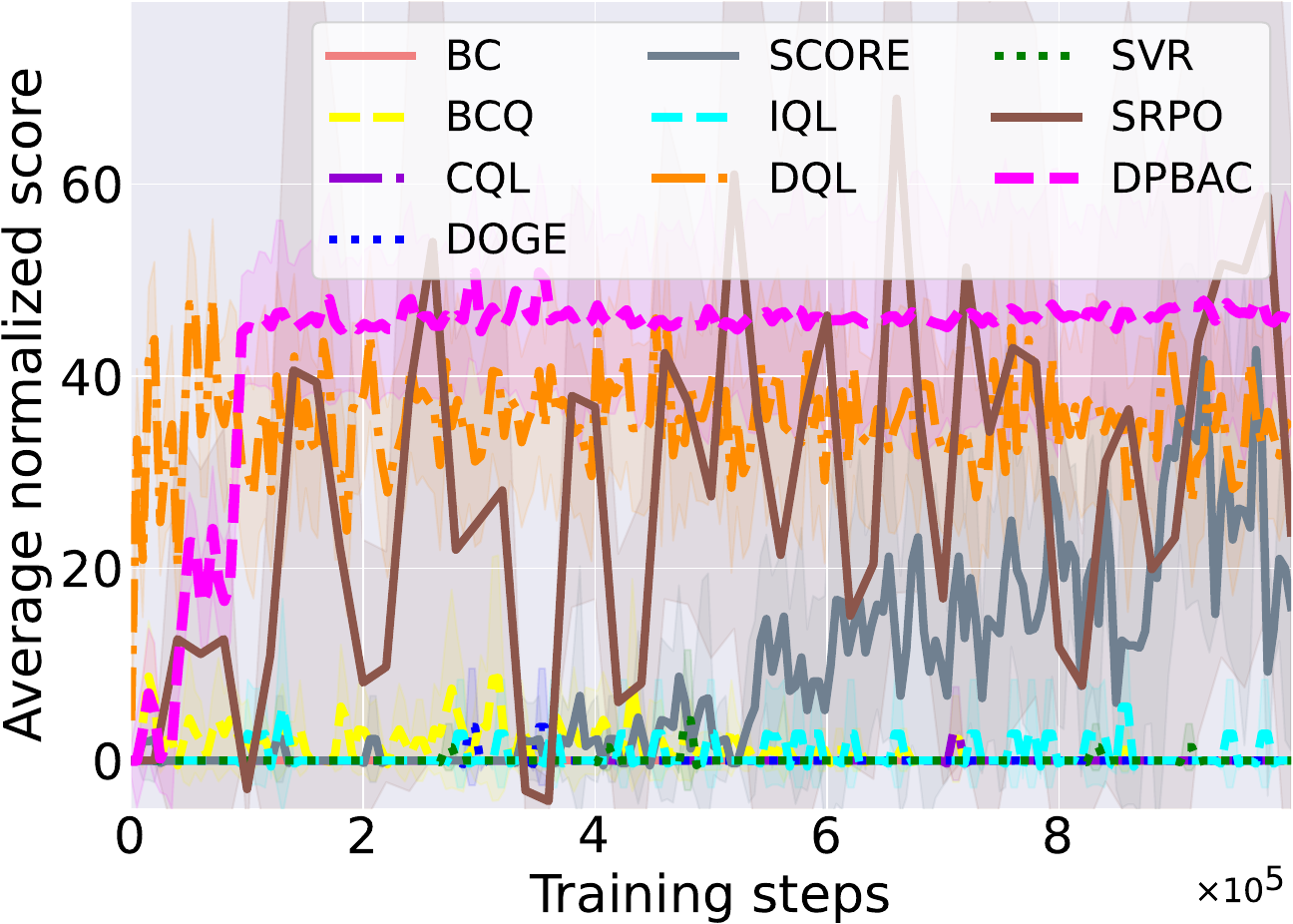}
	}
	\subfigure[pen-human-v1]{
	\includegraphics[width=3.288cm]{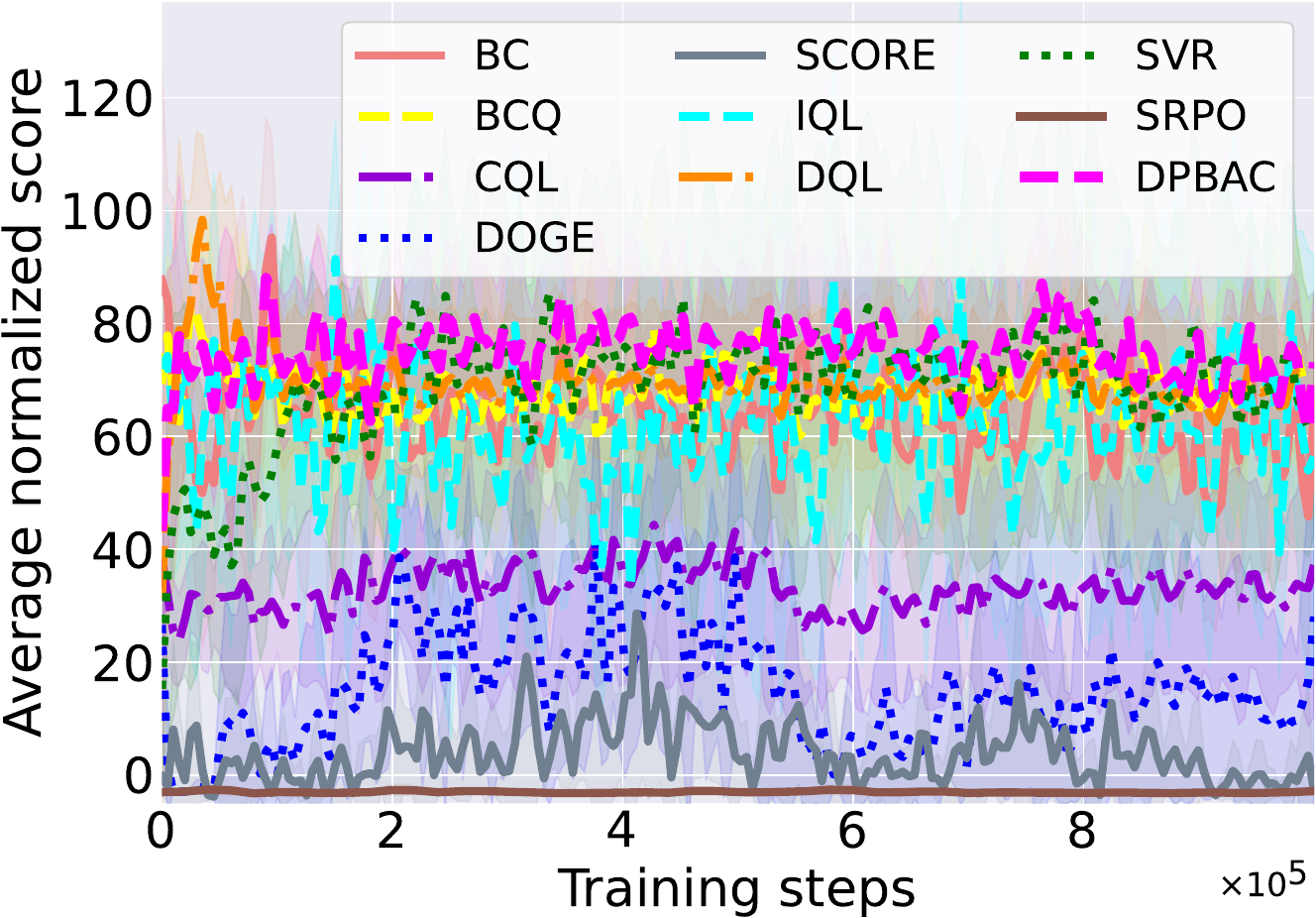}
	}
	\subfigure[pen-cloned-v1]{
	\includegraphics[width=3.288cm]{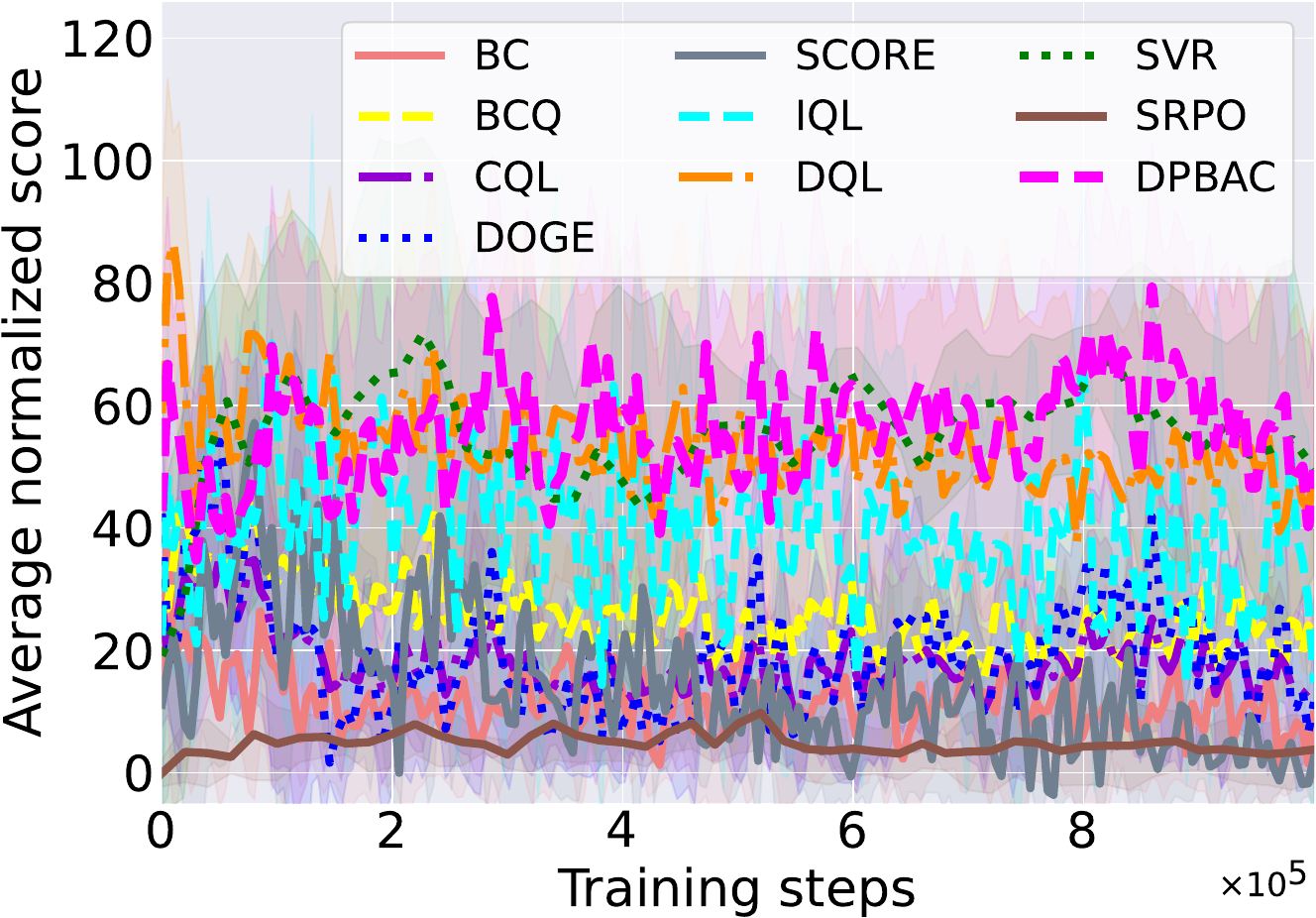}
	}
	\caption{Average normalized scores of DPBAC and SOTA baselines on Gym-MuJoCo, AntMaze, FrankaKitchen, and Adroit tasks. The central line and shaded region for each color represent the mean score curve and the standard deviation, respectively, of the corresponding algorithm across eight random seeds. m=medium, r=replay, e=expert, u=umaze, d=diverse, p=play, l=large.}
	\label{fig_normalized_scores_curves}
\end{figure*}

\begin{table}\footnotesize 
	\renewcommand\arraystretch{1.3}
	\renewcommand\tabcolsep{4.95pt}
	\centering 
	\caption{\centering Overview of testing task environment configurations} 
	\label{Description_of_testing_task_environments} 
	\begin{tabular}{c | c c} 
		\toprule[1pt] 
		\textbf{Task environments} & \textbf{State/Action dimensions} & \textbf{Action range} \\ 
		\midrule 
		Gym-MuJoCo: halfcheetah & $17 \ / \ 6$ & $\big[-1.0, 1.0\big]$ \\ 
		Gym-MuJoCo: hopper & $11 \ / \ 3$ & $\big[-1.0, 1.0\big]$ \\	
		Gym-MuJoCo: walker2d & $17 \ / \ 6$ & $\big[-1.0, 1.0\big]$ \\ 
		AntMaze: umaze & $27 \ / \ 8$ & $\big[-1.0, 1.0\big]$ \\ 
		AntMaze: medium & $27 \ / \ 8$ & $\big[-1.0, 1.0\big]$ \\ 
		AntMaze: large & $27 \ / \ 8$ & $\big[-1.0, 1.0\big]$ \\ 
		Adroit: pen & $45 \ / \ 24$ & $\big[-1.0, 1.0\big]$ \\ 
		FrankaKitchen: kitchen & $59 \ / \ 9$ & $\big[-1.0, 1.0\big]$ \\ 
		\Xhline{1pt} 
	\end{tabular}
\end{table}

\subsubsection{Baseline Offline RL Methods for Comparison}
To ensure a comprehensive and fair evaluation, we compare the proposed method against a broad set of representative baseline algorithms. Specifically, for value regularization approaches, we include CQL\footnote{CQL \cite{KumarA2020}: \url{https://github.com/aviralkumar2907/CQL}.}, SVR\footnote{SVR \cite{MaoY2024}: \url{https://github.com/MAOYIXIU/SVR}.}, IQL\footnote{IQL \cite{KostrikovI2021}: \url{https://github.com/ikostrikov/implicit_q_learning}.}, and SCORE\footnote{SCORE \cite{DengZ2024}: \url{https://github.com/familyld/SCORE}.}. For policy regularization methods, we consider BC\footnote{BC \cite{TorabiF2018}: \url{https://github.com/CherryPieSexy/imitation_learning}.}, BCQ\footnote{BCQ \cite{FujimotoS2019}: \url{https://github.com/sfujim/BCQ}.}, and DOGE\footnote{DOGE \cite{LiJXzhan2023}: \url{https://github.com/Facebear-ljx/DOGE}.}. In addition, we include diffusion-based policy approaches, namely SRPO\footnote{SRPO \cite{ChenH2024}: \url{https://github.com/thu-ml/SRPO}.} and DQL\footnote{DQL \cite{Wang172023}: \url{https://github.com/Zhendong-Wang/Diffusion-Policies-for-Offline-RL}.}, to highlight comparisons with recent diffusion-driven offline RL algorithms. All baseline results are obtained using the official implementations provided by the original authors to ensure reproducibility and consistency.

\subsubsection{Algorithm Implementation Details}
During training, we evaluate the performance of DPBAC every $5000$ steps. To mitigate the impact of randomness, all methods are executed eight times with different random seeds, and the average return is reported as the final result. All experiments are conducted on NVIDIA RTX 4090 GPUs with 24 GB of VRAM. The hyperparameter settings for DPBAC are summarized in Table~\ref{Hyperparameter_settings}. Based on the results of the sensitivity analysis presented in Section~\ref{sdnl23004hugf}, we set the diffusion steps $I_d$ to $6$ for Gym-MuJoCo, $8$ for AntMaze, and $10$ for both FrankaKitchen and Adroit. The regularization weight $\varsigma$ is configured as follows: $5$ for medium and medium-replay datasets, $1$ for medium-expert datasets, $6.5$ for AntMaze, $0.01$ for FrankaKitchen, and $0.1$ for the Adroit domain.

\begin{table}[ht]\footnotesize 
	\renewcommand\arraystretch{1.2}
	\renewcommand\tabcolsep{6pt}
	\centering 
	\caption{\centering Hyperparameters of DPBAC in our experiments} 
	\label{Hyperparameter_settings} 
	\begin{tabular}{c c} 
		\toprule[1pt] 
		\textbf{Hyperparameter} & \textbf{Value} \\ 
		\midrule 
		Target network smoothing parameter $\sigma$ & $5\times 10^{-3}$ \\ 
		\emph{Q}-function network learning rate & $3\times 10^{-4}$ \\
		Policy network learning rate & $3\times 10^{-4}$ \\ 
		Behavior policy training steps $n_{\max}$ & $2 \times 10^{5}$ \\
		Policy training steps $N_{\max}$ & $1 \times 10^{6}$ \\ 
		Diffusion noise schedule & cosine \\ 
		Parameters optimizer & Adam \\	
		Activation functions in all networks & Mish \\
		Correction coefficient $\eta$ & $0.35$ \\
		Mini-batch size & $256$ \\ 
		Hidden layer neuron count & 256 \\
		Target policy update periodicity $\delta$ & 5 \\
		Hidden layer count across all networks & 2 \\ 
		\bottomrule[1pt] 
	\end{tabular}
\end{table}

\subsection{Score Comparison with SOTA RL Algorithms}
We evaluate the performance of DPBAC across four task domains and compare it against a set of baseline algorithms. To ensure a fair comparison, we normalize the scores using the protocol provided in \cite{FuD4RL}, where a score of 100 corresponds to expert-level performance. The performance comparisons are presented in Fig. \ref{fig_normalized_scores_curves} and Table \ref{normalized_scores_curves}. Specifically, Fig. \ref{fig_normalized_scores_curves} illustrates the learning curves of all evaluated algorithms, where solid lines represent the average normalized scores over eight random seeds, and shaded areas indicate the standard deviation. Table \ref{normalized_scores_curves} reports the mean and standard deviation of the normalized scores over the last ten evaluation steps. Bolded values denote the highest scores for each task, while underlined values indicate the second-best results. Fig. \ref{RadarChart} presents a comprehensive comparison of the performance of all baseline methods over the entire suite of tasks. A detailed analysis of the results for each task domain is provided below.

\begin{figure}[!ht]
	\centering
	\includegraphics[width=4.37cm]{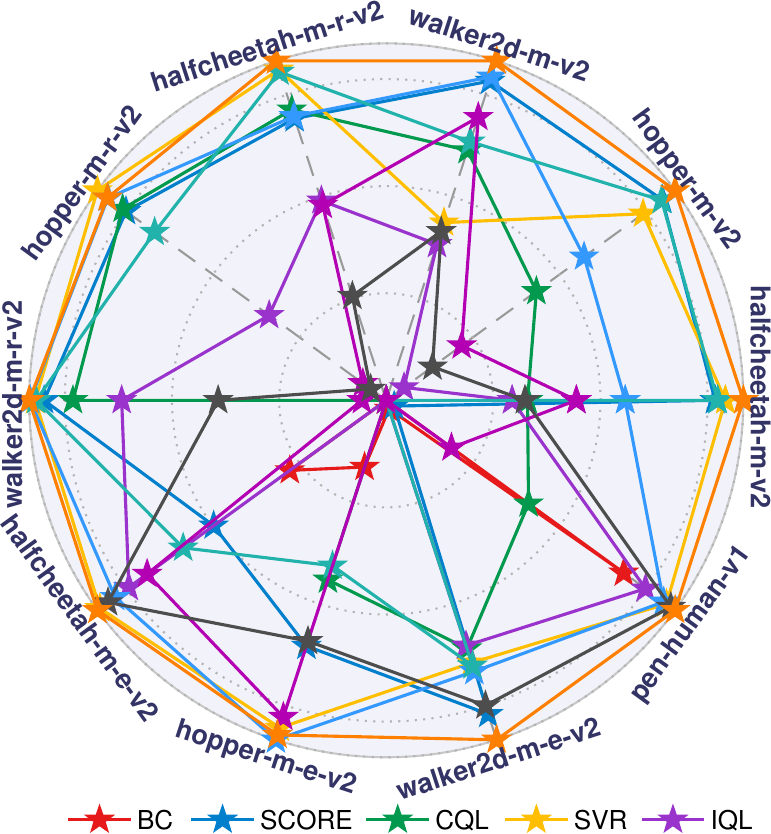}
	\includegraphics[width=4.37cm]{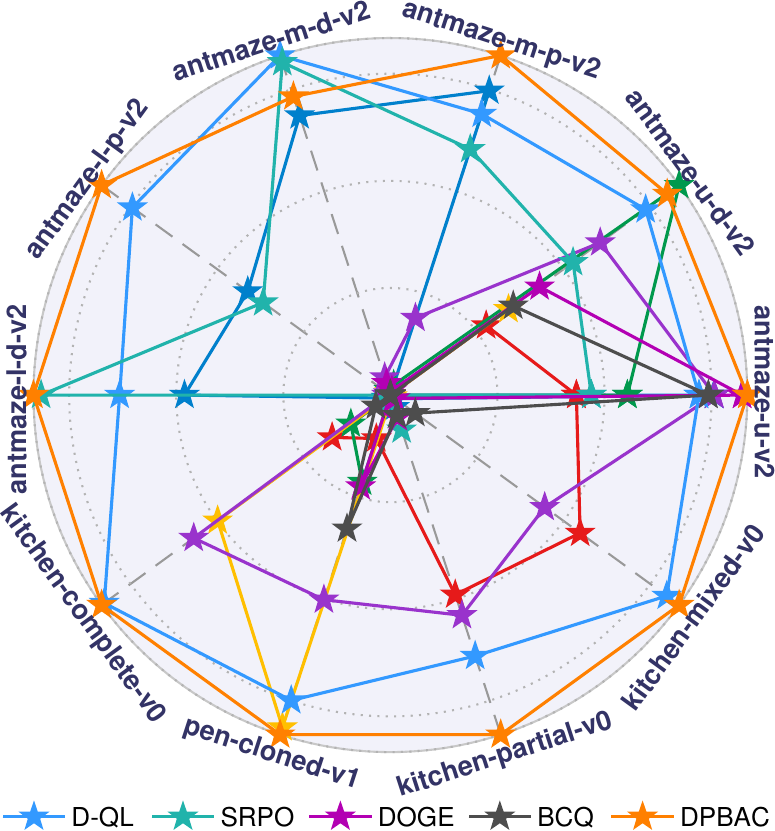}
	\caption{Radar chart visualization of the averaged normalized scores of various comparative algorithms across all evaluation tasks.}
	\label{RadarChart}
\end{figure}

\begin{table*}\footnotesize 
	\renewcommand\arraystretch{1.3}
	\renewcommand\tabcolsep{1.47pt}
	\centering 
	\caption{\centering The average normalized scores and corresponding standard deviations of DPBAC and SOTA baselines on Gym-MuJoCo, AntMaze, FrankaKitchen, and Adroit during the final 10 evaluations across 8 random seeds} 
	\label{normalized_scores_curves} 
	\begin{threeparttable}
	\begin{tabular}{c | | c c c c c c c c c | c} 
		\toprule[1pt] 
		\textbf{Gym-MuJoCo} & \textbf{BC} & \textbf{SCORE} & \textbf{CQL} & \textbf{SVR} & \textbf{IQL} & \textbf{DQL} & \textbf{SRPO} & \textbf{DOGE} & \textbf{BCQ} & \textbf{DPBAC} \\ 
		\midrule 
		halfcheetah-m-v2 & $42.5 \pm 0.4$ & $55.3 \pm 0.7$ & $48.0 \pm 0.4$ & $\underline{55.7} \pm 0.3$ & $47.4 \pm 0.3$ & $51.8 \pm 0.4$ & $55.4 \pm 0.7$ & $49.9 \pm 0.5$ & $47.9 \pm 0.3$ & $\textbf{56.4} \pm 0.6$ \\ 
		hopper-m-v2 & $52.1 \pm 4.1$ & $\underline{90.9} \pm 0.7$ & $73.2 \pm 6.4$ & $88.2 \pm 10.0$ & $54.6 \pm 5.7$ & $79.9 \pm 13.0$ & $90.8 \pm 2.1$ & $62.7 \pm 6.1$ & $58.6 \pm 4.3$ & $\textbf{92.7} \pm 6.7$ \\ 
		walker2d-m-v2 & $65.3 \pm 9.8$ & $87.4 \pm 3.6$ & $82.6 \pm 1.6$ & $77.6 \pm 24.8$ & $76.1 \pm 8.5$ & $\underline{87.7} \pm 0.6$ & $83.2 \pm 5.5$ & $84.9 \pm 3.2$ & $77.0 \pm 4.1$ & $\textbf{88.8} \pm 2.0$ \\ 
		halfcheetah-m-r-v2 & $35.0 \pm 2.9$ & $47.7 \pm 0.8$ & $48.1 \pm 19.4$ & $\underline{49.9} \pm 0.7$ & $44.0 \pm 0.9$ & $47.8 \pm 0.3$ & $49.8 \pm 1.4$ & $43.8 \pm 1.0$ & $39.7 \pm 2.1$ & $\textbf{50.3} \pm 0.8$ \\ 
		hopper-m-r-v2 & $25.6 \pm 8.1$ & $95.6 \pm 7.0$ & $96.5 \pm 2.3$ & $\textbf{103.3} \pm 0.9$ & $57.1 \pm 16.7$ & $100.5 \pm 0.1$ & $87.9 \pm 20.3$ & $31.9 \pm 3.4$ & $29.8 \pm 6.1$ & $\underline{100.6} \pm 5.6$ \\ 
		walker2d-m-r-v2& $17.7 \pm 7.6$ & $86.5 \pm 3.5$ & $80.7 \pm 3.5$ & $88.3 \pm 2.2$ & $70.9 \pm 13.5$ & $\underline{89.3} \pm 7.8$ & $88.1 \pm 4.3$ & $22.6 \pm 7.4$ & $51.5 \pm 9.5$ & $\textbf{89.5} \pm 8.4$ \\ 
		halfcheetah-m-e-v2  & $57.1 \pm 6.1$ & $72.1 \pm 11.8$ & $38.1 \pm 8.2$ & $\underline{94.5} \pm 1.7$ & $88.8 \pm 5.2$ & $91.5 \pm 0.7$ & $78.1 \pm 9.4$ & $85.2 \pm 7.0$ & $92.9 \pm 2.7$ & $\textbf{95.0} \pm 8.8$ \\ 
		hopper-m-e-v2 & $51.5 \pm 4.0$ & $84.5 \pm 23.4$ & $72.3 \pm 35.5$ & $99.6 \pm 1.1$ & $39.2 \pm 7.8$ & $\textbf{101.8} \pm 9.2$ & $69.7 \pm 27.4$ & $97.6 \pm 12.7$ & $83.6 \pm 14.1$ & $\underline{100.9} \pm 15.6$ \\ 
		walker2d-m-e-v2 & $100.5 \pm 7.7$ & $\underline{111.2} \pm 0.7$ & $108.9 \pm 3.2$ & $109.4 \pm 0.1$ & $108.8 \pm 2.5$ & $109.7 \pm 0.1$ & $109.5 \pm 4.2$ & $100.2 \pm 11.2$ & $110.9 \pm 0.4$ & $\textbf{112.1} \pm 0.5$ \\ 
		\hline
		\textbf{Average} & $49.7 \pm 5.6$ & $81.2 \pm 5.8$ & $72.0 \pm 8.9$ & $\underline{85.2} \pm 4.6$ & $65.2 \pm 6.8$ & $84.4 \pm 3.6$ & $79.2 \pm 8.4$ & $64.3 \pm 5.8$ & $65.8 \pm 4.8$ & $\textbf{87.4} \pm 5.4$ \\ 
		\Xhline{1pt} 
		\addlinespace[1.5pt]
		\toprule[1pt] 
		\textbf{AntMaze} & \textbf{BC} & \textbf{SCORE} & \textbf{CQL} & \textbf{SVR} & \textbf{IQL} & \textbf{DQL} & \textbf{SRPO} & \textbf{DOGE} & \textbf{BCQ} & \textbf{DPBAC} \\ 
		\midrule 
		antmaze-u-v2 & $49.6 \pm 18.2$ & $86.3 \pm 13.6$ & $63.2 \pm 18.9$ & $0.0 \pm 0.0$ & $86.4 \pm 14.5$ & $82.3 \pm 4.1$ & $53.6 \pm 16.6$ & $\underline{93.8} \pm 8.5$ & $84.8 \pm 8.7$ & $\textbf{95.2} \pm 10.6$ \\ 
		antmaze-u-d-v2 & $42.8 \pm 22.0$ & $30.8 \pm 32.5$ & $\textbf{67.0} \pm 10.5$ & $45.6 \pm 13.1$ & $57.1 \pm 19.9$ & $62.8 \pm 18.9$ & $53.7 \pm 41.8$ & $49.5 \pm 25.5$ & $46.2 \pm 15.2$ & $\underline{65.5} \pm 23.6$ \\ 
		antmaze-m-p-v2 & $0.0 \pm 0.0$ & $\underline{60.4} \pm 18.3$ & $2.0 \pm 3.5$ & $0.0 \pm 0.0$ & $15.3 \pm 10.1$ & $55.8 \pm 3.0$ & $48.8 \pm 20.9$ & $0.9 \pm 1.6$ & $0.0 \pm 0.0$ & $\textbf{67.4} \pm 12.8$ \\ 
		antmaze-m-d-v2 & $0.0 \pm 0.0$ & $57.2 \pm 22.6$ & $1.0 \pm 1.7$ & $0.6 \pm 1.1$ & $3.6 \pm 5.8$ & $\textbf{69.5} \pm 13.0$ & $\underline{68.1} \pm 26.0$ & $2.4 \pm 3.0$ & $0.1 \pm 0.2$ & $61.1 \pm 16.1$ \\ 
		antmaze-l-p-v2 & $0.0 \pm 0.0$ & $20.4 \pm 18.1$ & $0.0 \pm 0.0$ & $0.1 \pm 0.2$ & $0.3 \pm 0.5$ & $\underline{36.9} \pm 8.0$ & $18.2 \pm 45.8$ & $0.0 \pm 0.0$ & $0.0 \pm 0.0$ & $\textbf{41.3} \pm 14.8$ \\ 
		antmaze-l-d-v2 & $0.0 \pm 0.0$ & $27.0 \pm 15.2$ & $0.0 \pm 0.0$ & $0.1 \pm 0.2$ & $0.5 \pm 1.0$ & $35.5 \pm 9.3$ & $\underline{45.8} \pm 41.9$ & $0.0 \pm 0.0$ & $0.0 \pm 0.0$ & $\textbf{46.8} \pm 12.0$ \\ 
		\hline
		\textbf{Average} & $15.4 \pm 6.7$ & $47.0 \pm 20.1$ & $22.2 \pm 5.8$ & $7.7 \pm 2.4$ & $27.2 \pm 8.6$ & $\underline{57.1} \pm 9.4$ & $48.0 \pm 32.2$ & $24.4 \pm 6.4$ & $21.9 \pm 4.0$ & $\textbf{62.9} \pm 15.0$ \\ 
		\Xhline{1pt} 
		\addlinespace[1.5pt]
		\toprule[1pt] 
		\textbf{FrankaKitchen} & \textbf{BC} & \textbf{SCORE} & \textbf{CQL} & \textbf{SVR} & \textbf{IQL} & \textbf{DQL} & \textbf{SRPO} & \textbf{DOGE} & \textbf{BCQ} & \textbf{DPBAC} \\ 
		\midrule 
		kitchen-complete-v0 & $11.2 \pm 9.2$ & $0.8 \pm 1.1$ & $7.6 \pm 6.9$ & $33.2 \pm 10.2$ & $37.8 \pm 20.7$ & $\underline{55.0} \pm 32.0$ & $0.0 \pm 0.0$ & $0.0 \pm 0.0$ & $2.8 \pm 2.9$ & $\textbf{55.5} \pm 12.9$ \\ 
		kitchen-partial-v0 & $35.8 \pm 6.0$ & $2.2 \pm 2.9$ & $0.0 \pm 0.0$ & $0.5 \pm 0.8$ & $39.5 \pm 17.3$ & $\underline{46.8} \pm 14.7$ & $6.2 \pm 10.8$ & $4.0 \pm 5.1$ & $3.6 \pm 2.7$ & $\textbf{60.9} \pm 6.6$ \\ 
		kitchen-mixed-v0 & $41.2 \pm 8.5$ & $1.1 \pm 1.8$ & $0.0 \pm 0.0$ & $0.4 \pm 0.6$ & $33.5 \pm 9.6$ & $\underline{60.0} \pm 7.9$ & $0.0 \pm 0.0$ & $1.0 \pm 1.6$ & $5.4 \pm 4.2$ & $\textbf{62.7} \pm 3.3$ \\ 
		\hline
		\textbf{Average} & $29.4 \pm 7.9$ & $1.4 \pm 1.9$ & $2.5 \pm 2.3$ & $11.4 \pm 3.9$ & $36.9 \pm 15.9$ & $\underline{53.9} \pm 18.2$ & $2.1 \pm 3.6$ & $1.7 \pm 2.2$ & $3.9 \pm 3.3$ & $\textbf{59.7} \pm 7.6$ \\ 
		\Xhline{1pt} 
		\addlinespace[1.5pt]
		\toprule[1pt] 
		\textbf{Adroit} & \textbf{BC} & \textbf{SCORE} & \textbf{CQL} & \textbf{SVR} & \textbf{IQL} & \textbf{DQL} & \textbf{SRPO} & \textbf{DOGE} & \textbf{BCQ} & \textbf{DPBAC} \\ 
		\midrule 
		pen-human-v1 & $57.1 \pm 20.3$ & $-1.0 \pm 4.1$ & $33.0 \pm 18.9$ & $67.9 \pm 30.1$ & $62.6 \pm 28.3$ & $67.2 \pm 12.8$ & $-3.0 \pm 0.2$ & $13.5 \pm 20.2$ & $\underline{69.0} \pm 17.5$ & $\textbf{70.1} \pm 12.9$ \\ 
		pen-cloned-v1 & $9.3 \pm 12.9$ & $2.7 \pm 7.5$ & $15.9 \pm 12.1$ & $\underline{53.2} \pm 22.1$ & $33.8 \pm 24.3$ & $49.1 \pm 27.2$ & $3.5 \pm 5.3$ & $16.6 \pm 11.7$ & $23.0 \pm 12.6$ & $\textbf{54.3} \pm 23.1$ \\ 
		\hline
		\textbf{Average} & $33.2 \pm 16.6$ & $0.9 \pm 5.8$ & $24.5 \pm 15.5$ & $\underline{60.6} \pm 26.1$ & $48.2 \pm 26.3$ & $58.2 \pm 20.0$ & $0.3 \pm 2.8$ & $15.1 \pm 16.0$ & $46.0 \pm 15.1$ & $\textbf{62.2} \pm 18.0$ \\ 
		\Xhline{1pt} 
	\end{tabular}
	\begin{tablenotes}
        \small
        \item[1] In the Gym-MuJoCo suite of tasks, m denotes medium, r stands for replay, and e refers to expert.
        \item[2] In the AntMaze suite of tasks, u represents umaze, d denotes diverse, m stands for medium, p refers to play, and l indicates large.
    \end{tablenotes}
	\end{threeparttable}
\end{table*}

\subsubsection{Results in Gym-MuJoCo Domain}
As shown in Fig. \ref{fig_normalized_scores_curves} and Table \ref{normalized_scores_curves}, DPBAC achieves the best performance on the majority of locomotion control tasks. Most baseline methods perform well on the medium-expert datasets, which contain a subset of high-quality trajectories. DPBAC not only matches but further improves upon these results. However, in the medium and medium-replay datasets, where high-quality data is scarce, baseline methods often experience significant performance degradation. In contrast, DPBAC consistently delivers superior results, owing to the high expressiveness of its diffusion-based policy and the effectiveness of the behavioral advantage correction mechanism. These findings demonstrate the robustness and efficacy of the proposed approach.

\subsubsection{Results in AntMaze Domain}
The AntMaze domain presents a significant challenge due to its sparse reward signals and the absence of optimal trajectories. Consequently, accurate \emph{Q}-value estimation is crucial for achieving satisfactory performance. In addition, algorithms must be capable of composing near-optimal trajectories by stitching together suboptimal segments to accomplish goal-directed navigation. As shown in Table~\ref{normalized_scores_curves}, traditional approaches such as BC and DOGE yield near-zero scores in medium and large maze variants. In contrast, DPBAC obtains the best average AntMaze score and achieves the best or second-best performance on most variants, even in the most complex large maze scenarios. This advantage stems from its reliable \emph{Q}-value estimation mechanism, which remains effective under sparse-reward conditions. By providing value-guided feedback to the diffusion policy, DPBAC leverages the generative capacity of the diffusion model to generate effective trajectories that approximate optimal behavior.

\begin{figure}[!ht]
	\centering
	\subfigure[medium dataset]{
	\includegraphics[width=4.25cm]{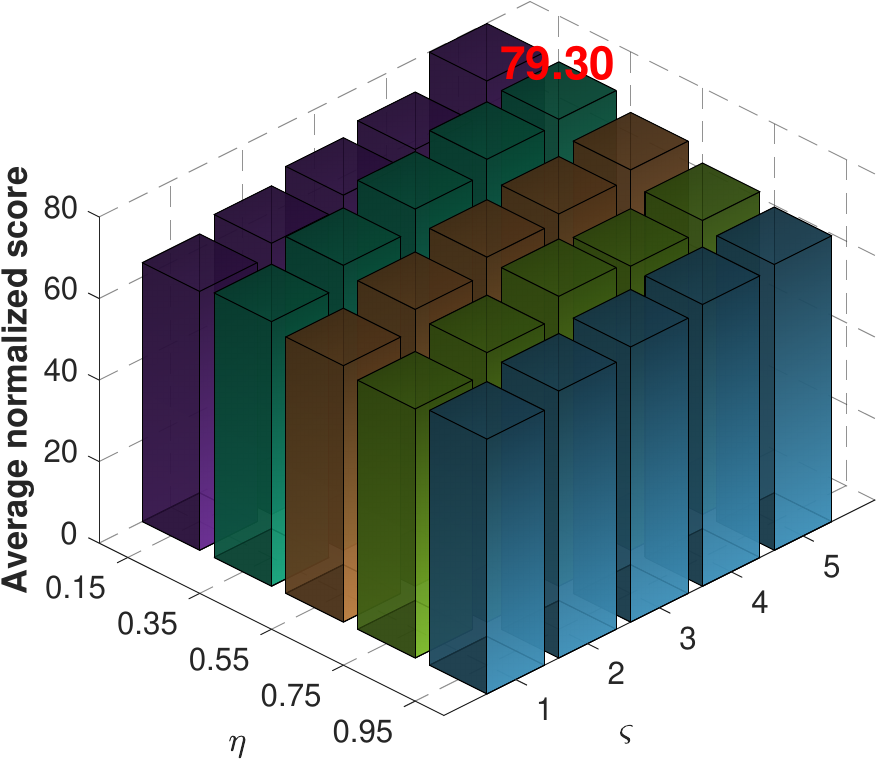}
	\includegraphics[width=4.25cm]{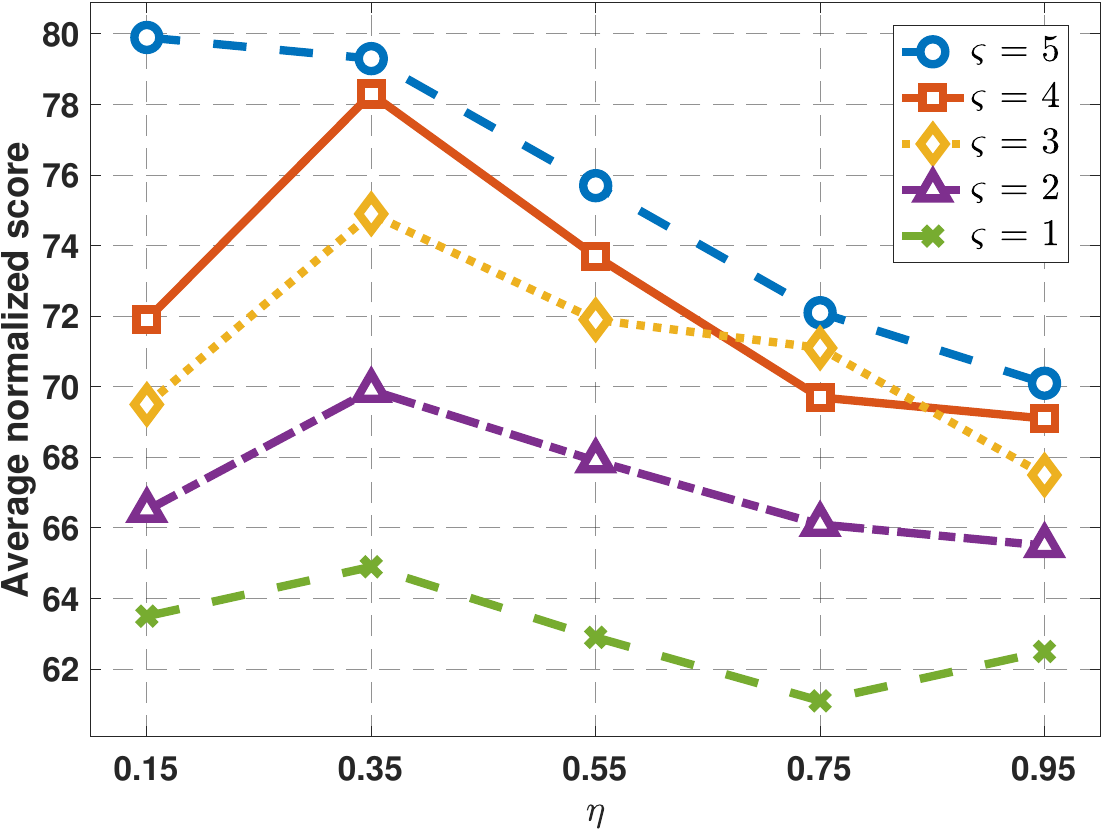}
	}
	\subfigure[medium-replay dataset]{
	\includegraphics[width=4.25cm]{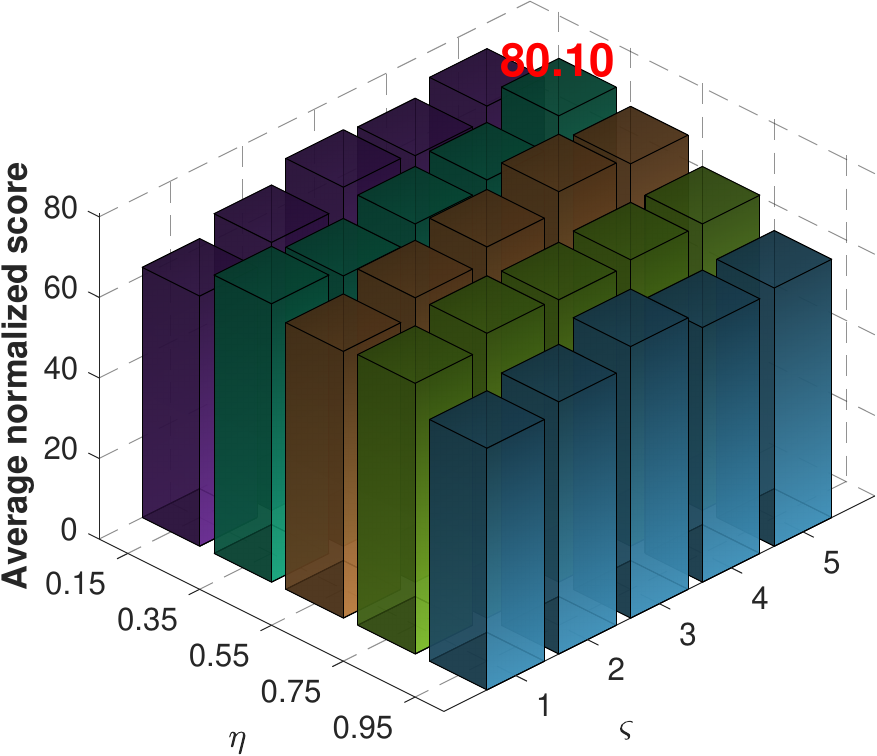}
	\includegraphics[width=4.25cm]{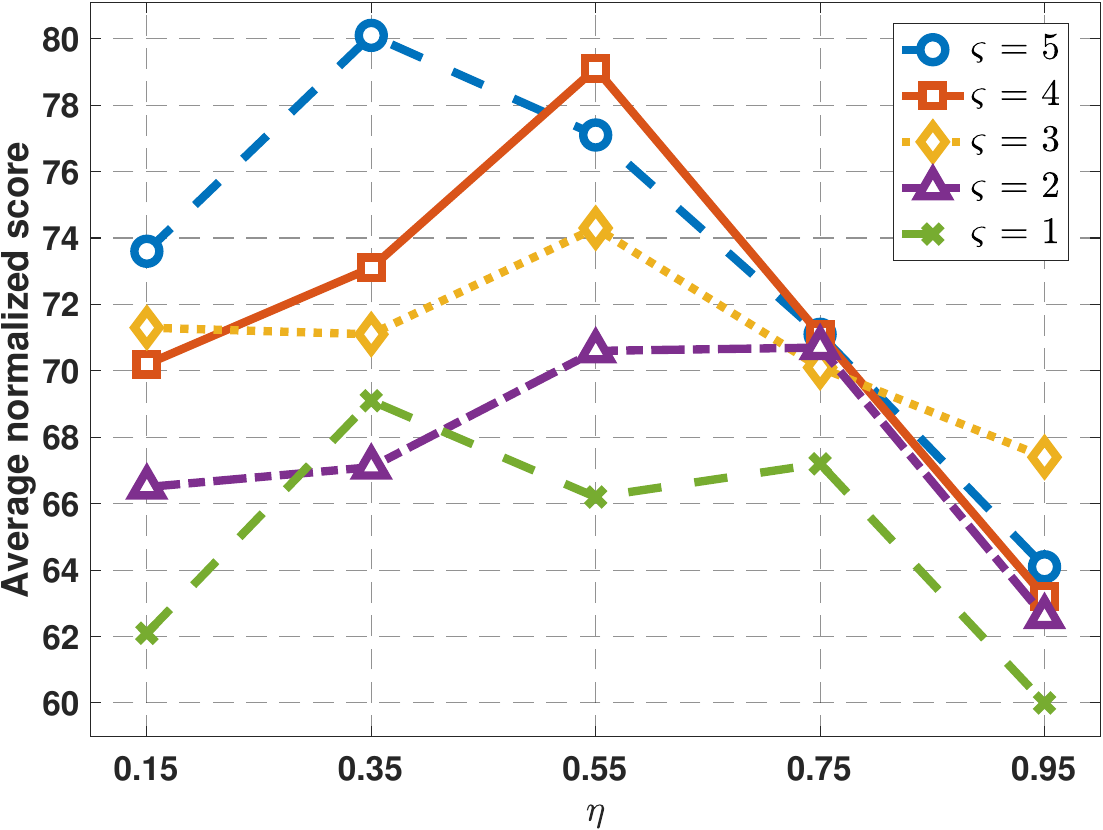}
	}
	\subfigure[medium-expert dataset]{
	\includegraphics[width=4.25cm]{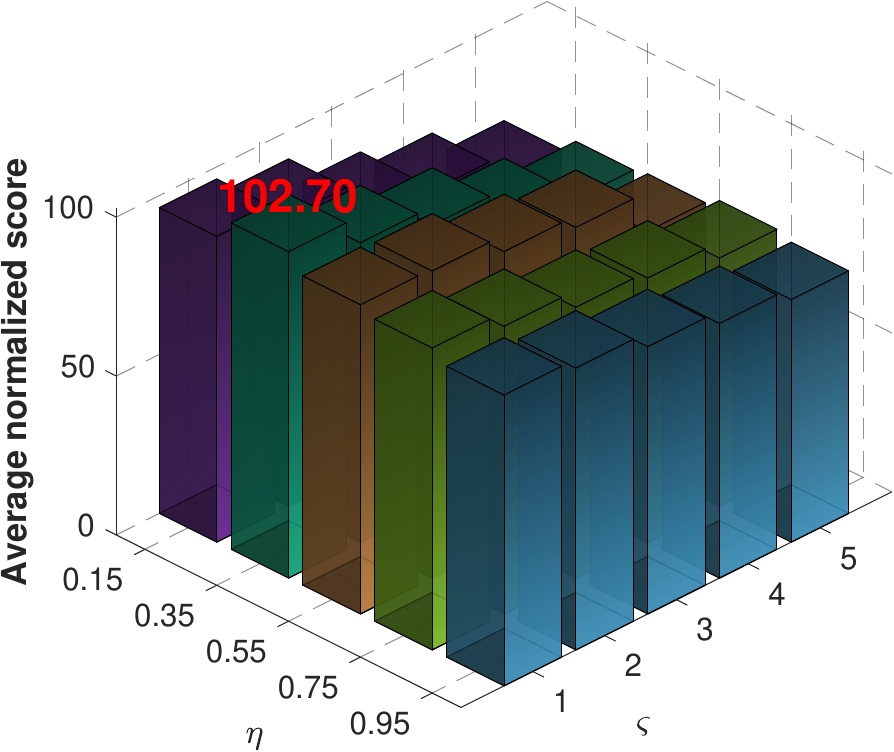}
	\includegraphics[width=4.25cm]{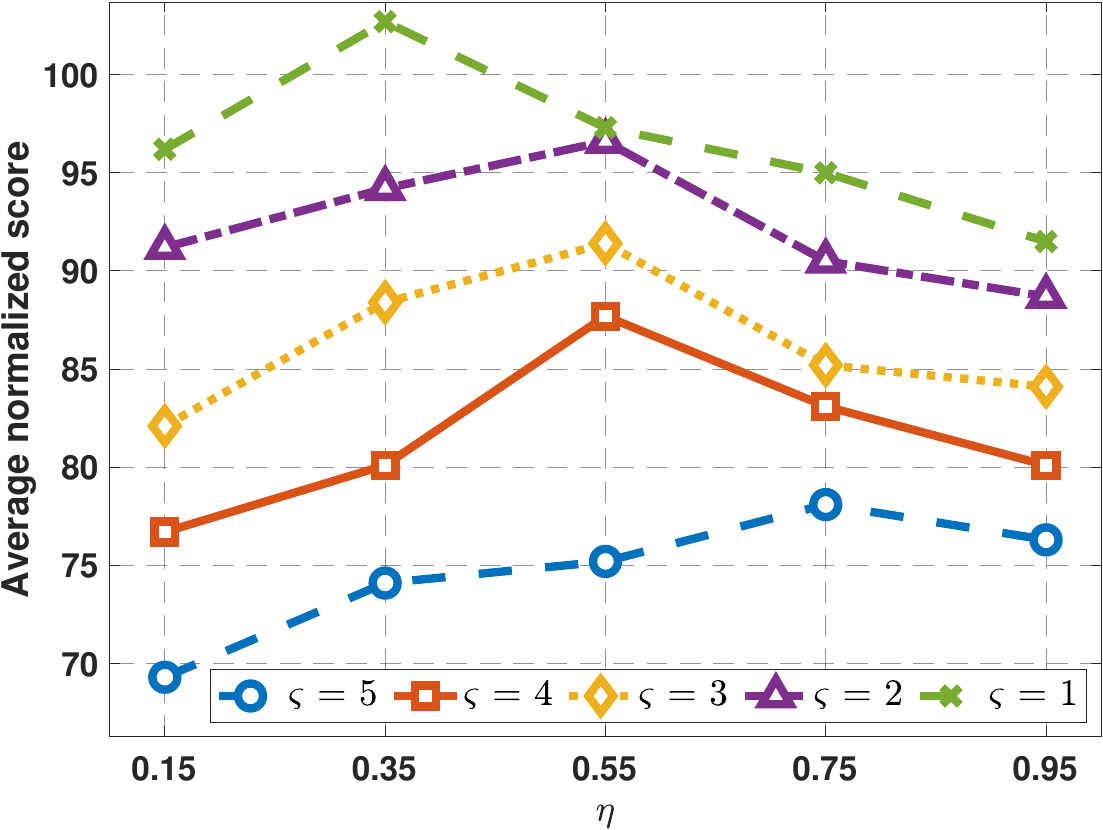}
	}
	\caption{Sensitivity test on the correction coefficient $\eta$ and regularization weight $\varsigma$ in Gym-MuJoCo tasks. Results are reported across eight random seeds for $\eta \in \{0.15, 0.35, 0.55, 0.75, 0.95\}$ and $\varsigma \in \{1, 2, 3, 4, 5\}$ on medium, medium-replay, and medium-expert datasets.}
	\label{3Dbar_test}
\end{figure}

\subsubsection{Results in FrankaKitchen and Adroit Domains}
The Adroit domain is characterized by narrow data distribution and high-dimensional state-action space, making it particularly susceptible to extrapolation errors. Consequently, the policy regularization ability becomes critical for ensuring reliable performance. The strong results achieved by DPBAC on Adroit domain can be attributed to the expressive ability of its diffusion policy, as enforced by the loss term $\mathcal{L}_{\mu}(\phi)$ in (\ref{adnb329045426}), which imposes strong regularization and effectively mitigates performance degradation caused by erroneous policy shift. The FrankaKitchen domain poses a long-horizon, multi-stage manipulation challenge that requires algorithms to model multi-goal behaviors and address the issue of long-horizon credit assignment. Most existing offline RL algorithms struggle to achieve satisfactory performance in this domain. However, DPBAC demonstrates consistently strong results, underscoring both the applicability and the effectiveness of diffusion-based policies in addressing such complex and compositional tasks.

\begin{figure}[!ht]
	\centering
	\subfigure[Gym-MuJoCo]{
	\includegraphics[width=4.120cm]{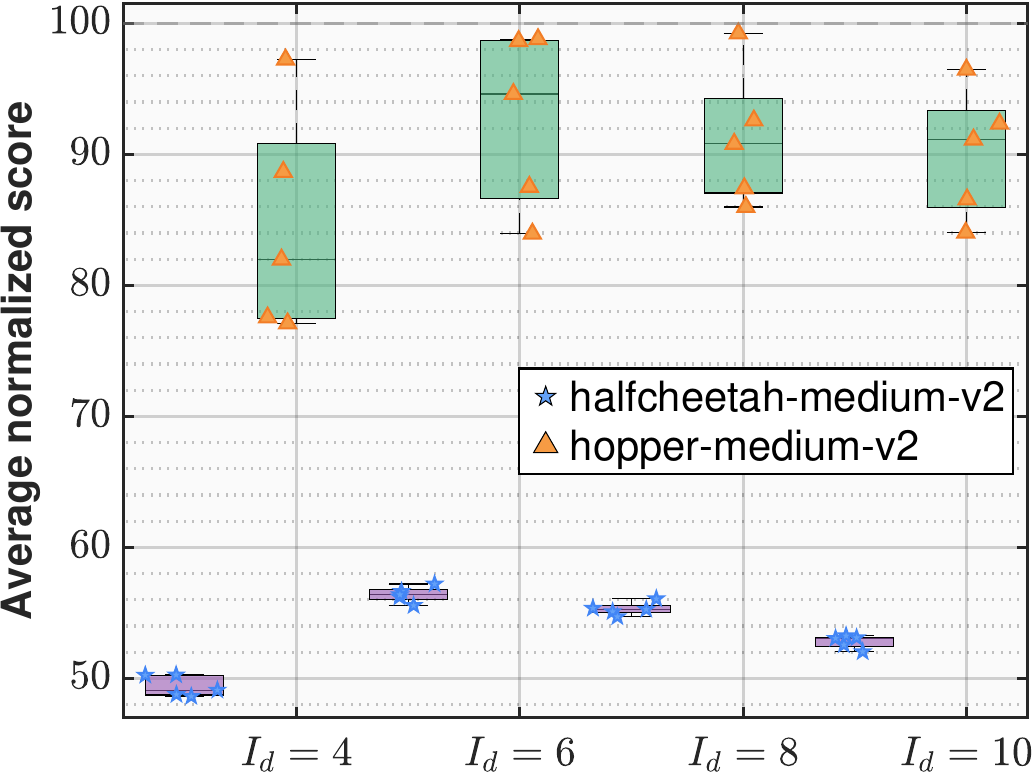}
	}
	\subfigure[AntMaze]{
	\includegraphics[width=4.120cm]{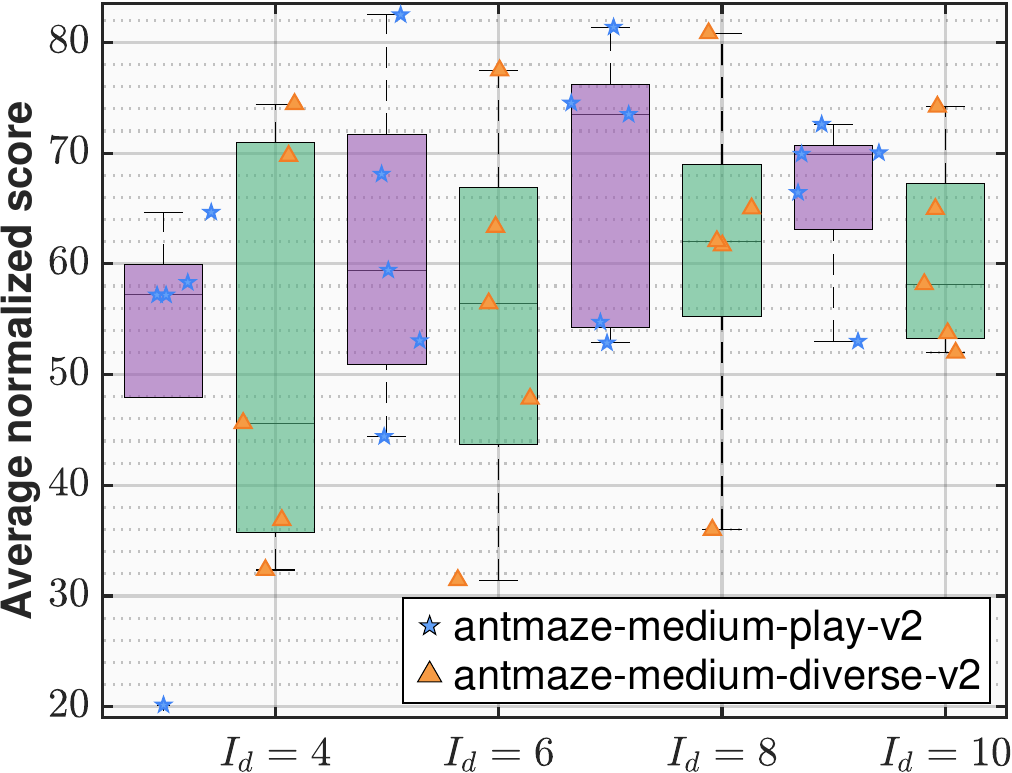}
	}
	\subfigure[FrankaKitchen]{
	\includegraphics[width=4.120cm]{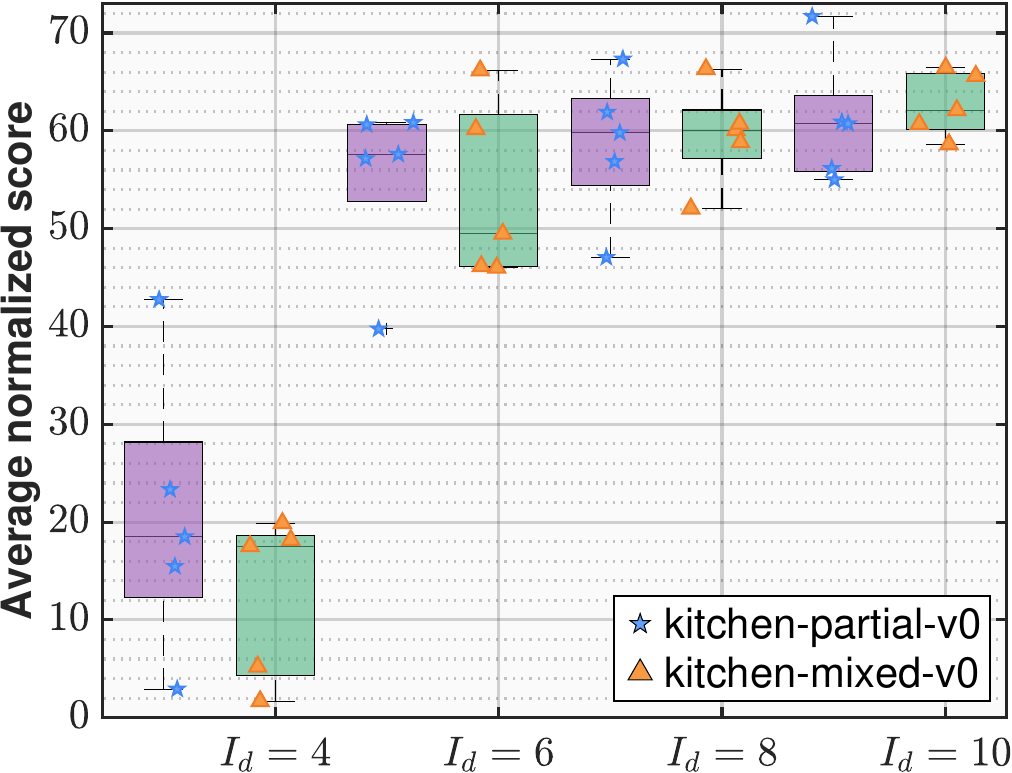}
	}
	\subfigure[Adroit]{
	\includegraphics[width=4.120cm]{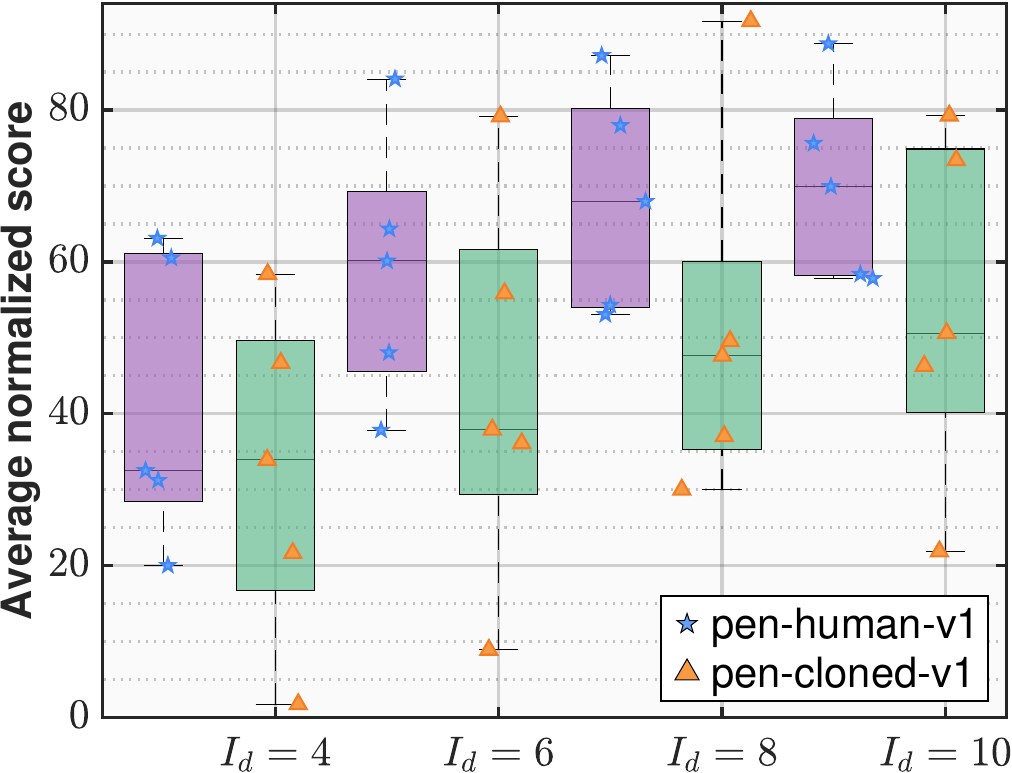}
	}
	\caption{Sensitivity test on the number of diffusion steps $I_d$, with results averaged over eight random seeds for $I_d \in \{4, 6, 8, 10\}$.}
	\label{Boxplot}
\end{figure}

\begin{figure*}[!ht]
	\centering
	\subfigure[halfcheetah-medium-v2]{
	\includegraphics[width=4.20cm]{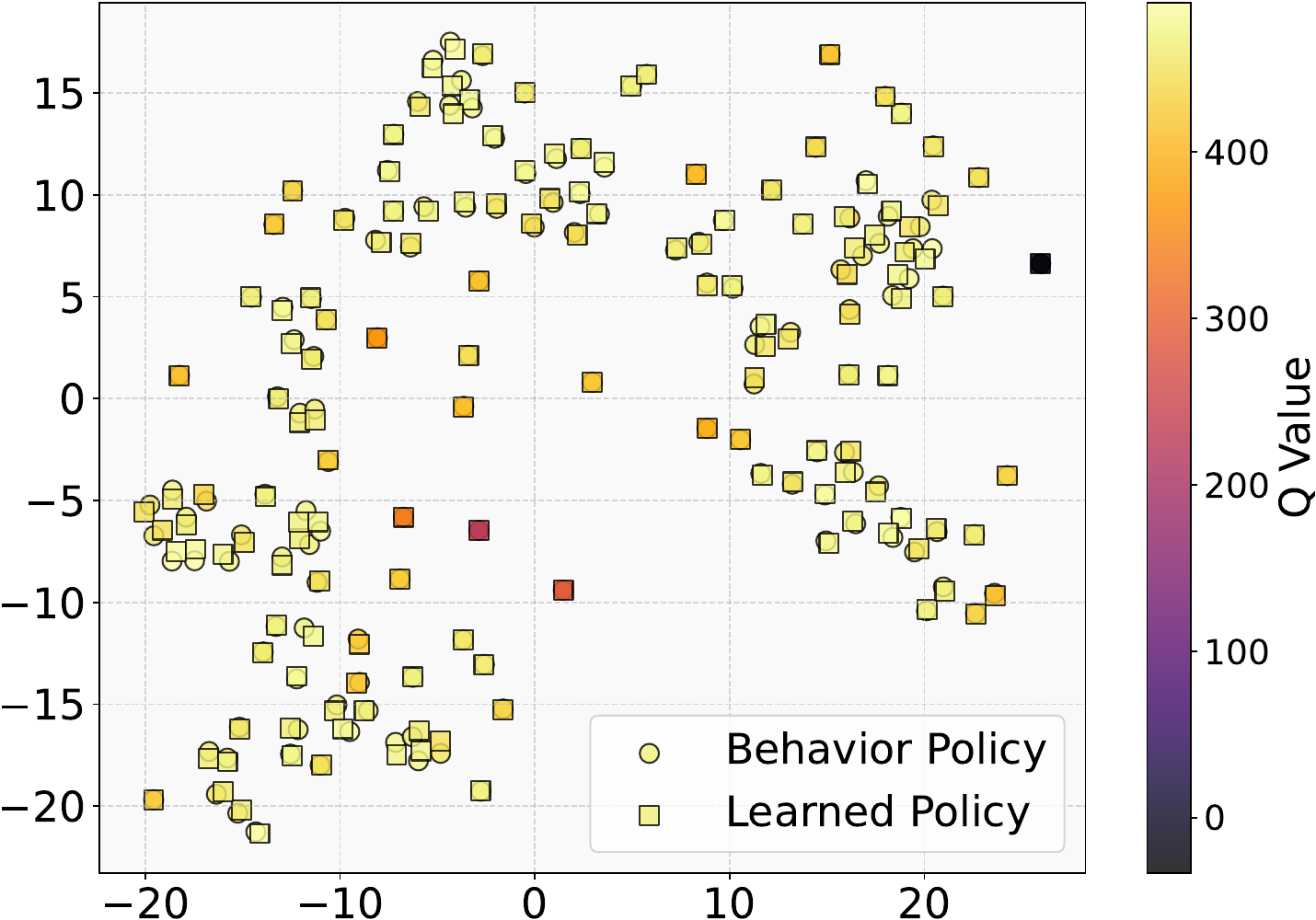}
	}
	\subfigure[hopper-medium-v2]{
	\includegraphics[width=4.20cm]{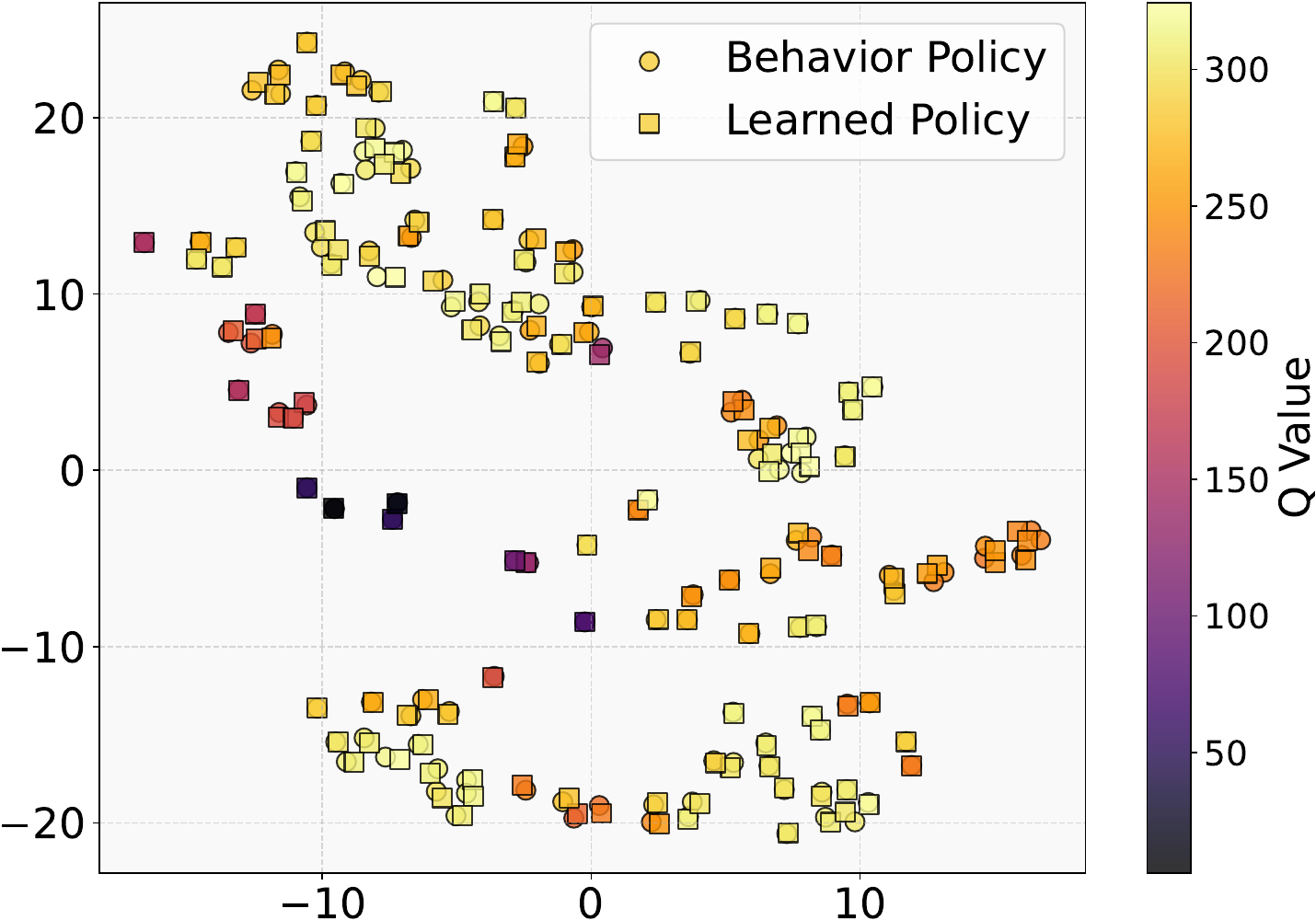}
	}
	\subfigure[walker2d-medium-v2]{
	\includegraphics[width=4.20cm]{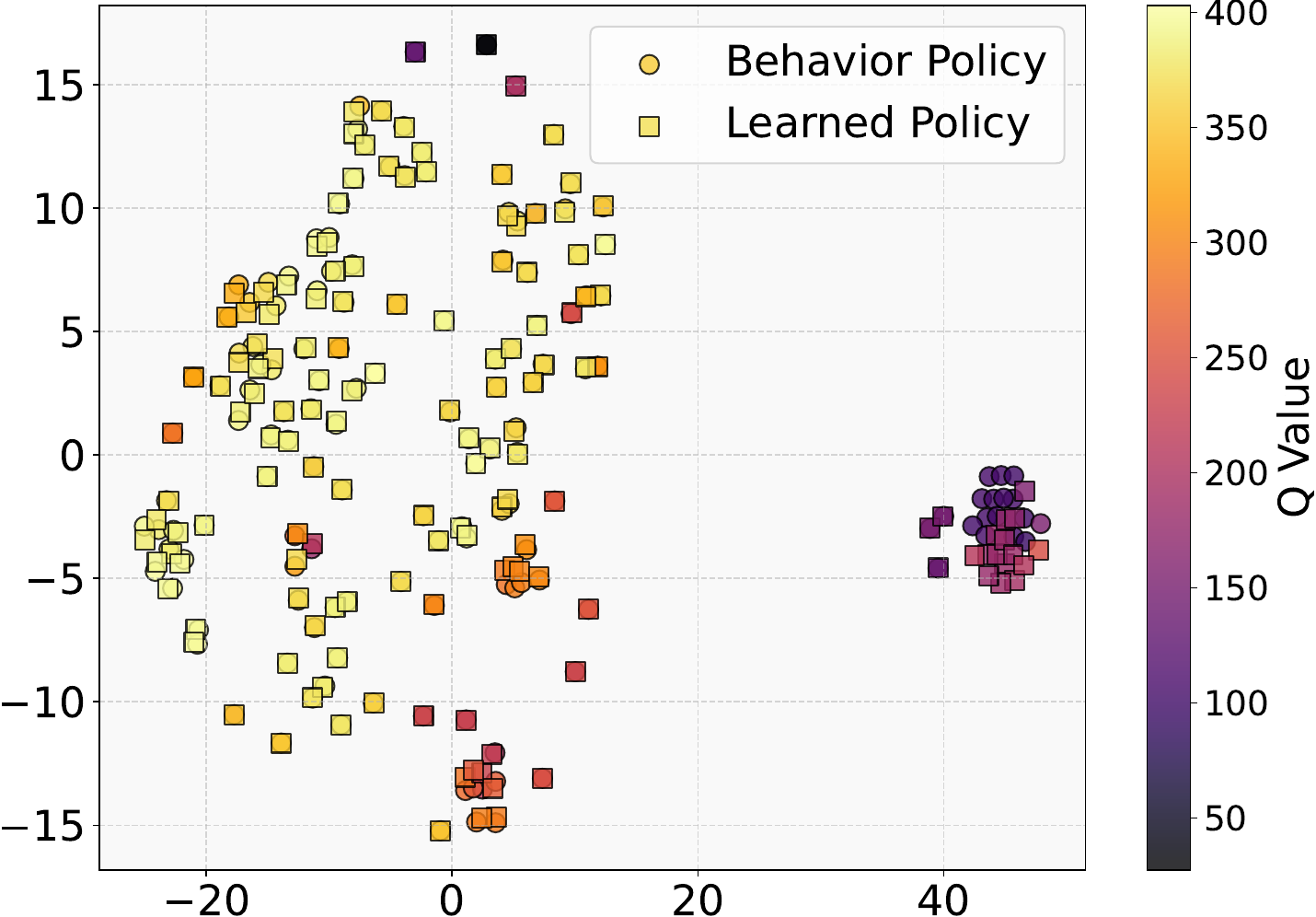}
	}
	\subfigure[halfcheetah-medium-replay-v2]{
	\includegraphics[width=4.20cm]{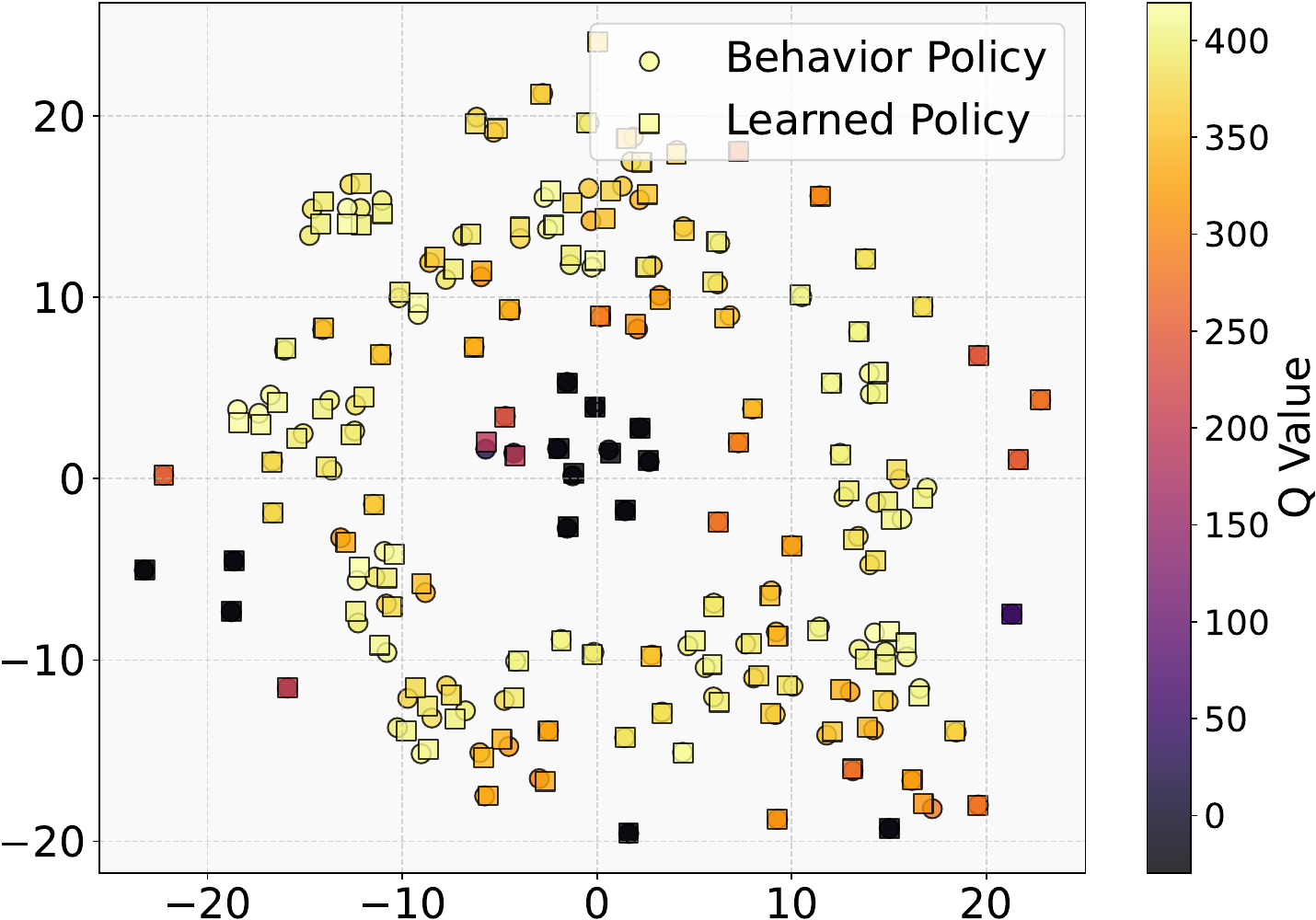}
	}
	\subfigure[hopper-medium-replay-v2]{
	\includegraphics[width=4.20cm]{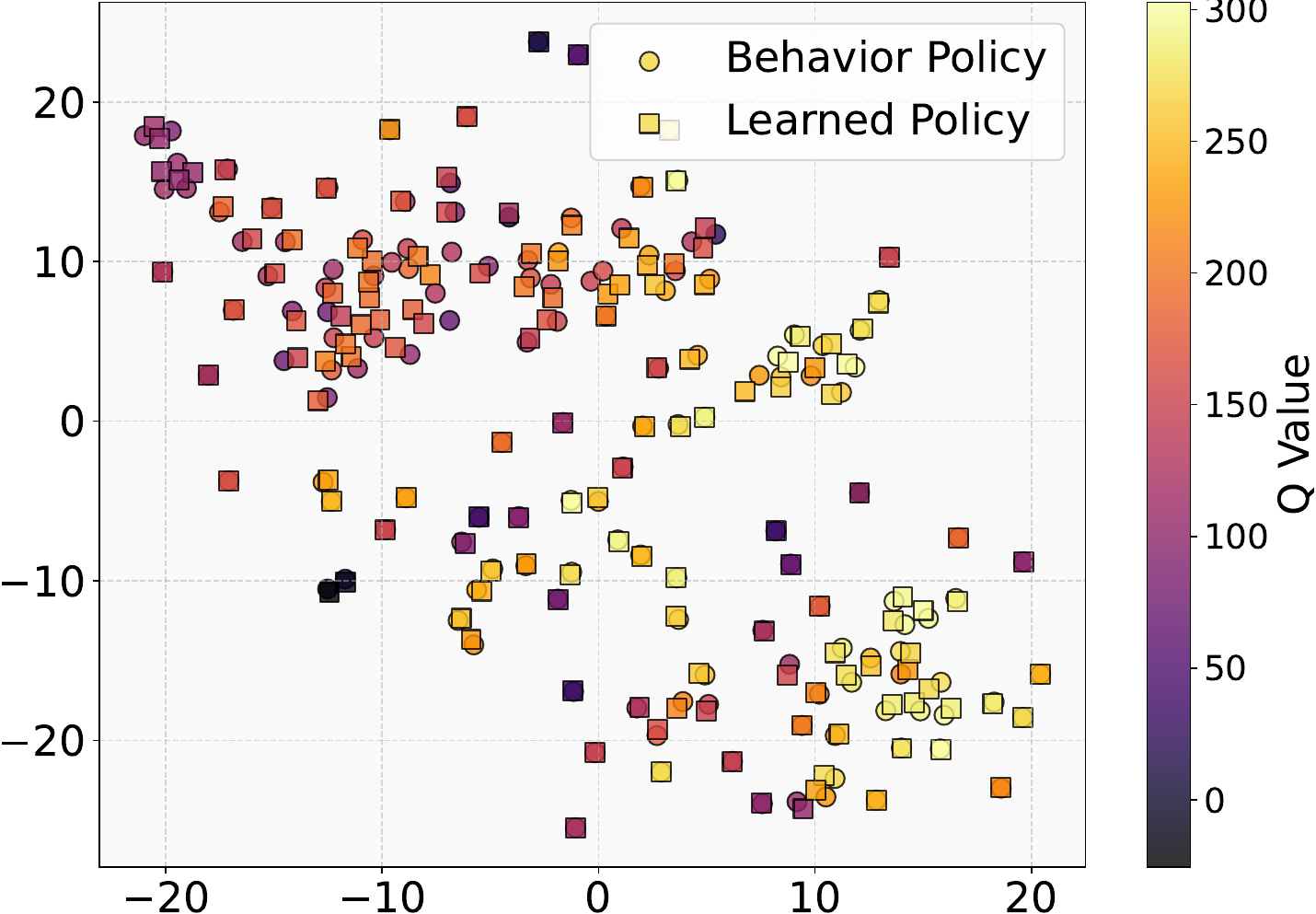}
	}
	\subfigure[walker2d-medium-replay-v2]{
	\includegraphics[width=4.20cm]{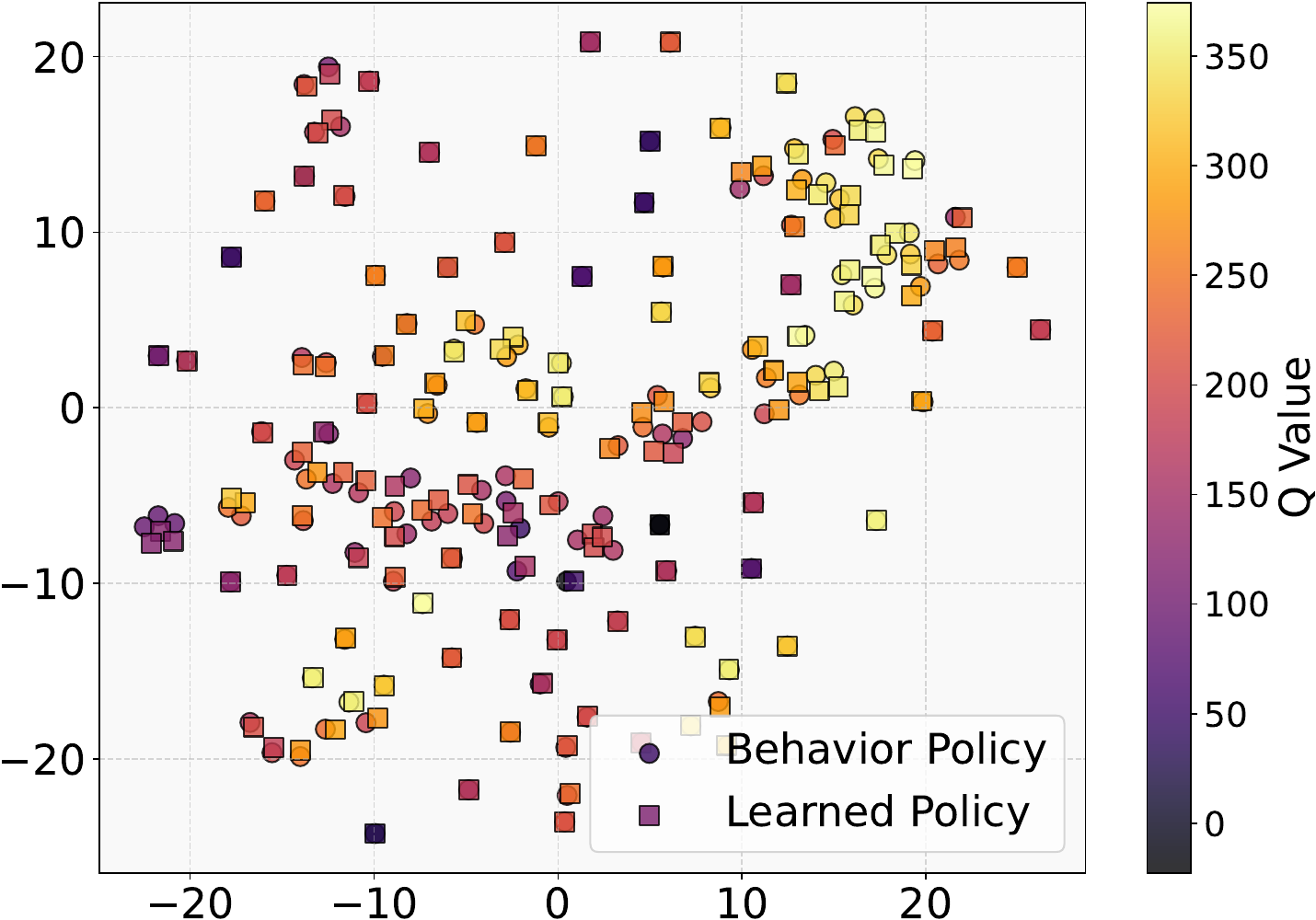}
	}
	\subfigure[halfcheetah-medium-expert-v2]{
	\includegraphics[width=4.20cm]{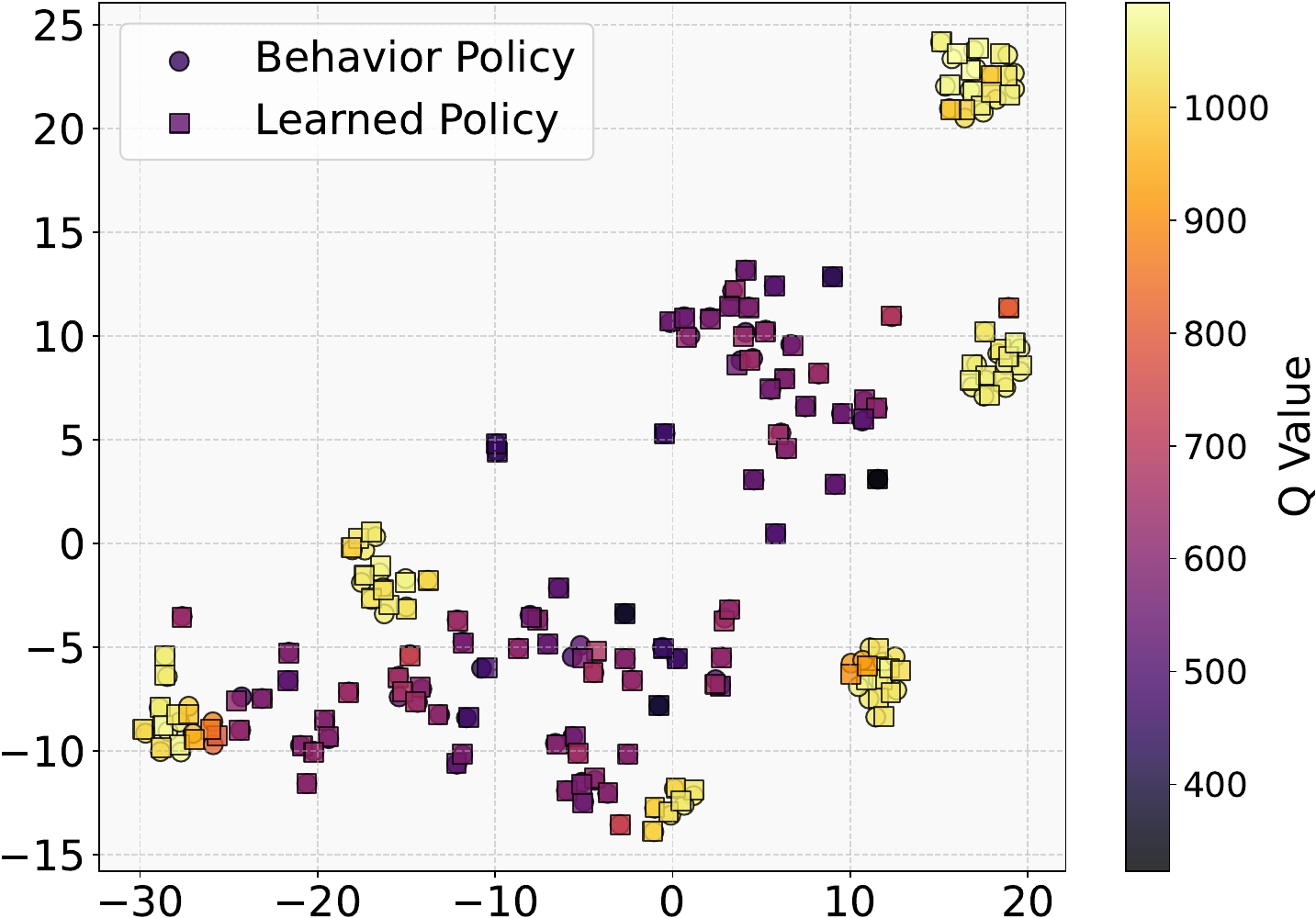}
	}
	\subfigure[hopper-medium-expert-v2]{
	\includegraphics[width=4.20cm]{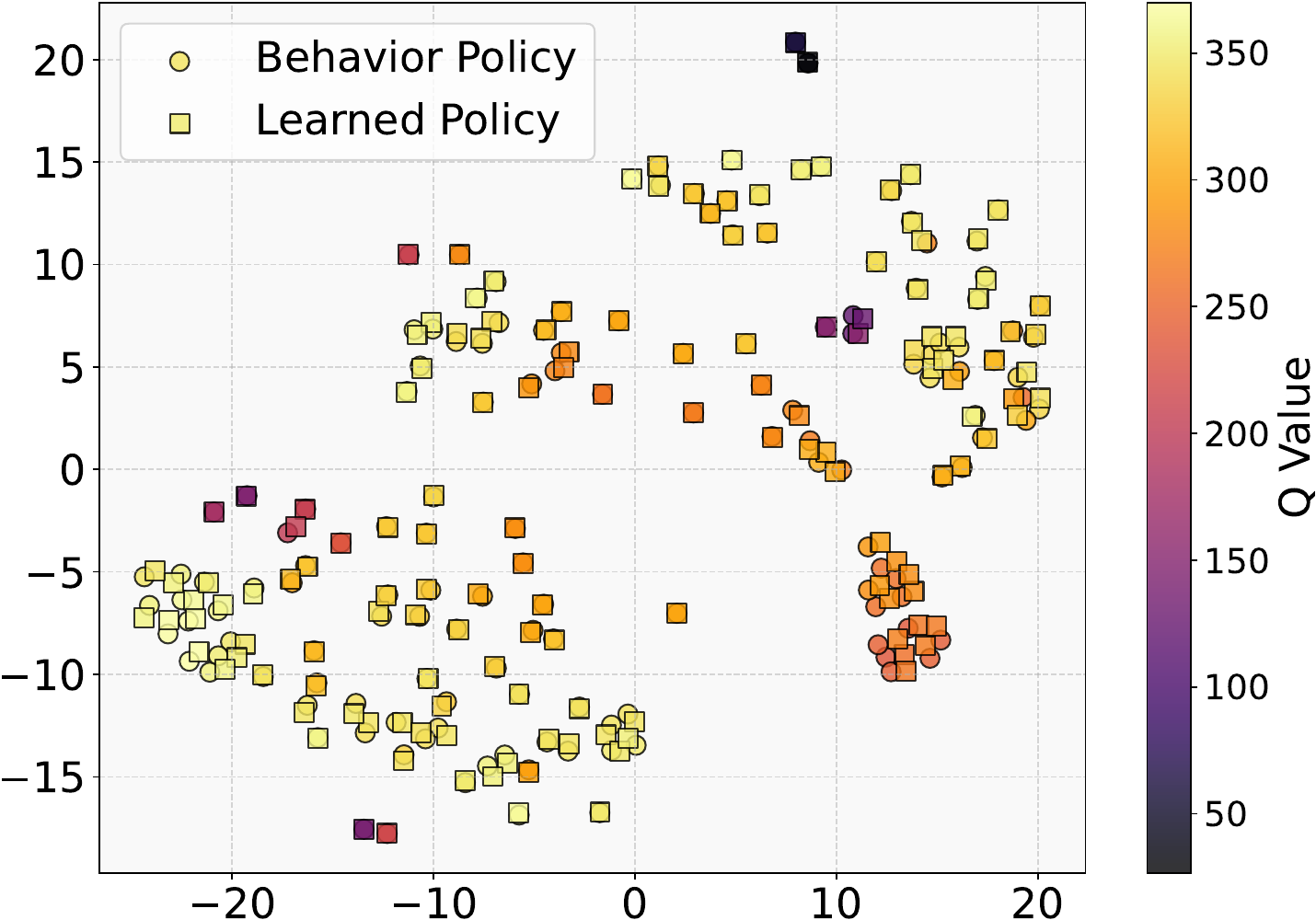}
	}
	\subfigure[walker2d-medium-expert-v2]{
	\includegraphics[width=4.20cm]{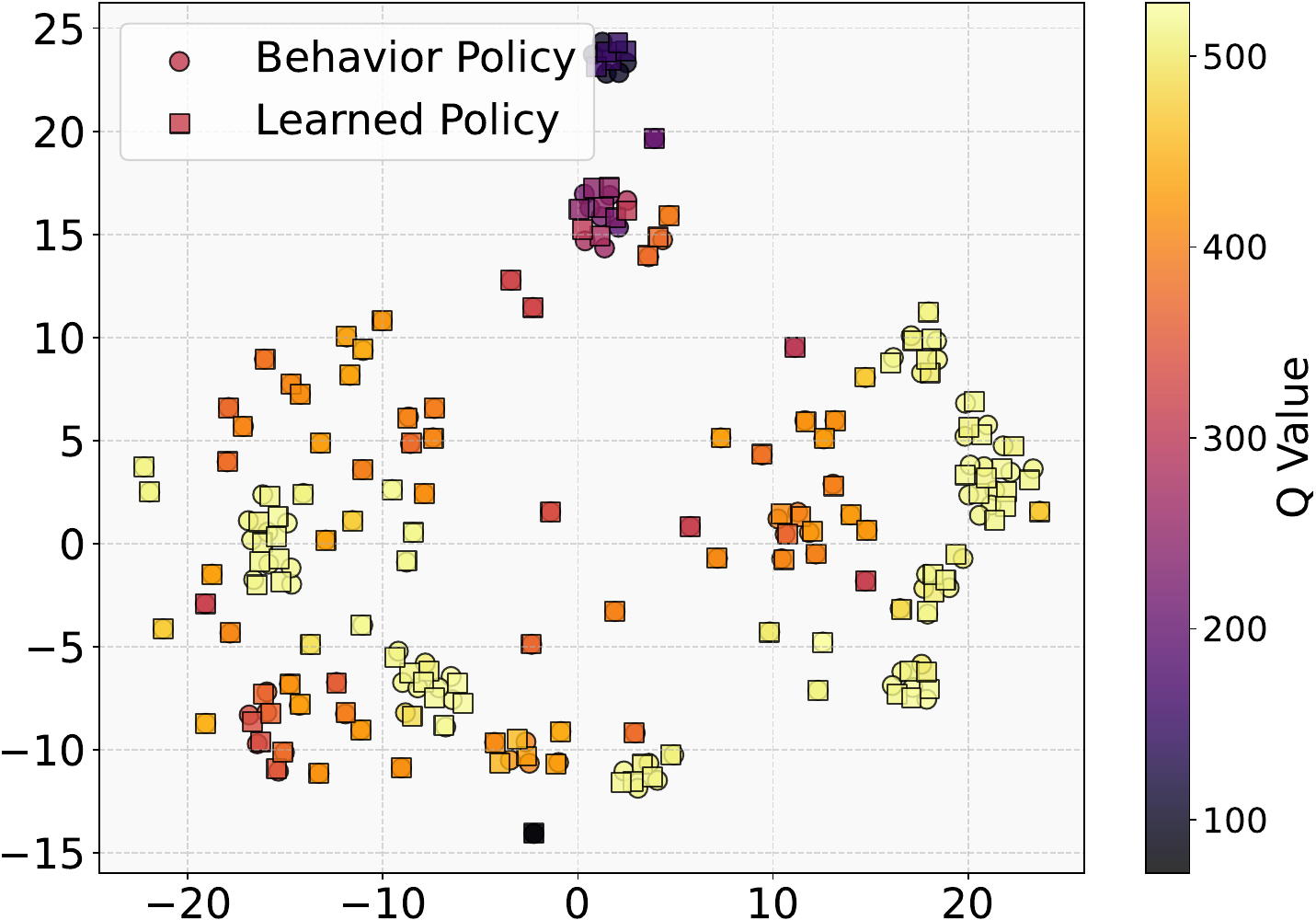}
	}
	\subfigure[antmaze-umaze-diverse-v2]{
	\includegraphics[width=4.20cm]{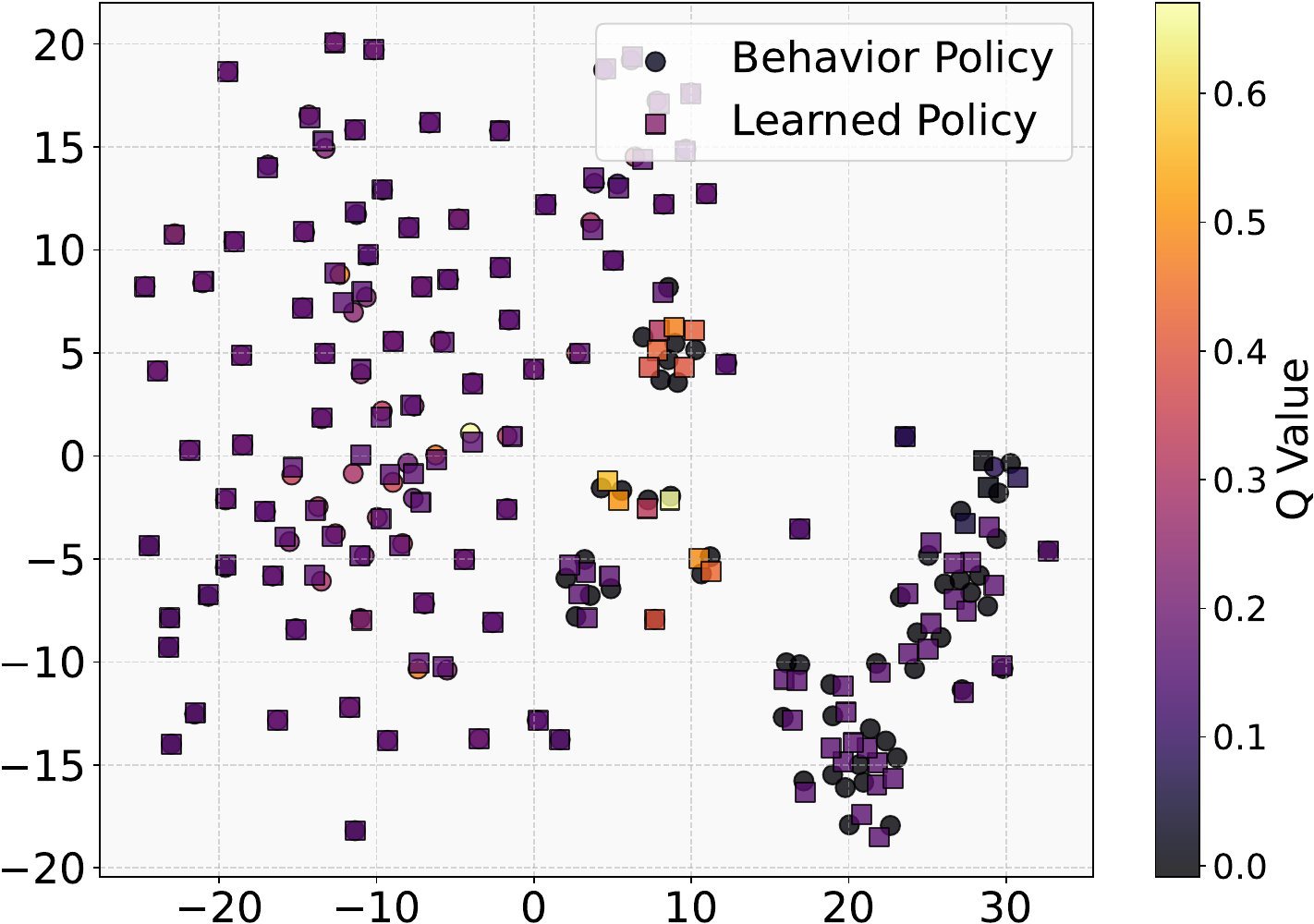}
	}
	\subfigure[antmaze-medium-play-v2]{
	\includegraphics[width=4.20cm]{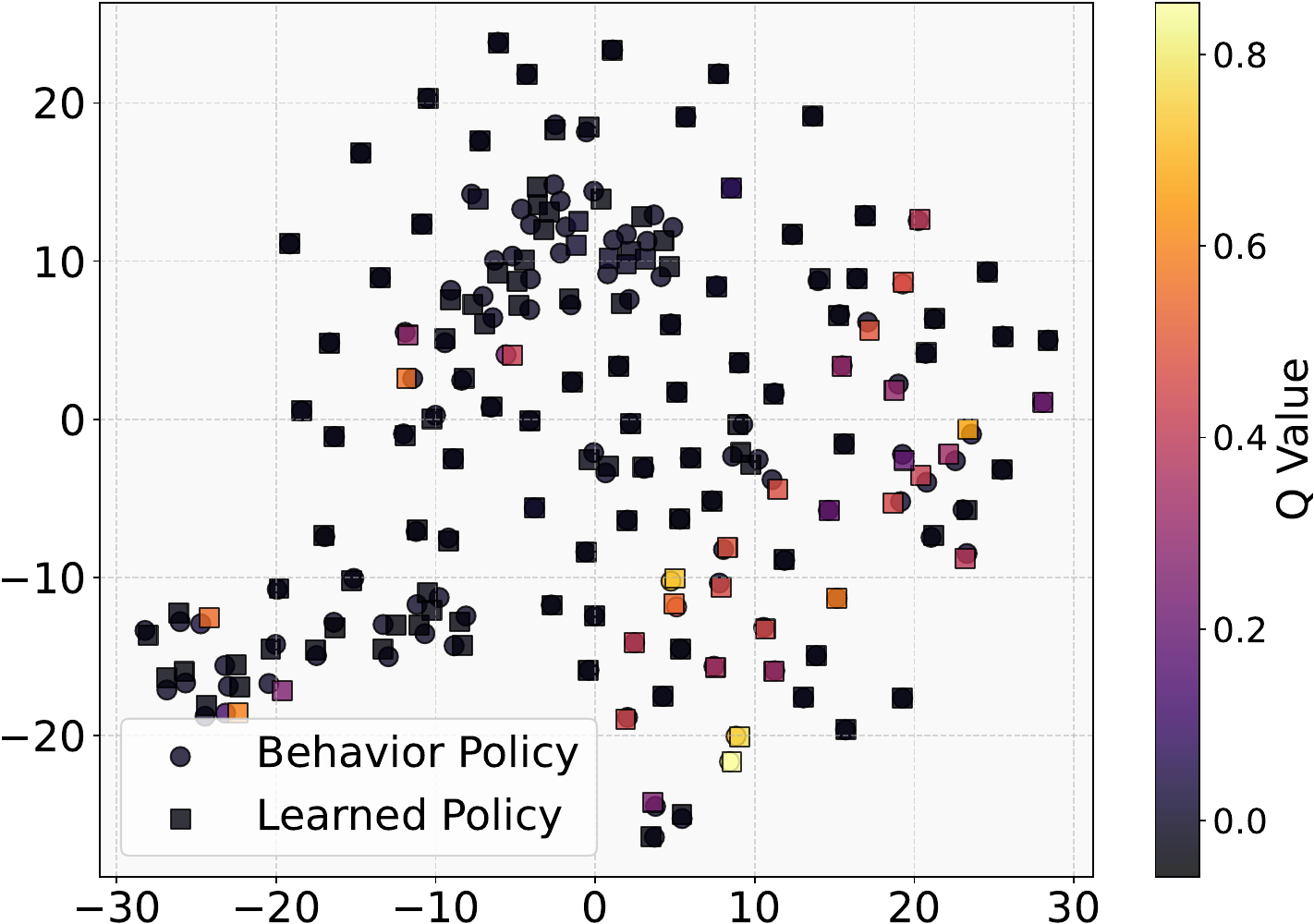}
	}
	\subfigure[antmaze-large-diverse-v2]{
	\includegraphics[width=4.20cm]{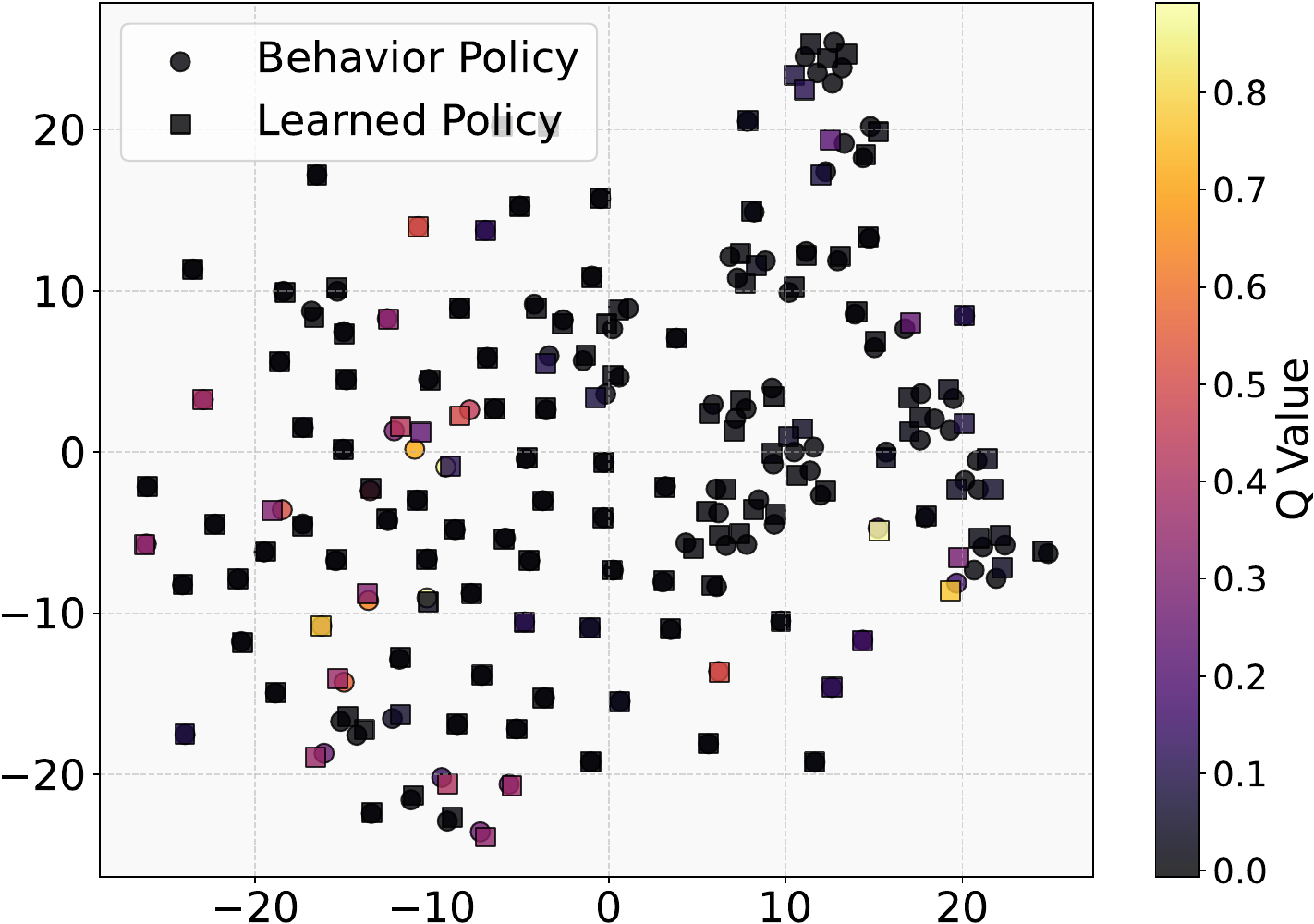}
	}
	\caption{Analysis of policy expressiveness and regularization ability: t-SNE visualization of state-action pair distributions from dataset $\mathcal{D}$ and those generated by the learned policy.}
	\label{fig_tsne}
\end{figure*}

\subsection{Sensitivity Studies and Parameter Analysis}\label{sdnl23004hugf}
To assess the contribution of each component within DPBAC to the overall performance, we conduct a series of sensitivity tests. Specifically, we investigate the impact of three key hyperparameters: the correction coefficient $\eta$, which controls the degree of \emph{Q}-value correction; the regularization weight $\varsigma$, which regulates the strength of policy regularization; and the number of diffusion steps $I_d$, which governs the expressiveness and precision of the generative process. Detailed test results for each hyperparameter are presented below.

\subsubsection{Evaluating the Impact of $\eta$ and $\varsigma$ via Sensitivity Tests} 
We conduct the following sensitivity tests by varying $\eta \in \{0.15, 0.35, 0.55, 0.75, 0.95\}$ and $\varsigma \in \{1, 2, 3, 4, 5\}$, with the corresponding results depicted in Fig.~\ref{3Dbar_test}. Overall, across different datasets, the performance of DPBAC exhibits a non-monotonic trend with respect to $\eta$, initially increasing and then decreasing. This phenomenon can be attributed to the trade-off in \emph{Q}-value correction: a small $\eta$ provides insufficient correction of the \emph{Q}-values, while a large $\eta$ may lead to overcorrection, thereby degrading performance. Although the optimal value of $\eta$ varies across datasets, $\eta = 0.35$ consistently yields strong performance. In addition, we observe that the algorithm's performance increases with larger values of $\varsigma$ on the medium and medium-replay datasets, but decreases with larger $\varsigma$ on the medium-expert datasets. This is due to the lower quality of trajectories in the former datasets, where weaker policy regularization (i.e., larger $\varsigma$) allows the policy to deviate more effectively from suboptimal behavior. In contrast, the medium-expert datasets contain higher-quality expert demonstrations, and stronger policy regularization helps preserve desirable behavior, thus enhancing performance.

\subsubsection{Evaluating the Impact of $I_d$ via Sensitivity Tests}
To investigate the effect of diffusion steps $I_d \in \{4, 6, 8, 10\}$ on algorithm performance across different task domains, we conduct sensitivity tests, with the corresponding results shown in Fig. \ref{Boxplot}. In the figure, triangles and pentagrams indicate individual experimental outcomes. As illustrated by the box plots, increasing the diffusion steps $I_d$ generally leads to improved performance. This trend can be attributed to the enhanced generative capacity of the diffusion policy with longer denoising steps, which facilitates the modeling of more complex action distributions. However, larger values of $I_d$ also incur higher computational costs. In high-dimensional domains such as FrankaKitchen and Adroit, selecting a larger $I_d$ leads to a more pronounced performance improvement due to the increased complexity of the state-action space.

\subsection{Policy Shift Investigation}
To evaluate the expressiveness and regularization ability of the diffusion policy in DPBAC, we randomly sample 150 state-action pairs from the dataset, where the actions are regarded as those taken by the behavior policy. For each sampled state, we generate a corresponding action using the trained diffusion policy, referred to as the learned policy action. We then visualize both sets of state-action pairs using t-SNE, as shown in Fig. \ref{fig_tsne}. Circles represent state-action pairs from the behavior policy, while squares represent those generated by the learned policy. To assess the quality of these actions, we compute the \emph{Q}-values of all visualized pairs using the trained \emph{Q}-function network, which are used for color encoding. The visualization reveals that the learned policy closely covers the distribution of the behavior policy, demonstrating the strong expressive capacity and regularization ability of the diffusion model. Moreover, the learned policy's state-action pairs exhibit generally higher \emph{Q}-values compared to those of the behavior policy, indicating that DPBAC effectively guides the diffusion policy toward higher-value regions of the action space, thereby achieving superior performance over the behavior policy.

\subsection{Behavioral Advantage Study}
To investigate the impact of the behavioral advantage $\widehat{A}^{\mu}_{\widetilde{Q}}$ on policy training, we vary the correction coefficient $\eta \in \{0.15, 0.25, 0.35, 0.55, 0.75, 0.95\}$. Fig. \ref{BA_plot} illustrates the evolution of the behavioral advantage throughout training under different $\eta$ values. It can be observed that stronger correction strengths (i.e., larger values of $\eta$) lead to smaller behavioral advantages, indicating that the learned \emph{Q}-function becomes more closely aligned with the \emph{Q}-function of the behavior policy. This observation is consistent with our theoretical analysis. A smaller behavioral advantage suggests reduced policy deviation from the behavior policy, which helps stabilize training in offline settings where distributional shift is detrimental. However, an excessively strong correction with a large $\eta$ may impede the policy's ability to explore higher-reward regions beyond those induced by the behavior policy, thereby potentially constraining overall performance improvement.

\begin{figure}[!ht]
	\centering
	\subfigure[halfcheetah-medium-v2]{
	\includegraphics[width=4.124cm]{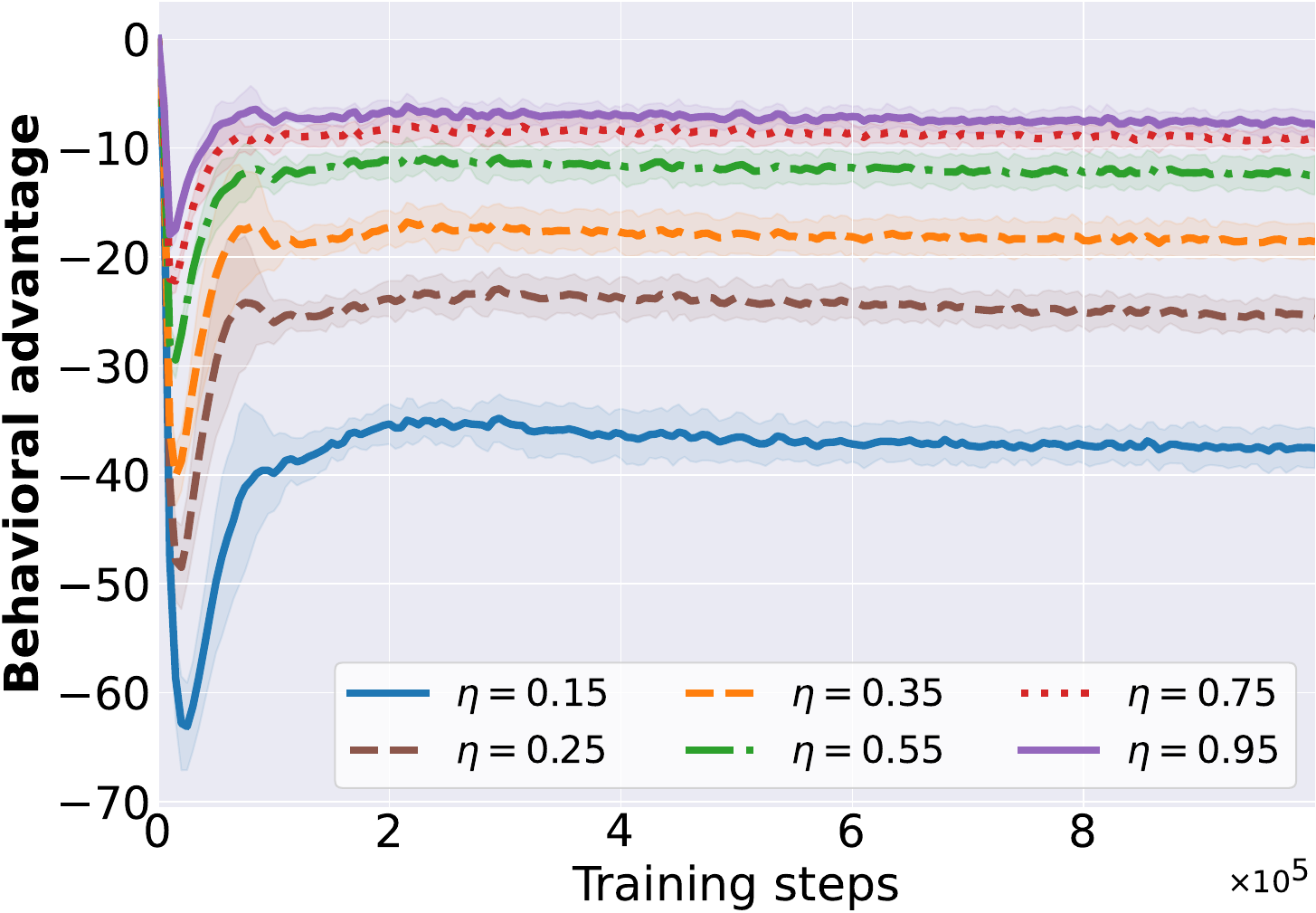}
	}
	\subfigure[halfcheetah-medium-replay-v2]{
	\includegraphics[width=4.124cm]{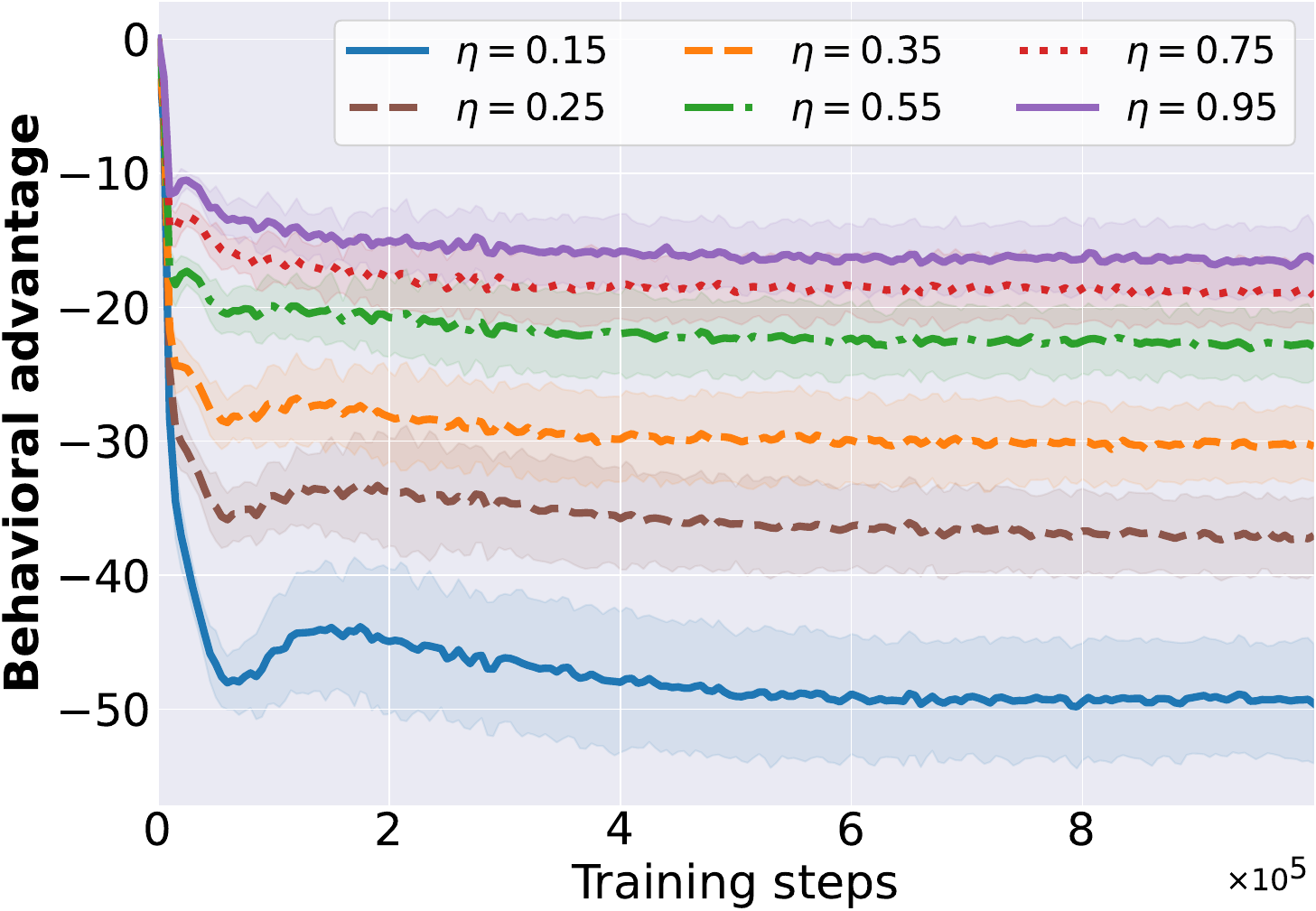}
	}
	\subfigure[halfcheetah-medium-expert-v2]{
	\includegraphics[width=4.124cm]{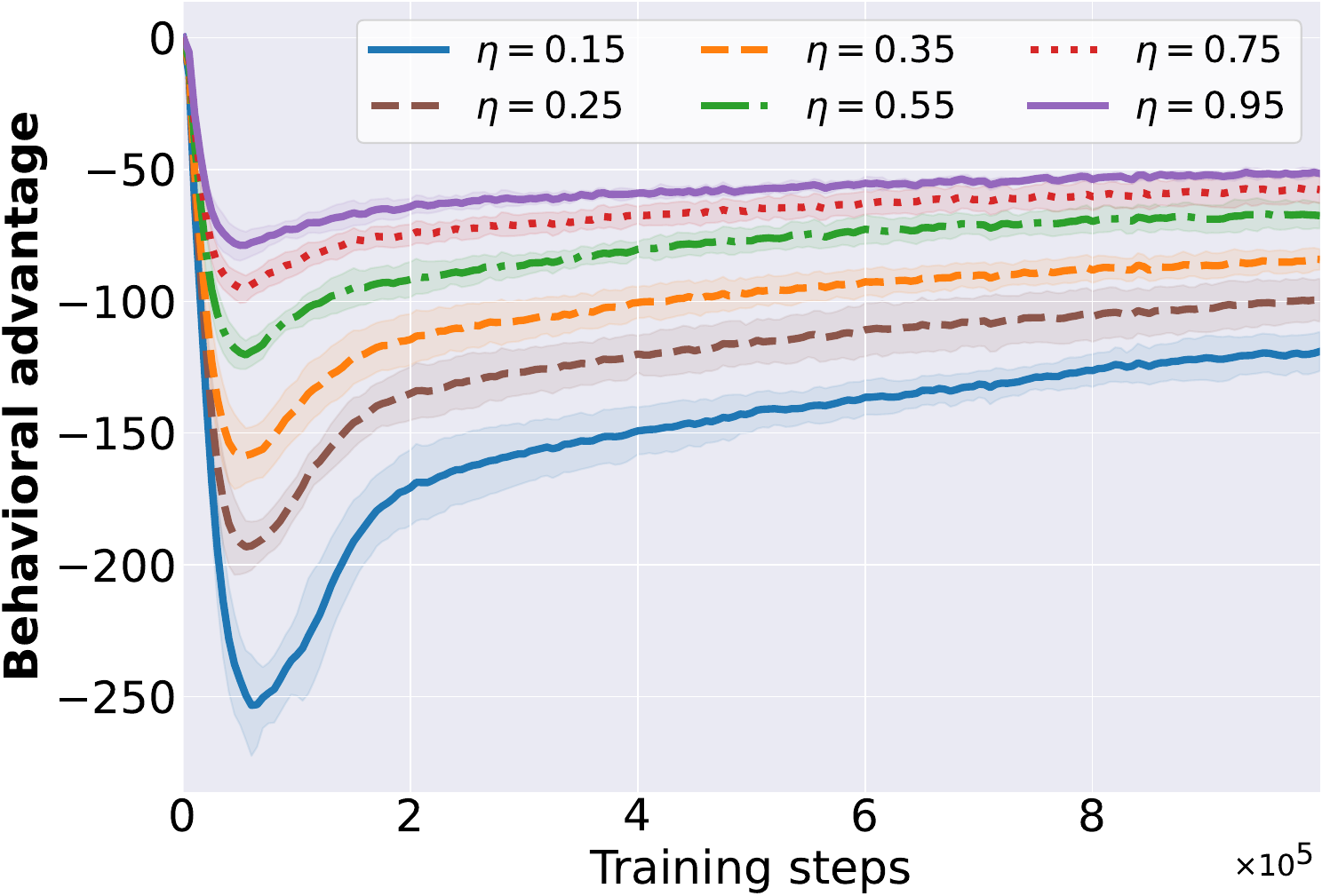}
	}
	\subfigure[hopper-medium-replay-v2]{
	\includegraphics[width=4.124cm]{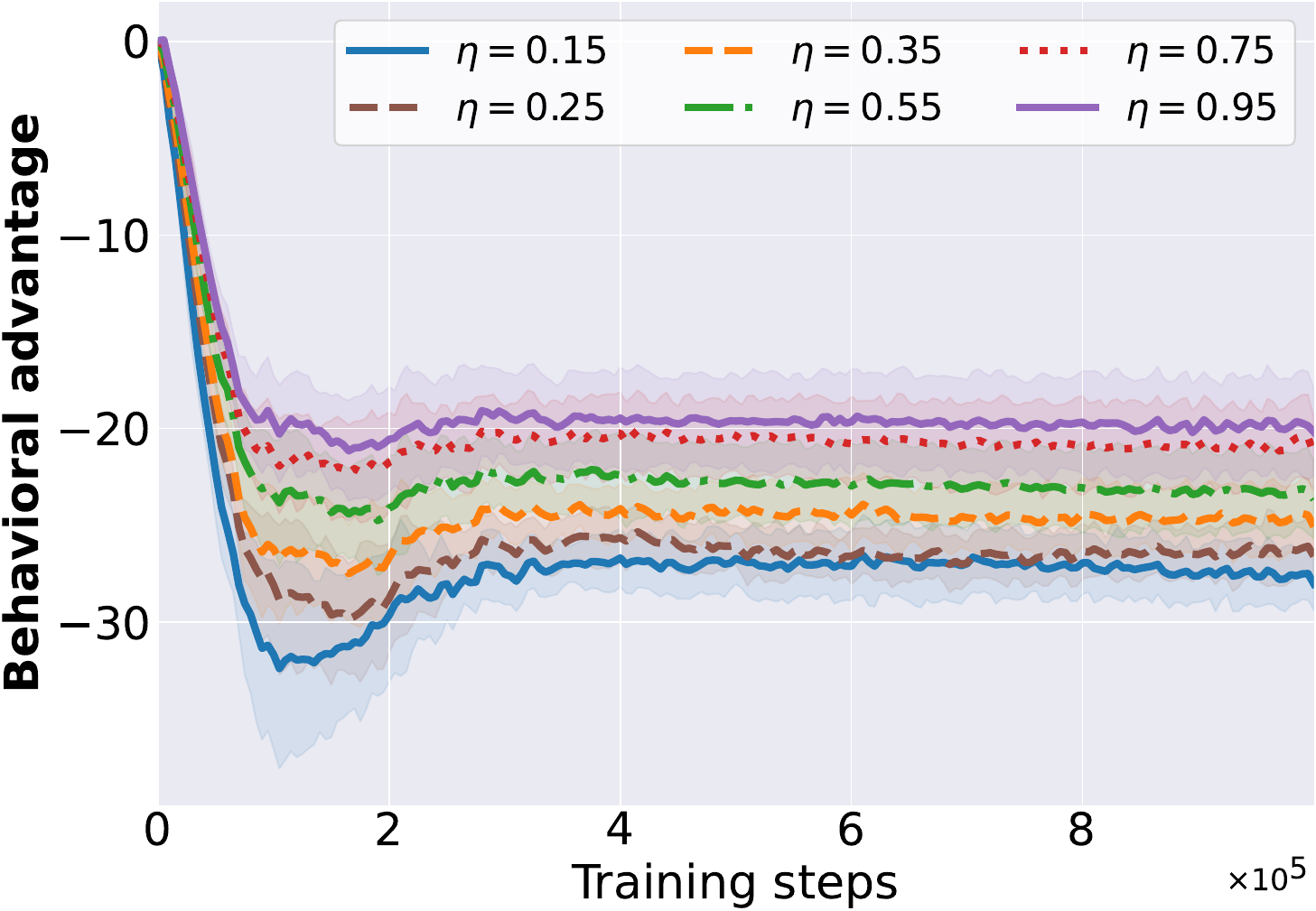}
	}
	\caption{Behavioral advantage curves of DPBAC over the full training process with varying $\eta$ values.}
	\label{BA_plot}
\end{figure}

\subsection{Sparse Reward Tasks Evaluation}
To evaluate the ability of DPBAC in handling sparse-reward tasks, we conduct experiments on the maze navigation tasks in the Maze2D domain, which includes four levels of maze complexity: open, umaze, medium, and large. We consider both dense and sparse reward settings. In the dense setting, rewards are distributed stepwise and are inversely proportional to the distance to the goal, whereas in the sparse setting, rewards are provided only upon reaching the goal. We compare DPBAC against the diffusion-based method DQL \cite{Wang172023}, the value regularization algorithm CQL \cite{KumarA2020}, and the conditional sequence modeling approach QDT \cite{Yamagata32842}. The experimental results are presented in Fig. \ref{Horizon_stack}. It can be observed that while CQL and DQL perform well under the dense reward setting, their performance significantly deteriorates under sparse rewards due to inaccurate value estimation. Similarly, QDT, which relies on CQL for return-to-go token relabeling, also struggles to adapt to sparse reward tasks. In contrast, DPBAC demonstrates consistently strong performance across both dense and sparse reward settings. This highlights the effectiveness of the proposed \emph{Q}-value correction mechanism and the powerful regularization ability of the diffusion policy, which together help mitigate the negative impact of \emph{Q}-value estimation bias.

\begin{figure}[!ht]
	\centering
	\subfigure[Dense reward setting]{
	\includegraphics[width=4.124cm]{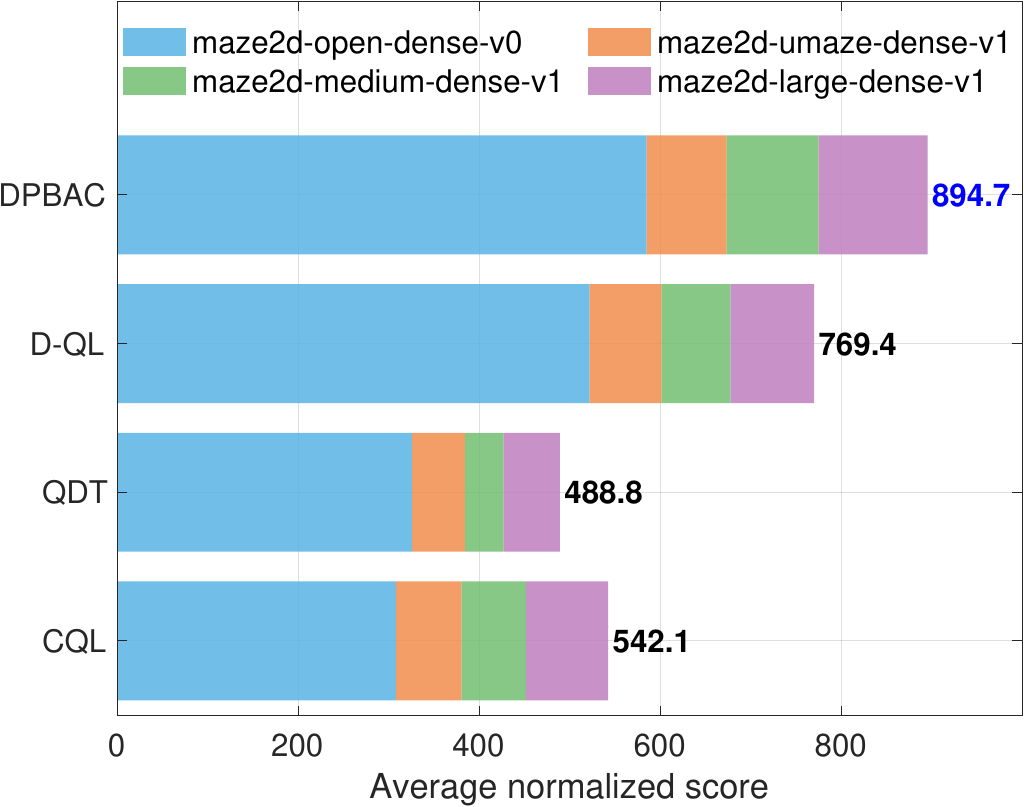}
	}
	\subfigure[Sparse reward setting]{
	\includegraphics[width=4.124cm]{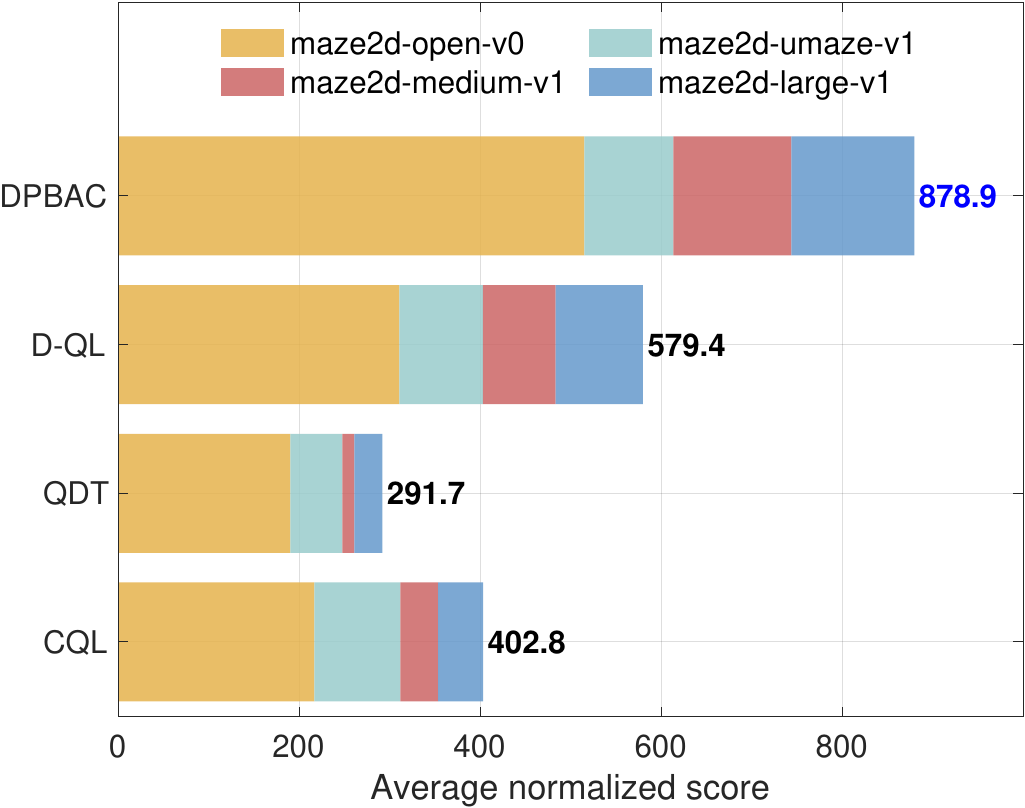}
	}
	\caption{Test results of algorithm performance on sparse-reward tasks. Evaluations are conducted on Maze2D tasks under both dense and sparse reward settings, with results averaged over eight random seeds.}
	\label{Horizon_stack}
\end{figure}

\subsection{Q-value Estimation Bias Analysis}
To evaluate the accuracy of \emph{Q}-value estimation in DPBAC algorithm, Fig. \ref{Q_value_Estimation_Bias} reports the estimation bias of the learned \emph{Q}-values under different correction strengths $\eta$ on the walker2d-medium-v2 and walker2d-medium-expert-v2 datasets. For a given initial state-action pair $(\bm{s}_0, \bm{a}_0)$, the \emph{Q}-value bias is computed as the absolute difference between the learned \emph{Q}-value $Q_{\theta_1}(\bm{s}_0, \bm{a}_0)$ and the corresponding ground-truth \emph{Q}-value. Following \cite{CaoSSMC2024}, the ground-truth \emph{Q}-value is estimated using the Monte Carlo method and serves as a reference for assessing the estimation bias of the learned \emph{Q}-values.

\begin{figure}[!ht]
	\centering
	\subfigure[walker2d-medium-v2]{
	\includegraphics[width=4.124cm]{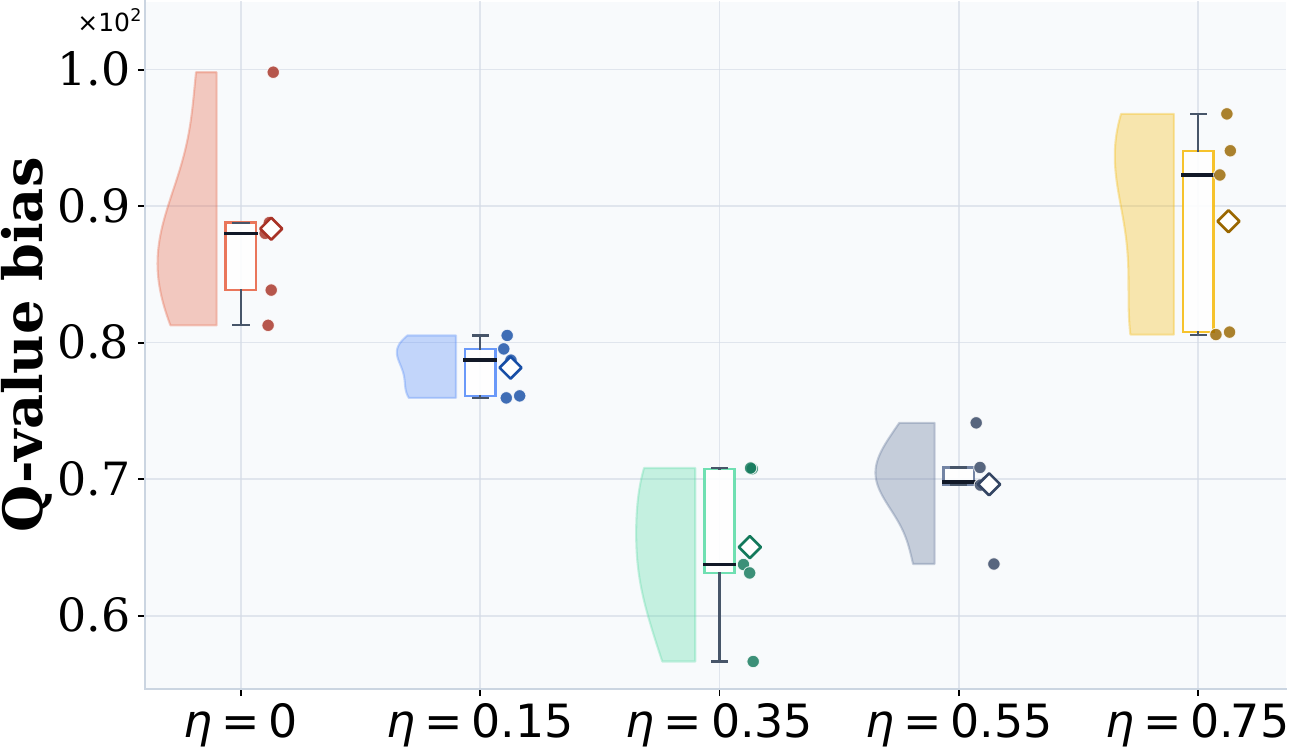}
	}
	\subfigure[walker2d-medium-expert-v2]{
	\includegraphics[width=4.124cm]{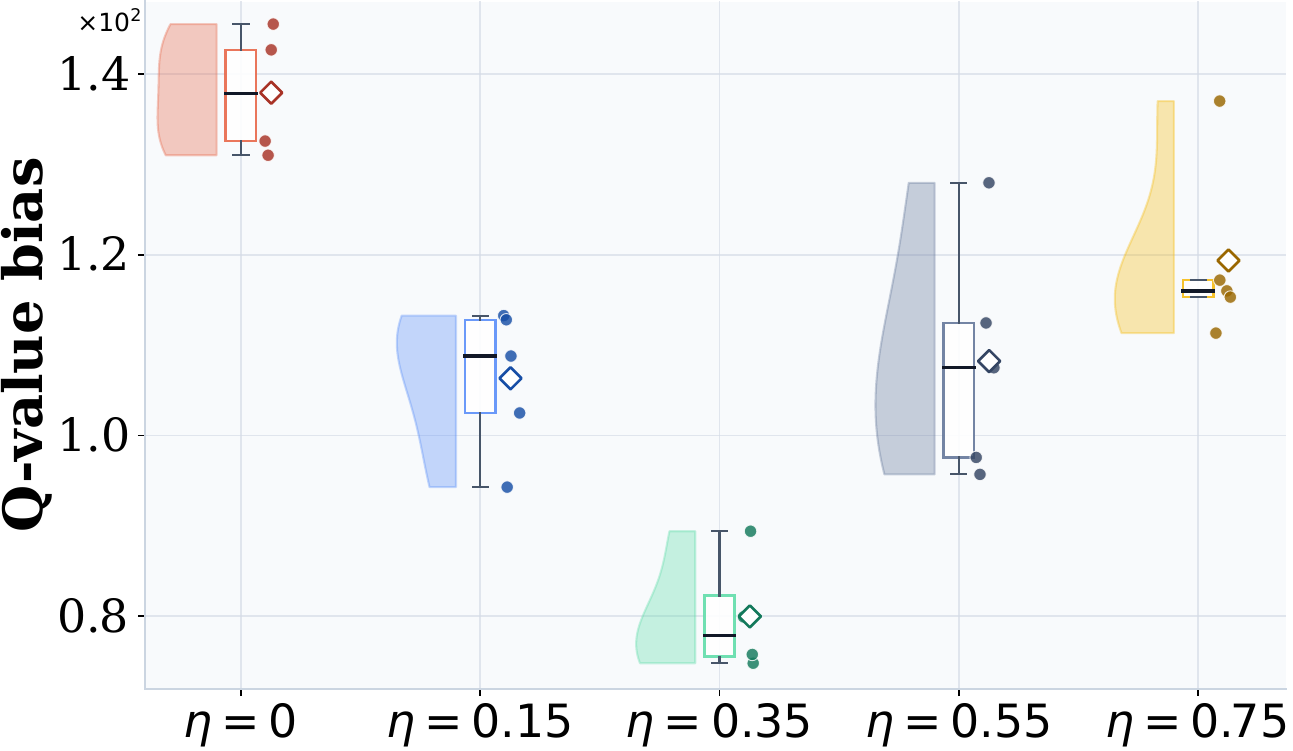}
	}
	\caption{Statistics of \emph{Q}-value bias for DPBAC under different correction coefficients $\eta$ on walker2d-medium-v2 and walker2d-medium-expert-v2 datasets.}
	\label{Q_value_Estimation_Bias}
\end{figure}

As shown in Fig. \ref{Q_value_Estimation_Bias}, when $\eta = 0$, no correction is applied to the learned \emph{Q}-values, resulting in a large absolute estimation error. As $\eta$ increases, this error gradually decreases, indicating that the proposed \emph{Q}-value correction mechanism effectively alleviates overestimation bias. However, when $\eta$ becomes excessively large, the bias increases again, suggesting that overly strong correction may lead to underestimation of the \emph{Q}-values. These results highlight the importance of selecting an appropriate correction strength $\eta$ to balance the mitigation of overestimation bias and the risk of underestimation, thereby enabling more accurate \emph{Q}-value estimation and more effective policy learning.

\section{Conclusion}\label{723843brevfkaf}
To address the pessimistic bias and \emph{Q}-function overestimation issues in existing offline RL algorithms, this study proposed a BAC-PE method. We established the convergence of BAC-PE and derived both the relationship between its converged \emph{Q}-function and the true \emph{Q}-function, and an upper bound on their difference. This upper bound is positively correlated with the distributional difference between the learned policy and the behavior policy. To minimize this distributional discrepancy, we utilized a diffusion model for policy modeling, enhancing the expressive ability of the learned policy to better capture multimodal behavior policies, leading to the development of DPBAC. Extensive experiments, including comparisons with SOTA baselines, sensitivity studies, sparse-reward evaluations, and \emph{Q}-value bias analysis across multiple D4RL task domains, demonstrate the superior performance of DPBAC.

\appendices

\section{Proof of Theorem \ref{194hb5476f4909j33}}\label{proof_of_theorem_1}
\begin{IEEEproof}
	Let $\widetilde{Q}_{\diamond}$ and $\widetilde{Q}_{\dagger}$ represent two distinct \emph{Q}-functions. Based on the iterative rule (\ref{dn3974t34th39}), it follows that:
	\begin{align}\label{dn38n467fvnb4}
		& \Big(\widehat{\mathcal{T}}^{\pi} \widetilde{Q}_{\diamond} \Big)(\bm{s}, \bm{a}) - \Big(\widehat{\mathcal{T}}^{\pi} \widetilde{Q}_{\dagger} \Big)(\bm{s}, \bm{a}) \nonumber \\ 
		& = \Big(\widehat{\mathcal{B}}^{\pi} \widetilde{Q}_{\diamond} \Big) (\bm{s}, \bm{a}) - \Big(\widehat{\mathcal{B}}^{\pi} \widetilde{Q}_{\dagger} \Big) (\bm{s}, \bm{a}) - \eta \Big(\widetilde{Q}_{\diamond}(\bm{s}, \bm{a}) - \widetilde{Q}_{\dagger}(\bm{s}, \bm{a})\Big) \nonumber \\ 
		& = \gamma \mathbb{E}_{\bm{s}' \sim \widehat{\mathbb{P}}(\cdot |\bm{s}, \bm{a}), \bm{a}' \sim \pi(\cdot |\bm{s}')} \Big[\widetilde{Q}_{\diamond}(\bm{s}', \bm{a}') - \widetilde{Q}_{\dagger}(\bm{s}', \bm{a}')\Big] \nonumber \\ & \ \ \ - \eta \Big(\widetilde{Q}_{\diamond}(\bm{s}, \bm{a}) - \widetilde{Q}_{\dagger}(\bm{s}, \bm{a})\Big).
	\end{align}
	Taking the absolute value on both sides of (\ref{dn38n467fvnb4}), one has
	\begin{align}\label{dn23r93r4t890}
		& \Big| \Big(\widehat{\mathcal{T}}^{\pi} \widetilde{Q}_{\diamond} \Big)(\bm{s}, \bm{a}) - \Big(\widehat{\mathcal{T}}^{\pi} \widetilde{Q}_{\dagger} \Big)(\bm{s}, \bm{a}) \Big| \nonumber \\
		& \ \ \ \leq \gamma \Big| \mathbb{E}_{\bm{s}' \sim \widehat{\mathbb{P}}(\cdot |\bm{s}, \bm{a}), \bm{a}' \sim \pi(\cdot |\bm{s}')} \Big[\widetilde{Q}_{\diamond}(\bm{s}', \bm{a}') - \widetilde{Q}_{\dagger}(\bm{s}', \bm{a}')\Big] \Big| \nonumber \\ & \ \ \ \ \ \ + \eta \Big|\widetilde{Q}_{\diamond}(\bm{s}, \bm{a}) - \widetilde{Q}_{\dagger}(\bm{s}, \bm{a}) \Big| \nonumber \\
		& \ \ \ \leq \gamma \mathbb{E}_{\bm{s}' \sim \widehat{\mathbb{P}}(\cdot |\bm{s}, \bm{a}), \bm{a}' \sim \pi(\cdot |\bm{s}')} \Big | \widetilde{Q}_{\diamond}(\bm{s}', \bm{a}') - \widetilde{Q}_{\dagger}(\bm{s}', \bm{a}')\Big | \nonumber \\ & \ \ \ \ \ \ + \eta \Big|\widetilde{Q}_{\diamond}(\bm{s}, \bm{a}) - \widetilde{Q}_{\dagger}(\bm{s}, \bm{a}) \Big| \nonumber \\ 
		& \ \ \ \leq \gamma \max_{\bm{s}',\bm{a}'} \Big | \widetilde{Q}_{\diamond}(\bm{s}', \bm{a}') - \widetilde{Q}_{\dagger}(\bm{s}', \bm{a}')\Big | \nonumber \\ & \ \ \ \ \ \ + \eta \max_{\bm{s},\bm{a}} \Big|\widetilde{Q}_{\diamond}(\bm{s}, \bm{a}) - \widetilde{Q}_{\dagger}(\bm{s}, \bm{a}) \Big| \nonumber \\ 
		& \ \ \ \leq \widetilde{\gamma} \max_{\bm{s},\bm{a}} \Big | \widetilde{Q}_{\diamond}(\bm{s}, \bm{a}) - \widetilde{Q}_{\dagger}(\bm{s}, \bm{a})\Big |, \forall (\bm{s}, \bm{a}). 
	\end{align}
	where $\widetilde{\gamma} = \gamma + \eta$. Equation (\ref{dn23r93r4t890}) can be reformulated in its vector representation as:
	\begin{align}\label{dan38654h3nc03}
		\big \| \widehat{\mathcal{T}}^{\pi} \widetilde{\bm{Q}}_{\diamond} - \widehat{\mathcal{T}}^{\pi} \widetilde{\bm{Q}}_{\dagger} \big \|_{\infty} \leq \widetilde{\gamma} \big \| \widetilde{\bm{Q}}_{\diamond} - \widetilde{\bm{Q}}_{\dagger} \big\|_{\infty},
	\end{align}
	which indicates $\widehat{\mathcal{T}}^{\pi}$ constitutes a contraction mapping, ensuring the existence of a fixed point $\widetilde{\bm{Q}}^{\pi}$ satisfying $\widetilde{\bm{Q}}^{\pi} = \widehat{\mathcal{T}}^{\pi} \widetilde{\bm{Q}}^{\pi}$. According to (\ref{dan38654h3nc03}), one has
	\begin{align}
		& \big\| \widetilde{\bm{Q}}_{k+1} - \widetilde{\bm{Q}}^{\pi} \big\|_{\infty} = \big\| \widehat{\mathcal{T}}^{\pi} \widetilde{\bm{Q}}_{k} - \widehat{\mathcal{T}}^{\pi} \widetilde{\bm{Q}}^{\pi} \big\|_{\infty} \nonumber \\
		& \ \ \ \leq \widetilde{\gamma} \big\| \widetilde{\bm{Q}}_{k} - \widetilde{\bm{Q}}^{\pi} \big\|_{\infty} = \widetilde{\gamma} \big\| \widehat{\mathcal{T}}^{\pi} \widetilde{\bm{Q}}_{k-1} - \widehat{\mathcal{T}}^{\pi} \widetilde{\bm{Q}}^{\pi} \big\|_{\infty} \nonumber \\
		& \ \ \ \leq \widetilde{\gamma}^2 \big\| \widetilde{\bm{Q}}_{k-1} - \widetilde{\bm{Q}}^{\pi} \big\|_{\infty} 
		\leq \cdots \leq \widetilde{\gamma}^{k+1} \big\| \widetilde{\bm{Q}}_{0} - \widetilde{\bm{Q}}^{\pi} \big\|_{\infty},
	\end{align}
	which demonstrates that the iterative sequence initiated from $\widetilde{\bm{Q}}_{0}$ converges to the fixed point $\widetilde{\bm{Q}}^{\pi}$.
\end{IEEEproof}

\section{Proof of Theorem \ref{dn239r8324r73840jfd9df}}\label{proof_of_theorem_2}
\begin{IEEEproof}
	Based on Theorem \ref{194hb5476f4909j33} and equation (\ref{dan73693hr39976g}), the following result holds:
	\begin{align}\label{adn3874rtn39489}
		\widetilde{\bm{Q}}^{\pi} & = \widehat{\mathcal{T}}^{\pi} \widetilde{\bm{Q}}^{\pi} \nonumber \\ 
		& = \widehat{\mathcal{B}}^{\pi} \widetilde{\bm{Q}}^{\pi} + \eta \widehat{\bm{A}}^{\mu}_{\widetilde{Q}^{\pi}} \nonumber \\
		& = \bm{r} + \gamma \widehat{\mathbf{P}}^{\pi} \widetilde{\bm{Q}}^{\pi} + \eta \big(\widehat{\bm{Q}}^{\mu} - \widetilde{\bm{Q}}^{\pi}\big).
	\end{align}
	From equation (\ref{dano382j0r537g67f4}), it follows that $\bm{Q}^{\pi} = \mathcal{B}^{\pi} \bm{Q}^{\pi} = \bm{r} + \gamma \mathbf{P}^{\pi} \bm{Q}^{\pi}$. Consequently, the vectorized form of reward function $\bm{r}$ can be expressed as $\bm{r} = (\bm{I} - \gamma \mathbf{P}^{\pi}) \bm{Q}^{\pi}$, which further leads to the derivation of (\ref{adn3874rtn39489}) as:
	\begin{align}\label{fm3944h84jnkro0m}
		\widetilde{\bm{Q}}^{\pi} & = (\bm{I} - \gamma \mathbf{P}^{\pi}) \bm{Q}^{\pi} + \gamma \widehat{\mathbf{P}}^{\pi} \widetilde{\bm{Q}}^{\pi} + \eta \big(\widehat{\bm{Q}}^{\mu} - \widetilde{\bm{Q}}^{\pi}\big) \nonumber \\
		& = (\bm{I} - \gamma \mathbf{P}^{\pi}) \bm{Q}^{\pi} + \gamma \mathbf{P}^{\pi} \widetilde{\bm{Q}}^{\pi} - \gamma \mathbf{P}^{\pi} \widetilde{\bm{Q}}^{\pi} \nonumber \\ & \ \ \ + \gamma \widehat{\mathbf{P}}^{\pi} \widetilde{\bm{Q}}^{\pi} + \eta \big(\widehat{\bm{Q}}^{\mu} - \widetilde{\bm{Q}}^{\pi}\big).
	\end{align}
	Reorganizing the terms in (\ref{fm3944h84jnkro0m}) yields:
	\begin{align}\label{dm203r23r8734r34989}
		& (\bm{I} - \gamma \mathbf{P}^{\pi}) \widetilde{\bm{Q}}^{\pi} = (\bm{I} - \gamma \mathbf{P}^{\pi}) \bm{Q}^{\pi} + \gamma \big(\widehat{\mathbf{P}}^{\pi} - \mathbf{P}^{\pi}\big) \widetilde{\bm{Q}}^{\pi} \nonumber \\
		& \ \ \ \ \ \ \ \ \ \ \ \ \ \ \ \ \ \ \ \ + \eta \big(\widehat{\bm{Q}}^{\mu} - \widetilde{\bm{Q}}^{\pi}\big).
	\end{align}
	By left-multiplying both sides of (\ref{dm203r23r8734r34989}) by $(\bm{I} - \gamma \mathbf{P}^{\pi})^{-1}$, equation (\ref{dan392893hr83}) is obtained.
\end{IEEEproof}

\section{Proof of Theorem \ref{sadn239er823rh923}}\label{proof_of_theorem_3}
\begin{IEEEproof}
	According to the analysis about the fixed-point of $\mathcal{T}^{\pi}$ presented in Corollary \ref{dn7539hr339h8}, one has
	\begin{align}\label{dm2138r2g85r2}
		& \overline{Q}^{\pi}(\bm{s}, \bm{a}) = \Big(\mathcal{T}^{\pi} \overline{Q}^{\pi}\Big)(\bm{s}, \bm{a}) \nonumber \\ 
		& = \Big(\mathcal{B}^{\pi} \overline{Q}^{\pi}\Big) (\bm{s}, \bm{a}) + \eta A^{\mu}_{\overline{Q}^{\pi}}(\bm{s}, \bm{a}) \nonumber \\
		& = r(\bm{s}, \bm{a}) + \gamma \mathbb{E}_{\substack{\bm{s}' \sim \mathbb{P}(\cdot |\bm{s}, \bm{a})\\\bm{a}' \sim \pi(\cdot |\bm{s}')}} \Big[\overline{Q}^{\pi}(\bm{s}', \bm{a}')\Big] + \eta A^{\mu}_{\overline{Q}^{\pi}}(\bm{s}, \bm{a}).
	\end{align}
	Based on the fixed-point equation of Bellman operator $\mathcal{B}^{\pi}$ in (\ref{dano382j0r537g67f4}), it follows that $Q^{\pi}(\bm{s}, \bm{a}) = (\mathcal{B}^{\pi} Q^{\pi}) (\bm{s}, \bm{a}) = r(\bm{s}, \bm{a}) + \gamma \mathbb{E}_{\bm{s}' \sim \mathbb{P}(\cdot |\bm{s}, \bm{a}), \bm{a}' \sim \pi(\cdot |\bm{s}')} [Q^{\pi}(\bm{s}', \bm{a}')]$, which leads to 
	\begin{align}\label{nd92137205g764765498u934}
		r(\bm{s}, \bm{a}) = Q^{\pi}(\bm{s}, \bm{a}) - \gamma \mathbb{E}_{\bm{s}' \sim \mathbb{P}(\cdot |\bm{s}, \bm{a}), \bm{a}' \sim \pi(\cdot |\bm{s}')} [Q^{\pi}(\bm{s}', \bm{a}')].
	\end{align}
	Substituting (\ref{nd92137205g764765498u934}) into (\ref{dm2138r2g85r2}) results in 
	\begin{align}
		& \overline{Q}^{\pi}(\bm{s}, \bm{a}) = Q^{\pi}(\bm{s}, \bm{a}) - \gamma \mathbb{E}_{\bm{s}' \sim \mathbb{P}(\cdot |\bm{s}, \bm{a}), \bm{a}' \sim \pi(\cdot |\bm{s}')} [Q^{\pi}(\bm{s}', \bm{a}')] \nonumber \\
		& \ + \gamma \mathbb{E}_{\bm{s}' \sim \mathbb{P}(\cdot |\bm{s}, \bm{a}), \bm{a}' \sim \pi(\cdot |\bm{s}')} \Big[\overline{Q}^{\pi}(\bm{s}', \bm{a}')\Big] + \eta A^{\mu}_{\overline{Q}^{\pi}}(\bm{s}, \bm{a}),
	\end{align}
	which implies that 
	\begin{align}\label{d2n3782r67349634b83}
		& \overline{Q}^{\pi}(\bm{s}, \bm{a}) - Q^{\pi}(\bm{s}, \bm{a}) = \eta \big(Q^{\mu}(\bm{s}, \bm{a}) - \overline{Q}^{\pi}(\bm{s}, \bm{a})\big) \nonumber \\ 
		& + \gamma \mathbb{E}_{\bm{s}' \sim \mathbb{P}(\cdot |\bm{s}, \bm{a}),\bm{a}' \sim \pi(\cdot |\bm{s}')} \Big[\overline{Q}^{\pi}(\bm{s}', \bm{a}') - Q^{\pi}(\bm{s}', \bm{a}') \Big] \nonumber \\ 
		& + \eta Q^{\pi}(\bm{s}, \bm{a}) - \eta Q^{\pi}(\bm{s}, \bm{a}) \nonumber \\
		= & \ \eta \big(Q^{\pi}(\bm{s}, \bm{a}) - \overline{Q}^{\pi}(\bm{s}, \bm{a})\big) + \eta \big(Q^{\mu}(\bm{s}, \bm{a}) - Q^{\pi}(\bm{s}, \bm{a})\big) \nonumber \\ 
		& + \gamma \mathbb{E}_{\bm{s}' \sim \mathbb{P}(\cdot |\bm{s}, \bm{a}),\bm{a}' \sim \pi(\cdot |\bm{s}')} \Big[\overline{Q}^{\pi}(\bm{s}', \bm{a}') - Q^{\pi}(\bm{s}', \bm{a}') \Big].
	\end{align}
	Computing the absolute value on both sides of (\ref{d2n3782r67349634b83}) yields:
	\begin{align}\label{31294h86fcj3893}
		& \Big| \overline{Q}^{\pi}(\bm{s}, \bm{a}) - Q^{\pi}(\bm{s}, \bm{a}) \Big| \nonumber \\ 
		& \ \ \leq \eta \Big| Q^{\pi}(\bm{s}, \bm{a}) - \overline{Q}^{\pi}(\bm{s}, \bm{a}) \Big| + \eta \Big| Q^{\mu}(\bm{s}, \bm{a}) - Q^{\pi}(\bm{s}, \bm{a}) \Big| \nonumber \\ 
		& \ \ \ \ \ + \gamma \Big| \mathbb{E}_{\bm{s}' \sim \mathbb{P}(\cdot |\bm{s}, \bm{a}),\bm{a}' \sim \pi(\cdot |\bm{s}')} \Big[\overline{Q}^{\pi}(\bm{s}', \bm{a}') - Q^{\pi}(\bm{s}', \bm{a}') \Big] \Big| \nonumber \\ 
		& \ \ \leq \eta \Big| \overline{Q}^{\pi}(\bm{s}, \bm{a}) - Q^{\pi}(\bm{s}, \bm{a}) \Big| + \eta \Big| Q^{\mu}(\bm{s}, \bm{a}) - Q^{\pi}(\bm{s}, \bm{a}) \Big| \nonumber \\ 
		& \ \ \ \ \ + \gamma \mathbb{E}_{\bm{s}' \sim \mathbb{P}(\cdot |\bm{s}, \bm{a}),\bm{a}' \sim \pi(\cdot |\bm{s}')} \Big| \overline{Q}^{\pi}(\bm{s}', \bm{a}') - Q^{\pi}(\bm{s}', \bm{a}') \Big| \nonumber \\
		& \ \ \leq \eta \max_{\bm{s},\bm{a}} \Big[\Big| \overline{Q}^{\pi}(\bm{s}, \bm{a}) - Q^{\pi}(\bm{s}, \bm{a}) \Big| + \Big| Q^{\mu}(\bm{s}, \bm{a}) - Q^{\pi}(\bm{s}, \bm{a}) \Big| \Big] \nonumber \\
		& \ \ \ \ \ + \gamma \max_{\bm{s}',\bm{a}'} \Big| \overline{Q}^{\pi}(\bm{s}', \bm{a}') - Q^{\pi}(\bm{s}', \bm{a}') \Big| \nonumber \\
		& \ \ \leq \widetilde{\gamma} \max_{\bm{s},\bm{a}} \Big| \overline{Q}^{\pi}(\bm{s}, \bm{a}) - Q^{\pi}(\bm{s}, \bm{a}) \Big| \nonumber \\
		& \ \ \ \ \ + \eta \max_{\bm{s},\bm{a}} \Big| Q^{\mu}(\bm{s}, \bm{a}) - Q^{\pi}(\bm{s}, \bm{a}) \Big|, 
	\end{align}
	with its vectorized form represented as:
	\begin{align}\label{dm93r72r7674n}
		\big\| \overline{\bm{Q}}^{\pi} - \bm{Q}^{\pi} \big\|_{\infty} & \leq \widetilde{\gamma} \big\| \overline{\bm{Q}}^{\pi} - \bm{Q}^{\pi} \big\|_{\infty} + \eta \big\| \bm{Q}^{\mu} - \bm{Q}^{\pi}\big\|_{\infty} \nonumber \\ 
		& \leq \frac{\eta}{1 - \widetilde{\gamma}} \big\| \bm{Q}^{\mu} - \bm{Q}^{\pi}\big\|_{\infty}.
	\end{align}
	Following Lemma \ref{en32976dg37239922}, (\ref{en238437gbh48784}) is straightforwardly obtained.  
\end{IEEEproof}

\section{Proof of Lemma \ref{d2n32983898y74bbnf}}\label{proof_of_lemma_2}
\begin{IEEEproof}
	By the triangle inequality, it follows that:
	\begin{align}\label{md921636gb73879nn}
		\Big\| \bm{Q}^{\pi} - \widehat{\bm{Q}}^{\mu} \Big\|_{\infty} & \leq \Big\| \bm{Q}^{\pi} - \bm{Q}^{\mu} + \bm{Q}^{\mu} - \widehat{\bm{Q}}^{\mu} \Big\|_{\infty} \nonumber \\ 
		& \leq \Big\| \bm{Q}^{\pi} - \bm{Q}^{\mu} \Big\|_{\infty} + \Big\| \bm{Q}^{\mu} - \widehat{\bm{Q}}^{\mu} \Big\|_{\infty}.
	\end{align}
	Based on the definitions of $\mathcal{B}^{\pi}$ in (\ref{dano382j0r537g67f4}) and $\widehat{\mathcal{B}}^{\pi}$ in (\ref{dn38hf7439b3uef}), one has
	\begin{align}
		\bm{Q}^{\mu} & = \mathcal{B}^{\mu} \bm{Q}^{\mu} = \bm{r} + \gamma \mathbf{P}^{\mu} \bm{Q}^{\mu}, \nonumber \\
		\widehat{\bm{Q}}^{\mu} & = \widehat{\mathcal{B}}^{\mu} \widehat{\bm{Q}}^{\mu} = \bm{r} + \gamma \widehat{\mathbf{P}}^{\mu} \widehat{\bm{Q}}^{\mu},
	\end{align}
	which implies that
	\begin{align}\label{dnm238nf34893ngft}
		\widehat{\bm{Q}}^{\mu} - \bm{Q}^{\mu} & = \gamma \widehat{\mathbf{P}}^{\mu} \widehat{\bm{Q}}^{\mu} - \gamma \mathbf{P}^{\mu} \bm{Q}^{\mu} \nonumber \\
		& = \gamma \widehat{\mathbf{P}}^{\mu} \widehat{\bm{Q}}^{\mu} - \gamma \widehat{\mathbf{P}}^{\mu} \bm{Q}^{\mu} + \gamma \widehat{\mathbf{P}}^{\mu} \bm{Q}^{\mu} - \gamma \mathbf{P}^{\mu} \bm{Q}^{\mu} \nonumber \\
		& = \gamma \widehat{\mathbf{P}}^{\mu} \big(\widehat{\bm{Q}}^{\mu} - \bm{Q}^{\mu}\big) + \gamma \big(\widehat{\mathbf{P}}^{\mu} - \mathbf{P}^{\mu} \big) \bm{Q}^{\mu} \nonumber \\ 
		& = \gamma \big(\bm{I} - \gamma \widehat{\mathbf{P}}^{\mu}\big)^{-1} \big(\widehat{\mathbf{P}}^{\mu} - \mathbf{P}^{\mu} \big) \bm{Q}^{\mu}.
	\end{align}
	Applying the maximum norm to both sides of (\ref{dnm238nf34893ngft}) yields:
	\begin{align}\label{dn23n98r7347638}
		\Big\| \widehat{\bm{Q}}^{\mu} - \bm{Q}^{\mu} \Big\|_{\infty} = \gamma \Big\| \Big(\bm{I} - \gamma \widehat{\mathbf{P}}^{\mu}\Big)^{-1} \Big\|_{\infty} \Big\| \Big(\widehat{\mathbf{P}}^{\mu} - \mathbf{P}^{\mu} \Big) \bm{Q}^{\mu} \Big\|_{\infty}.
	\end{align}
	Next, the term $\big(\widehat{\mathbf{P}}^{\mu} - \mathbf{P}^{\mu} \big) \bm{Q}^{\mu}$ in (\ref{dn23n98r7347638}) is analyzed as follows:
	\begin{align}\label{nd838m49784}
		& \Big(\widehat{\mathbf{P}}^{\mu} - \mathbf{P}^{\mu} \Big) Q^{\mu}(\bm{s}, \bm{a}) = \widehat{\mathbf{P}}^{\mu} Q^{\mu}(\bm{s}, \bm{a}) - \mathbf{P}^{\mu} Q^{\mu}(\bm{s}, \bm{a}) \nonumber \\
		& \ \ \ \ = \mathbb{E}_{\substack{\bm{s}' \sim \widehat{\mathbb{P}}(\cdot |\bm{s}, \bm{a})\\\bm{a}' \sim \mu(\cdot |\bm{s}')}} \big[Q^{\mu}(\bm{s}', \bm{a}')\big] - \mathbb{E}_{\substack{\bm{s}' \sim \mathbb{P}(\cdot |\bm{s}, \bm{a})\\\bm{a}' \sim \mu(\cdot |\bm{s}')}} \big[Q^{\mu}(\bm{s}', \bm{a}')\big] \nonumber \\
		& \ \ \ \ = \int \Big(\widehat{\mathbb{P}}(\bm{s}' |\bm{s}, \bm{a}) - \mathbb{P}(\bm{s}' |\bm{s}, \bm{a})\Big) \mathbb{E}_{\bm{a}' \sim \mu(\cdot |\bm{s}')} \big[Q^{\mu}(\bm{s}', \bm{a}')\big] \,d\bm{s}'.
	\end{align}
	Considering the absolute value of both sides of (\ref{nd838m49784}), one has:
	\begin{align}\label{dn329847bftht}
		& \Big| \big(\widehat{\mathbf{P}}^{\mu} - \mathbf{P}^{\mu} \big) Q^{\mu}(\bm{s}, \bm{a}) \Big| \nonumber \\
		& = \bigg| \int \Big(\widehat{\mathbb{P}}(\bm{s}' |\bm{s}, \bm{a}) - \mathbb{P}(\bm{s}' |\bm{s}, \bm{a})\Big) \mathbb{E}_{\bm{a}' \sim \mu(\cdot |\bm{s}')} \big[Q^{\mu}(\bm{s}', \bm{a}')\big] \,d\bm{s}' \bigg| \nonumber \\
		& \leq \int \Big|\widehat{\mathbb{P}}(\bm{s}' |\bm{s}, \bm{a}) - \mathbb{P}(\bm{s}' |\bm{s}, \bm{a})\Big| \Big| \mathbb{E}_{\bm{a}' \sim \mu(\cdot |\bm{s}')} \big[Q^{\mu}(\bm{s}', \bm{a}')\big] \Big| \,d\bm{s}' \nonumber \\
		& \leq \int \Big|\widehat{\mathbb{P}}(\bm{s}' |\bm{s}, \bm{a}) - \mathbb{P}(\bm{s}' |\bm{s}, \bm{a})\Big| \,d\bm{s}' \max_{\bm{s}',\bm{a}'} \big|Q^{\mu}(\bm{s}', \bm{a}') \big| \nonumber \\ 
		& = 2 \varUpsilon_{TV}\Big[\widehat{\mathbb{P}}(\cdot |\bm{s}, \bm{a}) \big\Vert \mathbb{P}(\cdot |\bm{s}, \bm{a}) \Big] \max_{\bm{s},\bm{a}} \big|Q^{\mu}(\bm{s}, \bm{a}) \big|.
	\end{align}
	Due to the property $\max_{\bm{s},\bm{a}} |Q^{\mu}(\bm{s}, \bm{a}) | < \overline{r}/(1-\gamma)$, it can be deduced from (\ref{dn329847bftht}) that
	\begin{align}\label{nd921373rb3878835}
		\Big\| \big(\widehat{\mathbf{P}}^{\mu} - \mathbf{P}^{\mu} \big) \bm{Q}^{\mu} \Big\|_{\infty} \leq \frac{2 \overline{r}}{1-\gamma} \max_{\bm{s},\bm{a}} \varUpsilon_{TV}\Big[\widehat{\mathbb{P}}(\cdot |\bm{s}, \bm{a}) \big\Vert \mathbb{P}(\cdot |\bm{s}, \bm{a}) \Big].
	\end{align}
	By synthesizing (\ref{d29387reh2937444}), (\ref{dn23n98r7347638}), and (\ref{nd921373rb3878835}), equation (\ref{dj2903r83b8h939hn34}) can be derived from (\ref{md921636gb73879nn}).
\end{IEEEproof}

\section{Proof of Theorem \ref{dmff2398472}}\label{proof_of_theorem_4}
\begin{IEEEproof}
	From the fixed-point equation of $\widehat{\mathcal{T}}^{\pi}$ presented in Theorem \ref{194hb5476f4909j33}, one can derive:
	\begin{align}\label{dn8327r4379fh9434}
		\widetilde{Q}^{\pi}(\bm{s}, \bm{a}) & = \Big(\widehat{\mathcal{T}}^{\pi} \widetilde{Q}^{\pi}\Big)(\bm{s}, \bm{a}) \nonumber \\ 
		& = \Big(\widehat{\mathcal{B}}^{\pi} \widetilde{Q}^{\pi}\Big) (\bm{s}, \bm{a}) + \eta \widehat{A}^{\mu}_{\widetilde{Q}^{\pi}}(\bm{s}, \bm{a}) \nonumber \\ 
		& = r(\bm{s}, \bm{a}) + \gamma \mathbb{E}_{\bm{s}' \sim \widehat{\mathbb{P}}(\cdot |\bm{s}, \bm{a}),\bm{a}' \sim \pi(\cdot |\bm{s}')} \Big[\widetilde{Q}^{\pi}(\bm{s}', \bm{a}')\Big] \nonumber \\ & \ \ \ + \eta \widehat{A}^{\mu}_{\widetilde{Q}^{\pi}}(\bm{s}, \bm{a}).
	\end{align}
	Incorporating (\ref{nd92137205g764765498u934}) into (\ref{dn8327r4379fh9434}) leads to:
	\begin{align}\label{sdn2398325h27838}
		& \widetilde{Q}^{\pi}(\bm{s}, \bm{a}) = Q^{\pi}(\bm{s}, \bm{a}) - \gamma \mathbb{E}_{\bm{s}' \sim \mathbb{P}(\cdot |\bm{s}, \bm{a}), \bm{a}' \sim \pi(\cdot |\bm{s}')} \big[Q^{\pi}(\bm{s}', \bm{a}')\big] \nonumber \\
		& \ + \gamma \mathbb{E}_{\bm{s}' \sim \widehat{\mathbb{P}}(\cdot |\bm{s}, \bm{a}),\bm{a}' \sim \pi(\cdot |\bm{s}')} \big[\widetilde{Q}^{\pi}(\bm{s}', \bm{a}')\big] + \eta \widehat{A}^{\mu}_{\widetilde{Q}^{\pi}}(\bm{s}, \bm{a}).
	\end{align} 
	It can be observed that
	\begin{align}\label{dm3976543bnds32l}
		& \mathbb{E}_{\substack{\bm{s}' \sim \widehat{\mathbb{P}}(\cdot |\bm{s}, \bm{a})\\\bm{a}' \sim \pi(\cdot |\bm{s}')}} \Big[\widetilde{Q}^{\pi}(\bm{s}', \bm{a}')\Big] - \mathbb{E}_{\substack{\bm{s}' \sim \mathbb{P}(\cdot |\bm{s}, \bm{a}) \\ \bm{a}' \sim \pi(\cdot |\bm{s}')}} \Big[Q^{\pi}(\bm{s}', \bm{a}')\Big] \nonumber \\ 
		& = \mathbb{E}_{\substack{\bm{s}' \sim \widehat{\mathbb{P}}(\cdot |\bm{s}, \bm{a})\\\bm{a}' \sim \pi(\cdot |\bm{s}')}} \Big[\widetilde{Q}^{\pi}(\bm{s}', \bm{a}')\Big] - \mathbb{E}_{\substack{\bm{s}' \sim \widehat{\mathbb{P}}(\cdot |\bm{s}, \bm{a})\\\bm{a}' \sim \pi(\cdot |\bm{s}')}} \Big[Q^{\pi}(\bm{s}', \bm{a}')\Big] \nonumber \\
		& \ \ \ \ + \mathbb{E}_{\substack{\bm{s}' \sim \widehat{\mathbb{P}}(\cdot |\bm{s}, \bm{a})\\\bm{a}' \sim \pi(\cdot |\bm{s}')}} \Big[Q^{\pi}(\bm{s}', \bm{a}')\Big] - \mathbb{E}_{\substack{\bm{s}' \sim \mathbb{P}(\cdot |\bm{s}, \bm{a}) \\ \bm{a}' \sim \pi(\cdot |\bm{s}')}} \Big[Q^{\pi}(\bm{s}', \bm{a}')\Big] \nonumber \\
		& = \int \Big(\widehat{\mathbb{P}}(\bm{s}' |\bm{s}, \bm{a}) - \mathbb{P}(\bm{s}' |\bm{s}, \bm{a})\Big) \mathbb{E}_{\bm{a}' \sim \pi(\cdot |\bm{s}')} \big[Q^{\pi}(\bm{s}', \bm{a}')\big] \,d\bm{s}' \nonumber \\
		& \ \ \ \ + \mathbb{E}_{\bm{s}' \sim \widehat{\mathbb{P}}(\cdot |\bm{s}, \bm{a}),\bm{a}' \sim \pi(\cdot |\bm{s}')} \Big[\widetilde{Q}^{\pi}(\bm{s}', \bm{a}') - Q^{\pi}(\bm{s}', \bm{a}')\Big].
	\end{align}
	Building on (\ref{dm3976543bnds32l}), (\ref{sdn2398325h27838}) can be derived as follows:
	\begin{align}\label{dnmn329845n349rnff}
		& \widetilde{Q}^{\pi}(\bm{s}, \bm{a}) - Q^{\pi}(\bm{s}, \bm{a}) = \gamma \mathbb{E}_{\substack{\bm{s}' \sim \widehat{\mathbb{P}}(\cdot |\bm{s}, \bm{a})\\\bm{a}' \sim \pi(\cdot |\bm{s}')}} \Big[\widetilde{Q}^{\pi}(\bm{s}', \bm{a}') - Q^{\pi}(\bm{s}', \bm{a}')\Big] \nonumber \\
		& + \gamma \int \Big(\widehat{\mathbb{P}}(\bm{s}' |\bm{s}, \bm{a}) - \mathbb{P}(\bm{s}' |\bm{s}, \bm{a})\Big) \mathbb{E}_{\bm{a}' \sim \pi(\cdot |\bm{s}')} \big[Q^{\pi}(\bm{s}', \bm{a}')\big] \,d\bm{s}' \nonumber \\ 
		& + \eta \Big(Q^{\pi}(\bm{s}, \bm{a}) - \widetilde{Q}^{\pi}(\bm{s}, \bm{a}) + \widehat{Q}^{\mu}(\bm{s}, \bm{a}) - Q^{\pi}(\bm{s}, \bm{a})\Big).
	\end{align}
	Considering the absolute value of both sides of (\ref{dnmn329845n349rnff}), one has:
	\begin{align}\label{nd83903biierer}
		& \Big| \widetilde{Q}^{\pi}(\bm{s}, \bm{a}) - Q^{\pi}(\bm{s}, \bm{a}) \Big| \leq \eta \Big| \widetilde{Q}^{\pi}(\bm{s}, \bm{a}) - Q^{\pi}(\bm{s}, \bm{a}) \Big| \nonumber \\ 
		& \ + \gamma \bigg| \int \Big(\widehat{\mathbb{P}}(\bm{s}' |\bm{s}, \bm{a}) - \mathbb{P}(\bm{s}' |\bm{s}, \bm{a})\Big) \mathbb{E}_{\bm{a}' \sim \pi(\cdot |\bm{s}')} \big[Q^{\pi}(\bm{s}', \bm{a}')\big] \,d\bm{s}' \bigg| \nonumber \\ 
		& \ + \gamma \Big| \mathbb{E}_{\bm{s}' \sim \widehat{\mathbb{P}}(\cdot |\bm{s}, \bm{a}),\bm{a}' \sim \pi(\cdot |\bm{s}')} \Big[\widetilde{Q}^{\pi}(\bm{s}', \bm{a}') - Q^{\pi}(\bm{s}', \bm{a}')\Big] \Big| \nonumber \\
		& \ + \eta \Big| \widehat{Q}^{\mu}(\bm{s}, \bm{a}) - Q^{\pi}(\bm{s}, \bm{a}) \Big| \nonumber \\
		& \leq \eta \Big| \widetilde{Q}^{\pi}(\bm{s}, \bm{a}) - Q^{\pi}(\bm{s}, \bm{a}) \Big| + \eta \Big| \widehat{Q}^{\mu}(\bm{s}, \bm{a}) - Q^{\pi}(\bm{s}, \bm{a}) \Big| \nonumber \\ 
		& \ + \gamma \int \Big|\widehat{\mathbb{P}}(\bm{s}' |\bm{s}, \bm{a}) - \mathbb{P}(\bm{s}' |\bm{s}, \bm{a})\Big| \mathbb{E}_{\bm{a}' \sim \pi(\cdot |\bm{s}')} \big[\big|Q^{\pi}(\bm{s}', \bm{a}')\big| \big] \,d\bm{s}' \nonumber \\
		& \ + \gamma \mathbb{E}_{\bm{s}' \sim \widehat{\mathbb{P}}(\cdot |\bm{s}, \bm{a}),\bm{a}' \sim \pi(\cdot |\bm{s}')} \Big[\Big|\widetilde{Q}^{\pi}(\bm{s}', \bm{a}') - Q^{\pi}(\bm{s}', \bm{a}')\Big| \Big] \nonumber \\
		& \leq 2 \gamma  \varUpsilon_{TV}\Big[\widehat{\mathbb{P}}(\cdot |\bm{s}, \bm{a}) \big\Vert \mathbb{P}(\cdot |\bm{s}, \bm{a}) \Big] \max_{\bm{s},\bm{a}} \big|Q^{\pi}(\bm{s}, \bm{a})\big| \nonumber \\ 
		& \ + (\gamma + \eta) \max_{\bm{s},\bm{a}} \Big| \widetilde{Q}^{\pi}(\bm{s}, \bm{a}) - Q^{\pi}(\bm{s}, \bm{a}) \Big| \nonumber \\ 
		& \ + \eta \max_{\bm{s},\bm{a}} \Big| \widehat{Q}^{\mu}(\bm{s}, \bm{a}) - Q^{\pi}(\bm{s}, \bm{a}) \Big|.
	\end{align}
	With the property that $\max_{\bm{s},\bm{a}} |Q^{\pi}(\bm{s}, \bm{a}) | < \overline{r}/(1-\gamma)$, (\ref{nd83903biierer}) can be reformulated as:
	\begin{align}\label{dnn9327665219ncedllv}
		& \Big\| \widetilde{\bm{Q}}^{\pi} - \bm{Q}^{\pi} \Big\|_{\infty} \leq \widetilde{\gamma} \Big\| \widetilde{\bm{Q}}^{\pi} - \bm{Q}^{\pi} \Big\|_{\infty} + \eta \Big\| \widehat{\bm{Q}}^{\mu} - \bm{Q}^{\pi} \Big\|_{\infty} \nonumber \\ 
		& \ \ \ \ \ \ \ \ \ \ \ \ \ \ \ \ \ \ \ \ \ \ \ \ \ + \frac{2 \gamma \overline{r}}{1-\gamma} \max_{\bm{s},\bm{a}} \varUpsilon_{TV}\Big[\widehat{\mathbb{P}}(\cdot |\bm{s}, \bm{a}) \big\Vert \mathbb{P}(\cdot |\bm{s}, \bm{a}) \Big] \nonumber \\
		& \ \ \ \ \ \ \ \ \ \ \ \ \ \leq \frac{2 \gamma \overline{r}}{(1-\gamma)(1 - \widetilde{\gamma})} \max_{\bm{s},\bm{a}} \varUpsilon_{TV}\Big[\widehat{\mathbb{P}}(\cdot |\bm{s}, \bm{a}) \big\Vert \mathbb{P}(\cdot |\bm{s}, \bm{a}) \Big] \nonumber \\
		& \ \ \ \ \ \ \ \ \ \ \ \ \ \ \ \ \ \ \ \ \ \ \ \ \ + \frac{\eta}{1 - \widetilde{\gamma}} \Big\| \widehat{\bm{Q}}^{\mu} - \bm{Q}^{\pi} \Big\|_{\infty}.
	\end{align}
	By combining (\ref{dj2903r83b8h939hn34}) with (\ref{dnn9327665219ncedllv}), (\ref{dnn38794br874r7484r49}) can be obtained.
\end{IEEEproof}

\ifCLASSOPTIONcaptionsoff
\newpage
\fi


\begin{thebibliography}{00}

\bibitem{Sutton1}
R.~S.~Sutton and A.~G.~Barto, {\it Reinforcement Learning.} Cambridge, MA, USA: MIT Press, 1998.

\bibitem{Feng2025TCDS}
Y. Feng, Z. Wu, J. Wang, S. Li, Y. Huang, J. Yu, and M. Tan, “Decentralized reinforcement learning for multiple robotic fish in cooperative pursuit task,” {\it IEEE Trans. Cogn. Develop. Syst.}, vol. 17, no. 4, pp. 1022--1034, Aug. 2025.

\bibitem{WangH2024}
C. Sun, X. Wu, Y. Su, X. Shi, and C. Sun, “Multithreaded asynchronous deep reinforcement learning with multisensor fusion for robot collision avoidance,” {\it IEEE Trans. Neural Netw. Learn. Syst.}, vol. 36, no. 9, pp. 16128--16142, Sep. 2025.

\bibitem{9817657}
C. Pan, Z. Peng, Y. Li, B. Han, and D. Wang, “Flocking of under-actuated unmanned surface vehicles via deep reinforcement learning and model predictive path integral control,” {\it IEEE Trans. Instrum. Meas.}, vol. 73, 2024, Art. no. 2505011.

\bibitem{Wang_X2023}
Y. Liu, Q. Zhang, Y. Gao, and D. Zhao, “Deep-reinforcement-learning-based driving policy at intersections utilizing lane graph networks,” {\it IEEE Trans. Cogn. Develop. Syst.}, vol. 16, no. 5, pp. 1759--1774, Oct. 2024.

\bibitem{Chen2023Milestones}
L.-Y. Hao, G. Dong, T. Li, and Z. Peng, “Path-following control with obstacle avoidance of autonomous surface vehicles subject to actuator faults,” {\it IEEE/CAA J. Autom. Sinica}, vol. 11, no. 4, pp. 956--964, Apr. 2024.

\bibitem{JiangY2024}
X. Huang, Y. Cheng, Q. Yu, and X. Wang, “Deep reinforcement learning for autonomous driving based on safety experience replay,” {\it IEEE Trans. Cogn. Develop. Syst.}, vol. 16, no. 6, pp. 2070--2084, Dec. 2024.

\bibitem{Dogru2024}
O. Dogru, J. Xie, O. Prakash, R. Chiplunkar, J. Soesanto, H. Chen, K. Velswamy, F. Ibrahim, and B. Huang, “Reinforcement learning in process industries: review and perspective,” {\it IEEE/CAA J. Autom. Sinica}, vol. 11, no. 2, pp. 283--300, Feb. 2024.

\bibitem{CuiY_TNNLS2023}
X. Hai, Q. Feng, W. Chen, C. Wen, and A. W. Khong, “Capability-oriented decision-making in multi-uav deployment and task allocation: A hierarchical game-based framework,” {\it IEEE Trans. Syst., Man, Cybern., Syst.}, vol. 55, no. 7, pp. 4562--4574, Jul. 2025.

\bibitem{CuiY_TNNLS2024}
W. Huang, Y. Cui, H. Li, and X. Wu, “Practical probabilistic model-based reinforcement learning by integrating dropout uncertainty and trajectory sampling,” {\it IEEE Trans. Neural Netw. Learn. Syst.}, vol. 36, no. 7, pp. 12812--12826, Jul. 2025.

\bibitem{WHChen2024}
S. Wei, X. Wang, X. Feng, and H. Yu, “MAST: Multiagent safe transformer for reinforcement learning,” {\it IEEE Trans. Cogn. Develop. Syst.}, vol. 17, no. 4, pp. 976--986, Aug. 2025.

\bibitem{Yang2022}
Y. Yang, Z. Ding, R. Wang, H. Modares, and D. C. Wunsch, “Data-driven human-robot interaction without velocity measurement using off-policy reinforcement learning,” {\it IEEE/CAA J. Autom. Sinica}, vol. 9, no. 1, pp. 47--63, Jan. 2022.

\bibitem{CaoSSMC2024}
L. Huang, B. Dong, W. Xie, and W. Zhang, “Offline reinforcement learning with behavior value regularization,” {\it IEEE Trans. Cybern.}, vol. 54, no. 6, pp. 3692--3704, Jun. 2024.

\bibitem{XiaL2024}
K. Jiang, W. Liu, Y. Wang, L. Dong, and C. Sun, “Discovering latent variables for the tasks with confounders in multi-agent reinforcement learning,” {\it IEEE/CAA J. Autom. Sinica}, vol. 11, no. 7, pp. 1591--1604, Jul. 2024.

\bibitem{Prudencio2023}
R. F. Prudencio, M. R. O. A. Maximo, and E. L. Colombini, “A survey on offline reinforcement learning: Taxonomy, review, and open problems,” {\it IEEE Trans. Neural Netw. Learn. Syst.}, vol. 35, no. 8, pp. 10237--10257, Aug. 2024.

\bibitem{WangJ2023}
J. Wang, J. Zhang, H. Jiang, J. Zhang, L. Wang, and C. Zhang, “Offline meta reinforcement learning with in-distribution online adaptation,” in {\it Proc. Int. Conf. Mach. Learn.}, vol. 202, 2023, pp. 36626--36669.

\bibitem{LiuJAAAI}
J. Liu, Z. Zhang, Z. Wei, Z. Zhuang, Y. Kang, S. Gai, and D. Wang, “Beyond ood state actions: Supported cross-domain offline reinforcement learning,” in {\it Proc. AAAI Conf. Artif. Intell.}, vol. 38, no. 12, 2024, pp. 13945--13953.

\bibitem{HuS2024}
S. Hu, Z. Fan, C. Huang, L. Shen, Y. Zhang, Y. Wang, and D. Tao, “Q-value regularized transformer for offline reinforcement learning,” in {\it Proc. Int. Conf. Mach. Learn.}, 2024, pp. 19165--19181.

\bibitem{YuanZ2024}
Z. Yuan, Z. Zhang, X. Li, Y. Cui, M. Li, and X. Ban, “Controlling partially observed industrial system based on offline reinforcement learning--a case study of paste thickener,” {\it IEEE Trans. Ind. Informat.}, vol. 21, no. 1, pp. 49--59, Jan. 2025.

\bibitem{MiaoC2024}
N. Pang, L. Huang, B. Dong, H. Chen, X. Wang, and W. Zhang, “Sarsa-augmented off-policy reinforcement learning,” {\it IEEE Trans. Cogn. Develop. Syst.}, vol. 18, no. 2, pp. 504--517, Apr. 2026.

\bibitem{ZhangY2024}
Y. Zhang, J. Liu, C. Li, Y. Niu, Y. Yang, Y. Liu, and W. Ouyang, “A perspective of q-value estimation on offline-to-online reinforcement learning,” in {\it Proc. AAAI Conf. Artif. Intell.}, vol. 38, no. 15, 2024, pp. 16908--16916.

\bibitem{KumarA2020}
A. Kumar, A. Zhou, G. Tucker, and S. Levine, “Conservative Q-learning for offline reinforcement learning,” in {\it Proc. Adv. Neural Inf. Process. Syst.}, vol. 33, 2020, pp. 1179--1191.

\bibitem{KostrikovI2021}
I. Kostrikov, A. Nair, and S. Levine, “Ofﬂine reinforcement learning with implicit Q-learning,” in {\it Proc. Int. Conf. Learn. Represent.}, 2021, pp. 1--13.

\bibitem{DengZ2024}
Z. Deng, Z. Fu, L. Wang, Z. Yang, C. Bai, T. Zhou, Z. Wang, and J. Jiang, “False correlation reduction for offline reinforcement learning,” {\it IEEE Trans. Pattern Anal. Mach. Intell.}, vol. 46, no. 2, pp. 1199--1211, Feb. 2024.

\bibitem{MaoY2024}
Y. Mao, H. Zhang, C. Chen, Y. Xu, and X. Ji, “Supported value regularization for ofﬂine reinforcement learning,” in {\it Proc. Adv. Neural Inf. Process. Syst.}, vol. 36, 2023, pp. 40587--40609.

\bibitem{ShaoJ20223}
J. Shao, Y. Qu, C. Chen, H. Zhang, and X. Ji, “Counterfactual conservative q learning for offline multi-agent reinforcement learning,” in {\it Proc. Adv. Neural Inf. Process. Syst.}, vol. 36, 2023, pp. 77290--77312.

\bibitem{MaC2025}
C. Ma, D. Yang, T. Wu, Z. Liu, H. Yang, X. Chen, X. Lan, and N. Zheng, “Improving offline reinforcement learning with in-sample advantage regularization for robot manipulation,” {\it IEEE Trans. Neural Netw. Learn. Syst.}, vol. 36, no. 6, pp. 11215--11227, Jun. 2025.

\bibitem{ZhangR2024}
R. Zhang, Z. Luo, J. Sj{\"o}lund, T. Sch{\"o}n, and P. Mattsson, “Entropy-regularized diffusion policy with q-ensembles for offline reinforcement learning,” in {\it Proc. Adv. Neural Inf. Process. Syst.}, vol. 37, 2024, pp. 98871--98897.

\bibitem{Fujimoto2021330}
S. Fujimoto and S. S. Gu, “A minimalist approach to offline reinforcement learning,” in {\it Proc. Adv. Neural Inf. Process. Syst.}, vol. 34, 2021, pp. 20132--20145.

\bibitem{10004017}
G. Peng, J. Yang, X. Li, and M. O. Khyam, “Deep reinforcement learning with a stage incentive mechanism of dense reward for robotic trajectory planning,” {\it IEEE Trans. Syst., Man, Cybern., Syst.}, vol. 53, no. 6, pp. 3566--3573, Jun. 2023.

\bibitem{FujimotoS2019}
S. Fujimoto, D. Meger, and D. Precup, “Off-policy deep reinforcement learning without exploration,” in {\it Proc. Int. Conf. Mach. Learn.}, vol. 97, 2019, pp. 2052--2062.

\bibitem{RanY2023ef}
Y. Ran, Y.-C. Li, F. Zhang, Z. Zhang, and Y. Yu, “Policy regularization with dataset constraint for offline reinforcement learning,” in {\it Proc. Int. Conf. Mach. Learn.}, vol. 202, 2023, pp. 28701--28717.

\bibitem{ZhangZ2024}
Z. Zhang and X. Tan, “An implicit trust region approach to behavior regularized offline reinforcement learning,” in {\it Proc. AAAI Conf. Artif. Intell.}, vol. 38, no. 15, 2024, pp. 16944--16952.

\bibitem{LiJXzhan2023}
J. Li, X. Zhan, H. Xu, X. Zhu, J. Liu, and Y. Q. Zhang, “When data geometry meets deep function: Generalizing offline reinforcement learning,” in {\it Proc. Int. Conf. Learn. Represent.}, 2023, pp. 1--35.

\bibitem{ChenH2024}
H. Chen, C. Lu, Z. Wang, H. Su, and J. Zhu, “Score regularized policy optimization through diffusion behavior,” in {\it Proc. Int. Conf. Learn. Represent.}, 2024, pp. 1--20.

\bibitem{CaoS2024}
S. Cao, X. Wang, and Y. Cheng, “Dual behavior regularized offline deterministic actor-critic,” {\it IEEE Trans. Syst., Man, Cybern., Syst.}, vol. 54, no. 8, pp. 4841--4852, Aug. 2024.

\bibitem{ZhouZ2024}
Z. Zhou, G. Liu, and M. Zhou, “A robust mean-field actor-critic reinforcement learning against adversarial perturbations on agent states,” {\it IEEE Trans. Neural Netw. Learn. Syst.}, vol. 35, no. 10, pp. 14370--14381, Oct. 2024.

\bibitem{Ho2020}
J. Ho, A. Jain, and P. Abbeel, “Denoising diffusion probabilistic models,” in {\it Proc. Adv. Neural Inf. Process. Syst.}, vol. 33, 2020, pp. 6840--6851.

\bibitem{ZhangJ2024}
J. Zhang, Y. Cheng, S. Cao, and X. Wang, “Offline reinforcement learning with reverse diffusion guide policy,” {\it IEEE Trans. Ind. Informat.}, vol. 20, no. 10, pp. 11785--11793, Oct. 2024.

\bibitem{HuangLTPAMI2024}
L. Huang, B. Dong, and W. Zhang, “Efﬁcient ofﬂine reinforcement learning with relaxed conservatism,” {\it IEEE Trans. Pattern Anal. Mach. Intell.}, vol. 46, no. 8, pp. 5260--5272, Aug. 2024.

\bibitem{CLu2022}
C. Lu, Y. Zhou, F. Bao, J. Chen, C. Li, and J. Zhu, “DPM-solver: A fast ODE solver for diffusion probabilistic model sampling in around 10 steps,” in {\it Proc. Adv. Neural Inf. Process. Syst.}, vol. 35, 2022, pp. 5775--5787.

\bibitem{FuD4RL}
J. Fu, A. Kumar, O. Nachum, G. Tucker, and S. Levine, “D4RL: Datasets for deep data-driven reinforcement learning,” 2020, {\it arXiv:2004.07219}.

\bibitem{TorabiF2018}
F. Torabi, G. Warnell, and P. Stone, “Behavioral cloning from observation,” in {\it Proc. Joint Conf. Artif. Intell.}, 2018, pp. 4950--4957.

\bibitem{Wang172023}
Z. Wang, J. J. Hunt, and M. Zhou, “Diffusion policies as an expressive policy class for offline reinforcement learning,” in {\it Proc. Int. Conf. Learn. Represent.}, 2023, pp. 1--17.

\bibitem{Yamagata32842}
T. Yamagata, A. Khalil, and R. Santos-Rodriguez, “Q-learning decision transformer: Leveraging dynamic programming for conditional sequence modelling in offline rl,” in {\it Proc. Int. Conf. Mach. Learn.}, vol. 202, 2023, pp. 38989--39007.



\end{thebibliography}
\end{document}